\documentclass[11pt]{article}

\usepackage[utf8]{inputenc}
\usepackage[T1]{fontenc}
\usepackage{lmodern}
\usepackage[margin=1in]{geometry}
\usepackage{amsmath,amssymb,amsthm}
\usepackage{mathtools}
\usepackage{booktabs}
\usepackage{longtable}
\usepackage{array}
\usepackage{enumitem}
\usepackage{microtype}
\usepackage[dvipsnames]{xcolor}
\usepackage{graphicx}
\usepackage{float}
\usepackage{adjustbox}
\usepackage[textsize=scriptsize]{todonotes}
\usepackage{placeins}
\usepackage{tikz}
\usepackage[framemethod=tikz]{mdframed}
\usepackage{nimi-title}
\usepackage[round,authoryear]{natbib}
\usepackage[hidelinks,colorlinks=true,linkcolor=BrickRed,citecolor=NavyBlue,urlcolor=NavyBlue]{hyperref}
\usepackage{cleveref}
\usetikzlibrary{arrows.meta,backgrounds,calc,fit,positioning}
\graphicspath{{figures/}}
\NIMISetLogo{nimi_logo.png}
\NIMISetProjectLink[Open-source implementation]
  {https://github.com/NIMI-research/Tycho}
  {github.com/NIMI-research/Tycho}
\NIMISetProjectMeta{Apache-2.0}

\setlist[itemize]{leftmargin=1.35em,itemsep=1pt,topsep=2pt}
\setlist[enumerate]{leftmargin=1.45em,itemsep=1pt,topsep=2pt}

\newtheorem{definition}{Definition}
\newtheorem{proposition}{Proposition}
\newtheorem{remark}{Remark}

\newcommand{\States}{\mathcal{S}}
\newcommand{\Actions}{\mathcal{A}}
\newcommand{\Obs}{\mathcal{O}}
\newcommand{\Grid}{\mathcal{G}}
\newcommand{\Levels}{\mathcal{L}}

\newcommand{\Bel}{\mathcal{B}}
\newcommand{\Outcome}{\mathcal{Y}}
\newcommand{\trans}{\delta}
\newcommand{\render}{\rho}
\newcommand{\outcome}{\omega}
\newcommand{\init}{\iota}
\newcommand{\avail}{\alpha}
\newcommand{\emit}{\varepsilon}

\newcommand{\code}[1]{\texttt{#1}}

\newcommand{\RESSolNoneRhae}{99.03}
\newcommand{\RESSolNoneActions}{7,047}
\newcommand{\RESSolNoneCost}{26}
\newcommand{\RESSolNoneVsBest}{1.27}
\newcommand{\RESSolNoneSolved}{24.7}
\newcommand{\RESSolSingleRhae}{100.00}
\newcommand{\RESSolSingleActions}{6,623}
\newcommand{\RESSolSingleCost}{35}
\newcommand{\RESSolSingleVsBest}{1.20}
\newcommand{\RESSolSingleSolved}{25.0}
\newcommand{\RESSolOrchRhae}{100.00}
\newcommand{\RESSolOrchActions}{6,531}
\newcommand{\RESSolOrchCost}{91}
\newcommand{\RESSolOrchVsBest}{1.19}
\newcommand{\RESSolOrchBuilder}{201}
\newcommand{\RESSolOrchSolved}{25.0}
\newcommand{\RESSolTrigRhae}{100.00}
\newcommand{\RESSolTrigActions}{6,516}
\newcommand{\RESSolTrigCost}{164}
\newcommand{\RESSolTrigVsBest}{1.18}
\newcommand{\RESSolTrigBuilder}{717}
\newcommand{\RESSolTrigSolved}{25.0}
\newcommand{\RESAstraNoneRhae}{100.00}
\newcommand{\RESAstraNoneActions}{7,014}
\newcommand{\RESAstraNoneCost}{49}
\newcommand{\RESAstraNoneVsBest}{1.28}
\newcommand{\RESAstraNoneSolved}{25.0}
\newcommand{\RESAstraSingleRhae}{100.00}
\newcommand{\RESAstraSingleActions}{6,763}
\newcommand{\RESAstraSingleCost}{62}
\newcommand{\RESAstraSingleVsBest}{1.23}
\newcommand{\RESAstraSingleSolved}{25.0}
\newcommand{\RESAstraOrchRhae}{100.00}
\newcommand{\RESAstraOrchActions}{6,449}
\newcommand{\RESAstraOrchCost}{161}
\newcommand{\RESAstraOrchVsBest}{1.17}
\newcommand{\RESAstraOrchBuilder}{141}
\newcommand{\RESAstraOrchSolved}{25.0}
\newcommand{\RESAstraTrigRhae}{100.00}
\newcommand{\RESAstraTrigActions}{6,277}
\newcommand{\RESAstraTrigActionsMin}{6,216}
\newcommand{\RESAstraTrigActionsMax}{6,313}
\newcommand{\RESAstraTrigCost}{335}
\newcommand{\RESAstraTrigVsBest}{1.14}
\newcommand{\RESAstraTrigVsBestMin}{1.13}
\newcommand{\RESAstraTrigVsBestMax}{1.15}
\newcommand{\RESAstraTrigBuilder}{602}
\newcommand{\RESAstraTrigSolved}{25.0}
\newcommand{\RESFableNoneRhae}{98.89}
\newcommand{\RESFableNoneActions}{8,301}
\newcommand{\RESFableNoneCost}{142}
\newcommand{\RESFableNoneVsBest}{1.52}
\newcommand{\RESFableNoneSolved}{24.7}
\newcommand{\RESFableSingleRhae}{98.74}
\newcommand{\RESFableSingleActions}{8,092}
\newcommand{\RESFableSingleCost}{146}
\newcommand{\RESFableSingleVsBest}{1.48}
\newcommand{\RESFableSingleSolved}{24.7}
\newcommand{\RESFableOrchRhae}{98.68}
\newcommand{\RESFableOrchActions}{8,247}
\newcommand{\RESFableOrchCost}{134}
\newcommand{\RESFableOrchVsBest}{1.49}
\newcommand{\RESFableOrchBuilder}{0}
\newcommand{\RESFableOrchSolved}{24.7}
\newcommand{\RESFableTrigRhae}{99.84}
\newcommand{\RESFableTrigActions}{7,050}
\newcommand{\RESFableTrigCost}{460}
\newcommand{\RESFableTrigVsBest}{1.28}
\newcommand{\RESFableTrigBuilder}{759}
\newcommand{\RESFableTrigSolved}{25.0}
\newcommand{\RESOpusNoneRhae}{99.80}
\newcommand{\RESOpusNoneActions}{7,306}
\newcommand{\RESOpusNoneCost}{164}
\newcommand{\RESOpusNoneVsBest}{1.33}
\newcommand{\RESOpusNoneSolved}{25.0}
\newcommand{\RESOpusSingleRhae}{99.96}
\newcommand{\RESOpusSingleActions}{7,234}
\newcommand{\RESOpusSingleCost}{158}
\newcommand{\RESOpusSingleVsBest}{1.32}
\newcommand{\RESOpusSingleSolved}{25.0}
\newcommand{\RESOpusOrchRhae}{100.00}
\newcommand{\RESOpusOrchActions}{7,177}
\newcommand{\RESOpusOrchCost}{160}
\newcommand{\RESOpusOrchVsBest}{1.31}
\newcommand{\RESOpusOrchBuilder}{0}
\newcommand{\RESOpusOrchSolved}{25.0}
\newcommand{\RESOpusTrigRhae}{100.00}
\newcommand{\RESOpusTrigActions}{6,791}
\newcommand{\RESOpusTrigCost}{449}
\newcommand{\RESOpusTrigVsBest}{1.23}
\newcommand{\RESOpusTrigBuilder}{749}
\newcommand{\RESOpusTrigSolved}{25.0}
\newcommand{\RESBestKnownPath}{5,499}

\newcommand{\ArxivOperationalTable}{%
\begin{table}[H]
\centering
\scriptsize
\setlength{\tabcolsep}{3.2pt}
\begin{tabular}{@{}lrrrrrrr@{}}
\toprule
Policy & Wins & Level limit & LM limit & Cost limit & Context comp. & Peak prompt & LM calls \\
\midrule
No world model & 19 & 3 & 0 & 3 & 3 & 850k & 24.5k \\
Single & 19 & 2 & 0 & 4 & 4 & 851k & 26.7k \\
Orchestrator & 21 & 1 & 1 & 2 & 3 & 852k & 26.9k \\
Trigger & 18 & 2 & 5 & 0 & 1 & 852k & 50.3k \\
GPT-5.6 & 25 & 0 & 0 & 0 & 5 & 219k & 29.1k \\
Opus 5 & 25 & 0 & 0 & 0 & 0 & 300k & 18.7k \\
\bottomrule
\end{tabular}
\caption{Operational diagnostics for the six first-version runs.  LM calls are the corrected counts of
calls that produced retained trajectory steps (\Cref{app:corrections}).}
\label{tab:appendix-operational}
\end{table}
}

\newcommand{\ArxivOutcomeTable}{%
\begin{table}[H]
\centering
\scriptsize
\setlength{\tabcolsep}{3.2pt}
\begin{tabular}{@{}lrrrr@{}}
\toprule
Policy & LC recall ($n$) & GO recall ($n$) & LC false + & GO false + \\
\midrule
No world model & -- & -- & -- & -- \\
Single & 28.5 (137) & 25.0 (12) & 0.2 & 0.6 \\
Orchestrator & 46.5 (155) & 66.7 (6) & 0.3 & 0.0 \\
Trigger & 98.1 (156) & 100.0 (20) & 0.0 & 0.0 \\
GPT-5.6 & 99.5 (183) & 100.0 (10) & 0.0 & 0.0 \\
Opus 5 & 87.8 (181) & 100.0 (6) & 3.4 & 0.0 \\
\bottomrule
\end{tabular}
\caption{End-of-run outcome diagnostics (percent except event counts).  The first version's
``verified games'' column was computed against incorrectly stored terminal frames and has been removed
(\Cref{app:corrections}).}
\label{tab:appendix-outcomes}
\end{table}
}

\newcommand{\ArxivPlanTableVone}{%
\begin{table}[t]
\centering
\small
\setlength{\tabcolsep}{4.2pt}
\begin{adjustbox}{max width=\linewidth}
\begin{tabular}{@{}lrrrrrrr@{}}
\toprule
Policy & Plans & \shortstack{Followed\\in full} & \shortstack{Level ended\\on plan} &
\shortstack{Ended after\\contradiction} & \shortstack{Levels ending\\on a plan} &
\shortstack{Actions on\\a plan (\%)} & \shortstack{Builder\\runs} \\
\midrule
Single & 25 & 11 & 9 & 1 & 2--8 / 162 & 0.9 & -- \\
Orchestrator & 11 & 2 & 0 & 1 & 0 / 166 & 0.2 & 147 \\
Trigger & 198 & 113 & 78 & 48 & 36--49 / 162 & 5.9 & 1{,}192 \\
\bottomrule
\end{tabular}
\end{adjustbox}
\caption{Plans returned by the harness planner in the first-version matched study (Claude Opus~4.8, one
run per policy) and their agreement with the executed actions.  Counts are lower bounds; ranges arise
where the end of a plan is not in the log.}
\label{tab:plan-stats-v1}
\end{table}%
}

\newcommand{\ArxivPlanTableVtwo}{%
\begin{table}[t]
\centering
\small
\setlength{\tabcolsep}{4.2pt}
\begin{adjustbox}{max width=\linewidth}
\begin{tabular}{@{}llrrrrrrr@{}}
\toprule
Model & Policy & Plans & \shortstack{Followed\\in full} & \shortstack{Level ended\\on plan} &
\shortstack{Ended after\\contradiction} & \shortstack{Levels ending\\on a plan} &
\shortstack{Actions on\\a plan (\%)} & \shortstack{Builder\\runs} \\
\midrule
GPT-6.2 Sol & Single & 137 & 103 & 96 & 27 & 81 & 34.3 & -- \\
 & Orchestrator & 166 & 127 & 122 & 23 & 113 & 40.1 & 201 \\
 & Trigger & 382 & 241 & 226 & 96 & 144 & 57.9 & 717 \\
\addlinespace
GPT-6.2 Astra & Single & 20 & 17 & 17 & 2 & 12 & 4.1 & -- \\
 & Orchestrator & 147 & 119 & 116 & 21 & 111 & 46.3 & 141 \\
 & Trigger & 434 & 283 & 275 & 135 & 165 & 78.8 & 602 \\
\addlinespace
Claude Fable~5.2 & Trigger & 333 & 183 & 173 & 101 & 112 & 42.3 & 759 \\
Claude Opus~5.5 & Trigger & 464 & 243 & 224 & 170 & 136 & 64.7 & 749 \\
\addlinespace
\multicolumn{9}{@{}l}{Claude models, single and orchestrator: no harness plans (one exception in twelve runs)} \\
\bottomrule
\end{tabular}
\end{adjustbox}
\caption{Plans returned by the harness planner in the updated study and their agreement with the
executed actions: means over three runs, with levels counted out of 183.  The policy without a world
model has no planner.}
\label{tab:plan-stats-v2}
\end{table}%
}

\hypersetup{
  pdftitle={Tycho: Active Abstraction with Programmatic World Models for ARC-AGI-3},
  pdfauthor={Jens Lehmann, Andrei Aioanei, Sahar Vahdati}
}

\NIMISetTitleNote{%
  \NIMINameNote{Why Tycho}{%
    Tycho Brahe, the sixteenth-century astronomer, charted the heavens with unmatched precision, but did
    not live to see their hidden order become law.  Kepler found that order in Tycho's observations.
    The name is our thesis: intelligence begins by looking closely at an unfamiliar
    world, turning evidence into a model, and using that model to see beyond the next observation.%
  }%
}

\title{Tycho: Active Abstraction with Programmatic World Models for ARC-AGI-3}
\author{%
  \NIMIAuthorBlock{0.32\linewidth}
    {Jens Lehmann}
    {Dresden University of Technology\\Amazon\\Dresden, Germany\NIMIEmail{jens.lehmann@tu-dresden.de}}
    {Work done outside of Amazon}%
  \hfill
  \NIMIAuthorBlock{0.32\linewidth}
    {Andrei Aioanei}
    {TIB -- Leibniz Information Centre for Science and Technology, \\ Hannover, Germany\NIMIEmail{andrei.aioanei@tib.eu}}
    {}%
  \hfill
  \NIMIAuthorBlock{0.32\linewidth}
    {Sahar Vahdati}
    {Leibniz University of Hannover \\  
    TIB -- Leibniz Information Centre for Science and Technology, \\ Hannover, Germany\NIMIEmail{sahar.vahdati@tib.eu}}
    {}%
}
\date{November 22, 2026 (v2)}

\begin{document}
\maketitle

\begin{nimiabstract}
ARC-AGI-3 turns abstraction into an interactive problem of skill acquisition.
A player must figure out an unfamiliar game’s rules, hidden state, and goal 
while maintaining action efficiency, since every move counts. We formalize 
these game environments as parameterized rendered deterministic Moore machines.
We then introduce Tycho, a coding-agent system for constructing and using 
game-specific models during interaction. Tycho distinguishes observations 
on which the agent can act from intermediate animation, level-completion, 
and game-over frames. Using this structured interaction history, an agent 
can model, test, plan with, repair, or bypass a free-form executable hypothesis 
about the game.

In one public-set run per policy, we compare four orchestration policies on all 25 public ARC-AGI-3
games with Claude~Opus~4.8: direct reasoning, actor-authored modeling, actor-requested delegation to a
builder, and automatically triggered model repair on verification failure.  Actor-requested delegation
obtains the highest observed mean Relative Human Action Efficiency (RHAE), at $88.49$.  With this policy,
GPT-5.6 Sol and Opus~5 both reach $100.00$ public-set RHAE; Opus~5 completes all 183 levels in $6{,}641$
actions, 61\% fewer than the aggregate official human baselines.
Automatic repair produces models that reproduce observed transitions much more accurately, yet reaches
only $83.07$ RHAE: transition match shows whether a simulator reproduces observed dynamics, not whether
it has identified the objective or whether consulting it improves the next action.  Strong play
therefore also requires deciding when to construct, repair, use, or bypass the model.
We call this joint problem \emph{active abstraction}: generating a testable model from costly interaction and deciding when acquiring or using that model is worth its cost.

An updated study repeats the comparison three times with GPT-6.2 Sol, GPT-6.2 Astra, Claude Fable~5.2,
and Claude Opus~5.5 on a revised harness. Every condition now averages at least $\RESFableOrchRhae$
public-set RHAE, so the comparison moves to actions and cost: an executable model reduces actions for
the GPT-6.2 models, and automatic repair does so for the Claude models, which do not build a model
unprompted.  GPT-6.2 Sol reaches $100.00$ public-set RHAE in $6{,}626$ actions for about \$31 of
inference, against \$2{,}986 for Opus~5 less than four months earlier.
\end{nimiabstract}

\section{Introduction}

The hard part of intelligence is not storing many solutions.  It is acquiring useful new skills from a
small amount of experience.  Chollet's definition of intelligence as skill-acquisition efficiency
\citep{chollet2019measure} makes this view of intelligence measurable: a system that
solves a task only after exhaustive search, hidden pretraining leakage, or thousands of environment
trials has not demonstrated the same kind of generalization as a human who infers a rule from a
few examples.

The first two ARC-AGI benchmarks made this idea concrete through static reasoning tasks and became widely
regarded as exceptionally difficult tests of generalization in artificial intelligence. ARC-AGI-3
\citep{arcprize2026arcagi3} moves the same principle into an interactive setting. The agent observes a $64\times64$ colored grid,
chooses an action, receives the next grid, and must complete a sequence of levels.
No natural-language rule, goal description, or action semantics are provided. Every environment action
is part of the score.  The agent must therefore decide not only \emph{what is true}, but also \emph{which action is worth performing next}.

ARC-AGI-3 is a bounded instance of a broader challenge for general agents: entering a new domain,
forming hypotheses about its mechanics, selecting informative actions, and using the resulting model for
action selection. Unlike static ARC, evidence is not fixed in advance: each transition $(o_t,a_t,o_{t+1})$ is both
evidence about the world and an irreversible expenditure from the action budget. We call this problem
\emph{active abstraction}: automatically constructing an explicit, testable model from costly interaction and deciding when
acquiring or using that model is worth its cost. A useful abstraction must preserve hidden state and
cross-level regularities, guide action from limited evidence, and justify the actions spent testing it.

Tycho operationalizes several central elements of the world-model-induction agenda of
\citet{ying2025adaptive}: constructing and revising models from limited interaction, using them to guide
action and exploration, and making the induced representation directly inspectable. World models are
useful here for three reasons. 
They let the agent simulate alternative action sequences before spending
real actions; they make exploration purposeful when competing hypotheses predict different observations
or plans; and they help localize failure when modeled dynamics fit the interaction history but outcome
inference or planning remains wrong. 
Programmatic models are directly inspectable and editable.
They can be tested against observed
transitions, revised after mismatches, and used by standard planning algorithms.
Passing observed replay checks does not prove generalization, but it turns an
informal interpretation into a concrete hypothesis.
Independent analysis of frontier-model trajectories identifies the corresponding failures in practice:
agents can observe a correct local effect but form the wrong global model, import an inappropriate game
abstraction, or complete one level without extracting mechanics that transfer to the next
\citep{arcprize2026failureanalysis}.

We use \emph{world model} in the agent-relative sense, a model of the environment in which an agent
acts.  Here, that environment is one ARC-AGI-3 game. Tycho induces a game-specific model from interaction,
not a general model of the physical world.
More precisely, the model represents latent environment state,
action-conditioned dynamics, observation formation, and terminal outcomes in a form that supports
prospective simulation and action selection.  This paper studies one instantiation: \emph{programmatic} world models.  
Tycho externalizes the representation as an editable, task-specific executable hypothesis in Python
with freely chosen state variables, initialization and transition rules, a renderer back to the observed
grid, and an outcome classifier. It may also expose planner-facing search guidance.  
It can be verified against past experience and used for planning, and its structure can be read, edited, and simplified.

Tycho instantiates this research program through a loop over evidence, executable hypotheses, and action.
\Cref{fig:tycho-architecture} summarizes the shared loop and \Cref{tab:orchestration-policies} defines the
four model-maintenance policies varied in our matched comparison. Its contributions are:
\begin{itemize}
  \item We formalize ARC-AGI-3 as active identification of deterministic rendered Moore machines with
  transition emissions and a typed agent-facing evidence contract.  The formulation makes hidden state,
  non-injective observations, level boundaries, task-equivalent models, and information-gathering actions
  explicit.
  \item We define a free-form, executable hypothesis language for programmatic world models and give the agent freedom to decide whether to use it.  The same interaction interface supports four policies: (1) no model, (2) single-actor modeling, (3) actor-requested delegation, and (4) automatically triggered repair. We benchmark and compare all four policies.
  \item We separate model accuracy from game-playing performance through pre-action execution, accepted
  transition match, prediction coverage, terminal-outcome prediction, plan following, scored actions,
  and inference cost.  In a replay-audited public-set evaluation, the matched study selects Orchestrator;
  GPT-5.6 Sol and Opus~5 both reach $100.00$ RHAE, with Opus~5 using 14.5\% fewer actions.
  \item We describe the human-mediated cross-run adaptation loop used to identify candidate architectural
  changes from limited development feedback and validate them before they enter future frozen harnesses.
  \item In this second version we repeat the four-policy comparison on an updated harness with four
  current models and three runs per condition (\Cref{sec:vtwo}).  Public-set RHAE is saturated under every
  policy, so we compare scored actions, actions relative to the best known route, plan use, and
  inference cost, and we report how far the cost of a saturated public-set score has fallen.
\end{itemize}

Corrections to the first version are applied in place and summarized in \Cref{app:corrections}: inference
costs and call counts are restated under a corrected accounting, and one frame-recording error and its
consequences are documented.  No score, completed level, or action count changes.

\begin{figure}[t]
\centering
\resizebox{0.9\linewidth}{!}{%
\begin{tikzpicture}[
  font=\small,
  box/.style={rounded corners=2pt, thick, align=center, minimum height=1.45cm,
    inner sep=6pt, text width=3.20cm},
  env/.style={box, draw=black!58, fill=black!3, text width=2.45cm},
  record/.style={box, draw=NavyBlue!75, fill=NavyBlue!6, text width=4.20cm},
  actor/.style={box, draw=ForestGreen!75!black, fill=ForestGreen!7, text width=3.80cm},
  construct/.style={box, draw=RoyalPurple!75, fill=RoyalPurple!6},
  model/.style={box, draw=BrickRed!75, fill=BrickRed!6},
  evaluate/.style={box, draw=ForestGreen!65!black, fill=ForestGreen!4},
  arr/.style={-{Latex[length=2.3mm]}, thick, draw=black!70},
  feedback/.style={-{Latex[length=2.3mm]}, thick, draw=RoyalPurple!72}
]
\node[env] (environment) at (0.70,0) {\textbf{ARC-AGI-3}\\environment};
\node[record] (record) at (5.05,0) {\textbf{Interaction record}\\frames and action masks,\\typed events and history};
\node[actor] (actor) at (11.60,0) {\textbf{Frontier actor}\\reasons over evidence,\\commits each action};

\draw[arr] (environment) -- (record);
\draw[arr] (record) -- node[above]{evidence} (actor);
\draw[arr] (actor.north) -- ++(0,0.52) -|
  node[pos=0.28, above]{action} (environment.north);

\node[construct] (construct) at (2.35,-4.00) {\textbf{Construct or repair}\\actor or builder};
\node[model] (wm) at (7.00,-4.00) {\textbf{Executable model}\\state and dynamics,\\render and outcome};
\node[evaluate] (evaluate) at (11.60,-4.00) {\textbf{Verify and plan}\\replay history,\\search model states};

\draw[arr] (record.south) -- ++(0,-0.78) -| (construct.north);
\draw[feedback] (construct) -- (wm);
\draw[arr] (wm) -- (evaluate);
\draw[arr] (evaluate.north) -- node[right]{advice} (actor.south);
\coordinate (feedbackbottom) at ([yshift=-0.62cm]evaluate.south);
\coordinate (feedbackleft) at (construct.south |- feedbackbottom);
\draw[feedback] (evaluate.south) -- (feedbackbottom) --
  node[midway, above] (feedbacklabel) {repair feedback} (feedbackleft) -- (construct.south);

\node[font=\bfseries\small, text=BrickRed!78!black] (workbenchtitle)
  at (7.00,-2.55) {Optional executable-model workbench};

\begin{scope}[on background layer]
  \node[draw=BrickRed!35, fill=BrickRed!2, rounded corners=4pt, inner sep=12pt,
    fit=(workbenchtitle)(construct)(wm)(evaluate)(feedbackbottom)(feedbackleft)(feedbacklabel)] (workbench) {};
\end{scope}
\end{tikzpicture}%
}
\caption{Tycho's shared task-time architecture.  The actor reasons over the interaction record and alone
commits environment actions.  The optional workbench constructs or repairs an executable model, verifies it,
and returns plans or diagnostics.}
\label{fig:tycho-architecture}
\end{figure}
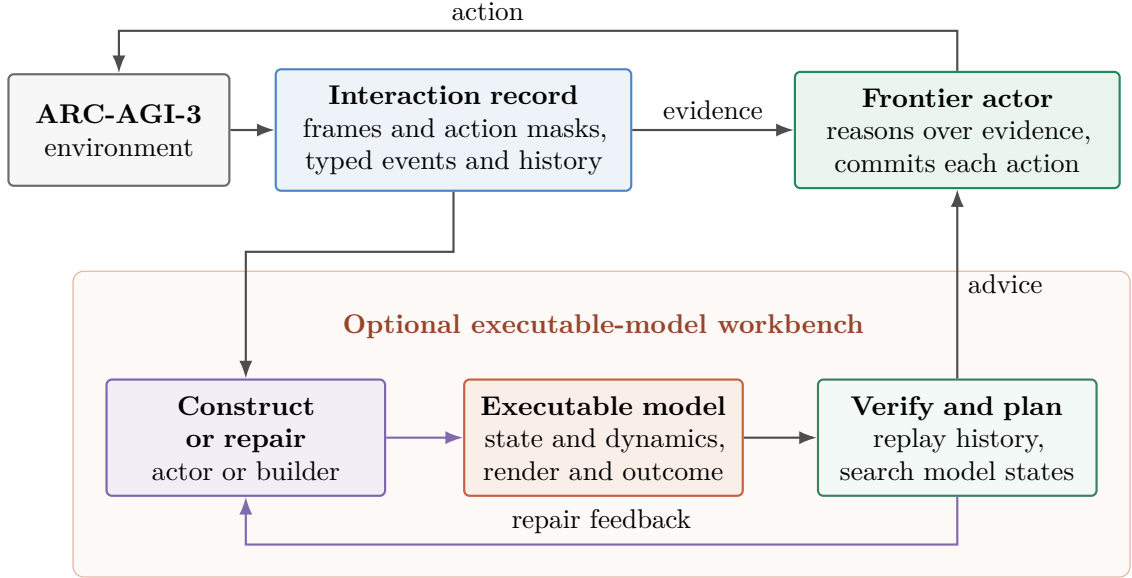

\section{Related Work}
\label{sec:related}

\paragraph{ARC, skill acquisition, and induction versus transduction.}
ARC was designed to measure broad generalization and skill-acquisition efficiency rather than
task-specific training performance \citep{chollet2019measure}.  In static ARC, \citet{li2025inductiontransduction}
distinguish induction---constructing an intermediate rule that can be applied to the test input---from
transduction, which predicts the test output directly.  They find the approaches complementary: induction
is stronger when exact computation and composition matter, whereas transduction can better capture less
precisely specified perceptual concepts.  ARC-AGI-3 carries a related tension into sequential interaction.
An executable world model makes the inferred rule persistent and testable while direct actor reasoning is
\emph{transduction-like}, because it maps the current observation and
interaction record to an action without first externalizing a reusable transition function.  This
complementarity motivates Tycho's central design choice: world modeling is available throughout the task,
but is not mandatory when direct reasoning is sufficient.

ARC-AGI-3 extends skill-acquisition to interactive agents
\citep{arcprize2026arcagi3}: the agent must infer reusable structure while every
exploratory action is scored.  It operationalizes the proposal to evaluate adaptive world-model induction
through novel games \citep{ying2025adaptive}.  A cross-generation survey identifies test-time
adaptation and refinement loops as recurring strengths of approaches solving ARC, while compositional generalization, interactive
learning, and computational efficiency remain central bottlenecks \citep{vahdati2026arcprogress}.  Tycho
therefore treats action efficiency as evidence of useful abstraction, not merely of eventual task success.

\paragraph{ARC-AGI-3 agent architectures.}
Recent systems emphasize different parts of the interactive inference problem.  DreamTeam treats the
editable workspace---artifacts, evidence, feedback, and agent roles---as an object of inference-time
optimization, with specialist roles for executable modeling, probing, and planning
\citep{sarafian2026workspace}.  AERA separates exploration, verification, and planning
\citep{han2026explore}.  Its released evaluator maps \code{WIN}, \code{GAME\_OVER}, and a
\code{None} step return to the same termination path, which the agent records as solved
\citep{han2026aeracode}.  We therefore do not treat its completion counts as directly comparable to
official outcome-based counts.  The Duck
Harness demonstrates a deliberately lightweight alternative in which a local model writes and executes
Python in a short-context REPL \citep{tufalabs2026duck}. Graph-based exploration instead makes systematic
state-space coverage the organizing principle \citep{rudakov2026graph}.  Together these systems expose the
need to distinguish base-model capability, evidence and memory design, exploration, model construction,
metareasoning, and evaluation protocol.

\paragraph{Contemporaneous coding-agent systems for ARC-AGI-3.}
Several contemporaneous systems independently combine persistent evidence, executable hypotheses, replay,
and planning.
Rodionov's coding-agent system uses Python world models, a scripted controller, predefined interfaces,
replay verification, plan execution, and refactoring as a practical MDL proxy
\citep{rodionov2026executable}.  Its component study compares textual reasoning, flexible executable
modeling, scheduled simplification, and fixed-interface verification \citep{rodionov2026components}.
Verification ranks first in all four matched settings but uses the most resources, while requiring a
persistent executable deliverable is not uniformly beneficial.  OPINE-World couples acting and
model-synthesis agents through counterexample-guided repair, exact replay admission, forward planning, and
ontology-error-prioritized exploration \citep{courtis2026opine}.  The Schema harness combines an
append-only transition record, persistent notes, editable Python hypotheses, full-history backtesting,
certified-model search, discriminating probes, and guarded action queues invalidated by prediction
mismatches \citep{schema2026}.
\mbox{PRO-LONG isolates} a simpler memory intervention: every observation, action, and outcome is appended to one
structured log, which a coding agent searches with tools such as \code{grep} and Python
\citep{fox2026prolong}.  It neither requires an executable model nor supplies a dedicated model interface,
although its agents sometimes synthesize transition programs and breadth-first search on their own.

The main architectural differences concern representation and model allocation.  OPINE-World commits to an
object-centric factorization.  It synthesizes a typed extractor, uses Bayesian effect statistics to
prioritize unresolved object types, and admits only exact-replay models into a fixed actor--synthesizer
loop.  Its formal analysis assumes an observable-Markov representation~\citep{courtis2026opine}.  The Schema harness likewise
places a persistent executable hypothesis at the center of control.  Backtesting certifies the hypothesis
for search, and prediction mismatches invalidate queued actions~\citep{schema2026}.  Rodionov instead varies executable-model
components within a prescribed controller~\citep{rodionov2026components}.  Tycho imposes neither a fixed object decomposition nor a fixed
internal state representation.  It also makes the executable model optional.  Across a shared evidence and action
interface, the experiment varies who constructs or repairs a model and when.  This makes model allocation,
rather than model synthesis alone, the object of study.  PRO-LONG is closest to Tycho's persistent-evidence
layer: both let a coding agent query exact earlier interactions programmatically.  PRO-LONG deliberately
stops at this general memory interface~\citep{fox2026prolong}.  Tycho additionally represents transient and terminal evidence,
gives executable models explicit initialization, rendering, outcome, verification, and planning interfaces,
and tests when that machinery should be invoked.

The evaluation protocols also differ.  OPINE-World reports $78.4$ RHAE with Opus~4.8.  Rodionov's fixed
verification variant reports $98.97$ with GPT-5.6 Sol at \texttt{xhigh}.  The Schema harness reports
$98.98$ by running Opus~4.8 first, rerunning games below a score of 80 with Fable~5, and retaining the
better trajectory for each game \citep{schema2026}.  On every rerun game, taking the better trajectory can
only preserve or raise the first-pass score.  A fallback model that is stronger on average does not remove
this best-of-two advantage because it need not dominate the first model on every game or run.  The GPT
protocol applies the same rule to \texttt{xhigh} and \texttt{max}.  The released traces allow the 25 retained
trajectories to be rescored, but the aggregate is a conditional multi-model selection result rather than a
score from one fixed model pass.  PRO-LONG reports $94.6$ from one Fable~5 pass with a 2,000-action limit,
then $97.4$ after selectively rerunning some games and retaining the better result.  The authors call this
a lower bound on best@2 because several games have only one run.  With one fixed configuration and one
trajectory per game, Tycho reaches $100.00$ with GPT-5.6 Sol/\texttt{max} and $100.00$ with
Opus~5/\texttt{xhigh}.  The headline numbers alone therefore mix architecture, model
choice, rerunning, and inference allocation.  The protocol must be stated alongside the score.

Released usage artifacts add a cost dimension: our list-price normalization estimates Tycho at
\$6.51k versus \$12.4--15.2k for OPINE-World in the Opus comparison, and at \$4.88k versus
\$13.6--15.5k for Rodionov's near-saturation GPT verification runs under the latter's reported
preliminary API-key cache behavior.  The Opus~5 run costs an estimated \$2.99k at the frozen Opus-rate
schedule.  PRO-LONG reports a substantially lower \$1.75k list-price equivalent
for its selective Fable protocol.  Its pricing formula is consistent with public Fable rates, but the
released bundle strips the per-call token telemetry and omits the additional runs, so the total cannot yet
be independently reconstructed (\Cref{sec:cross-system-cost}).

Tycho's four matched Opus scorecards were completed by July~13, before we learned of the Schema harness and
Rodionov's component study.  Tycho's distinguishing contributions are a rendered Moore-machine
formalization of interactive games, a matched comparison of four modeling policies, and a trace-grounded
qualitative analysis of agent adaptation.  The matched comparison shows that automatic repair makes
models much more exact yet underperforms actor-requested delegation.

\paragraph{Programmatic world models and metareasoning.}
WorldCoder studies LLM agents that build programmatic world models from interaction
\citep{tang2024worldcoder}.  Code World Models use generated Python dynamics and execution feedback for
offline model-based control \citep{dainese2024codeworldmodels}.  Theory-based reinforcement learning
starts from structured causal theories.  EMPA performs program-based model inference, exploration, and
planning \citep{tsividis2021human}.  TheoryCoder combines synthesized Python dynamics with human-specified
high-level operators \citep{ahmed2025theorycoder}.  TheoryCoder-2 also learns reusable planning abstractions
\citep{ahmed2026theorycoder2}.  PoE-World composes programmatic experts and extends sparse-data
program synthesis to stochastic, non-gridworld domains \citep{piriyakulkij2025poeworld}.  Code world
models for general game playing are synthesized from supplied rules and trajectories and searched with
MCTS \citep{lehrach2026codeworldmodels}, and OneLife infers probabilistic laws from structured symbolic
states in a single unguided episode \citep{khan2026onelife}.  These systems
show that source-code hypotheses can support sample-efficient control.  Tycho learns from rendered frames
without supplied rules and asks the complementary
policy question: when should a frontier coding agent construct, repair, query, plan through, or bypass one?

That allocation is a form of metareasoning: computation is useful only through its effect on external
decisions \citep{russell1991metareasoning}.  Across visual and agentic tasks, agents may rarely invoke
available simulators, misuse their rollouts, or degrade when simulation is forced
\citep{qian2026foresight}.  Similarly, planning before every action can waste test-time compute and reduce
long-horizon performance, motivating policies that learn when to plan \citep{paglieri2025when}.  Closest
to this question is work on selective tool use, which argues for invoking tools only when epistemically
necessary \citep{wang2026epistemic} and triggers them from meta-cognitive signals \citep{li2025meco}.  Tool-using
agents such as ReAct, Reflexion, and Voyager show the broader benefits of external tools, memory, and
self-revision \citep{yao2023react,shinn2023reflexion,wang2023voyager}.  Tycho makes the allocation decision
explicit and separately measures pre-action model execution, accepted transition match, surfaced
recommendations, exact following, and game-playing performance.

\paragraph{Active identification and decision-relevant models.}
World models are central to model-based agents and reinforcement learning
\citep{ha2018world,hafner2020dream,schrittwieser2020mastering}.  Tycho uses discrete programs because they
make hypotheses inspectable, testable against recorded history, and directly usable by ordinary planners.
Prediction accuracy and task performance are distinct objectives: globally accurate dynamics can be irrelevant to the current decision,
while a task-useful model may abstract away most observable detail
\citep{lambert2020objective,grimm2020value}.  WorldTest instead separates reward-free model acquisition
from later prediction and planning tests \citep{warrier2025worldtest}.  Tycho measures transition prediction
and game play within the same scarce-action trajectory.

POMDPs and belief-space planning formalize action under hidden state
\citep{kaelbling1998planning,bonet2000planning}.
Bayes-adaptive and dual-control perspectives couple action selection to uncertainty about the dynamics:
an action can improve both the physical state and the model used for later decisions
\citep{klenske2016dual}.  Value-of-information and active-data-selection results likewise explain why an
action can be useful because it discriminates among hypotheses rather than directly approaching a goal
\citep{howard1966information,mackay1992information}.  Tycho does not assign priors or likelihoods to the
free-form programs proposed for a one-off game.  Full Bayesian planning would therefore require an
additional probability model over programs and planning through that distribution.  Tycho instead uses
selected deterministic hypotheses consistent with the observed record, without maintaining an enumerated
version space.  This stops short of belief-space optimization but preserves the distinction between task
progress and information value.

Identifying a machine by experiment goes back to \citet{moore1956gedanken}, and learning
deterministic machines from queries and counterexamples provides a complementary identification
perspective \citep{angluin1987learning,rivest1993inference}.  ARC-AGI-3 is a rendered,
cost-sensitive variant: reset and membership queries are not freely available, observations are pixel
grids, and the terminal-outcome map is hidden.  Work on Moore and Mealy machines from traces and apartness
suggests tools for finite-state identification \citep{giantamidis2016learning,vaandrager2022apartness}.  Agentic
automata learning nevertheless finds frontier agents brittle in query planning, evidence integration, and
hypothesis construction as machines grow \citep{menaged2026automata}.  Tycho studies the analogous problem
with unknown goals and irreversible scored actions, where abstraction, simulation of alternative actions,
and model testing and revision must all fit within a scarce interaction budget.

\section{Formal Setting: Rendered Deterministic Moore Machines}
\label{sec:formal}

\subsection{Level mechanics and scored interaction}

\paragraph{Why rendered Moore machines?}
In ARC-AGI-3, an agent faces a sequence of decision frames. 
It observes a grid and the available
actions, chooses one, and receives a new observation or a terminal outcome.  
A model that supports action selection and planning must explain more than the visible frame.  
It must retain hidden facts from earlier interaction, predict how actions
change those facts, render the resulting observation, and recognize completion or failure.  
A Moore machine \citep{moore1956gedanken} makes this separation explicit: actions update an underlying state, while each state produces the observation
and outcome available to the agent. 
We call the machine \emph{rendered} because its principal observation is a grid.  
We augment it with transition emissions to preserve informative animation frames between
decision frames.  \Cref{fig:rendered-moore-machine} illustrates the resulting cycle.

Let $\Levels=\{1,\ldots,L\}$ denote the ordered levels of a game, let $\Grid=\{0,\ldots,15\}^{64\times64}$ denote the space of rendered grids, and let
    \[
 \Outcome=\{\mathrm{ongoing},\mathrm{level\_complete},\mathrm{game\_over}\}
\]
   denote the possible outcomes. At the start of level $\ell$, the environment is in an underlying state ($s_0^\ell\in\States$). This state encodes the level’s layout, object positions, counters, and every other variable—whether visible in the
  rendered grid or hidden—needed to determine subsequent transitions and outcomes.

Formally, the object below is a deterministic transition system with Moore-style state outputs and
transition emissions.  We use \emph{rendered Moore machine} as shorthand.

\begin{definition}[Level mechanics]
The mechanics shared by the levels of one ARC-AGI-3 game form a parameterized rendered deterministic
Moore machine with transition emissions,
\[
M_\theta^\ell=
(\States,\Actions,s_0^\ell,\trans_\theta,q_\theta,\emit_\theta).
\]
Let $\States$ denote the environment’s underlying state space, containing every visible or hidden
variable needed to determine future behavior, and let $\Actions$ denote the set of game actions.  The output map
\[
q_\theta(s)=
\bigl(\underbrace{\render_\theta(s)}_{\text{rendered grid}},
      \underbrace{\avail_\theta(s)}_{\text{available actions}},
      \underbrace{\outcome_\theta(s)}_{\text{outcome}}\bigr)
\in \Grid\times 2^{\Actions}\times\Outcome
\]
gives the settled observation at a state.  If $\outcome_\theta(s)=\mathrm{ongoing}$, each available action
$a\in\avail_\theta(s)$ has the unique successor
$s'=\trans_\theta(s,a)$.  The emission
$\emit_\theta(s,a,s')\in\Grid^*$ is the finite sequence of transient frames shown between the two settled
outputs.  These frames provide evidence about the transition, but the agent cannot act from them.
\end{definition}

\begin{figure}[t]
\centering
\resizebox{0.98\linewidth}{!}{%
\begin{tikzpicture}[
  font=\small,
  state/.style={draw=RoyalPurple!72, fill=RoyalPurple!4, rounded corners=2pt, thick,
    align=center, inner sep=5pt, text width=3.25cm, minimum height=3.85cm},
  emission/.style={draw=ForestGreen!65!black, fill=ForestGreen!4, rounded corners=2pt, thick,
    align=center, font=\scriptsize, inner sep=2.5pt, text width=2.40cm, minimum height=0.50cm},
  note/.style={draw=BrickRed!55, fill=BrickRed!3, rounded corners=2pt, thick,
    align=center, inner sep=5pt, text width=14.45cm, minimum height=0.82cm},
  arr/.style={-{Latex[length=2.2mm]}, thick, draw=black!70}
]
  \node[font=\small, align=center] at (6.675,2.75)
    {\textbf{Toy scrolling level:} The blue cell is a player needing to discover and move to the yellow goal cell.};

  \node[state] (s0) at (0,0) {};
  \node[state] (s1) at (6.675,0) {};
  \node[state] (s2) at (13.35,0) {};

  \node[font=\small, align=center] at ([yshift=1.50cm]s0.center)
    {\textbf{Initial state} $s_0^\ell$};
  \node[font=\scriptsize, align=center] at ([yshift=1.06cm]s0.center)
    {latent world offset $x=0$};
  \node[font=\small, align=center] at ([yshift=1.50cm]s1.center)
    {\textbf{State} $s_1^\ell$};
  \node[font=\scriptsize, align=center] at ([yshift=1.06cm]s1.center)
    {latent world offset $x=1$};
  \node[font=\small, align=center] at ([yshift=1.50cm]s2.center)
    {\textbf{State} $s_2^\ell$};
  \node[font=\scriptsize, align=center] at ([yshift=1.06cm]s2.center)
    {latent world offset $x=2$};

  \begin{scope}[shift={(-0.76,-0.74)}]
    \fill[black!24] (0,0) rectangle (1.52,1.52);
    \fill[white] (0,0.38) rectangle (1.52,1.14);
    \fill[NavyBlue!72] (0.38,0.38) rectangle (0.76,0.76);
    \draw[step=0.38cm, draw=black!32, thick] (0,0) grid (1.52,1.52);
  \end{scope}
  \draw[NavyBlue!60, thick, rounded corners=1pt] (-0.81,-0.79) rectangle (0.81,0.83);

  \begin{scope}[shift={(5.915,-0.74)}]
    \fill[black!24] (0,0) rectangle (1.52,1.52);
    \fill[white] (0,0.38) rectangle (1.52,1.14);
    \fill[NavyBlue!72] (0.38,0.38) rectangle (0.76,0.76);
    \draw[step=0.38cm, draw=black!32, thick] (0,0) grid (1.52,1.52);
  \end{scope}
  \draw[NavyBlue!60, very thick, rounded corners=1pt] (5.865,-0.79) rectangle (7.485,0.83);

  \begin{scope}[shift={(12.59,-0.74)}]
    \fill[black!24] (0,0) rectangle (1.52,1.52);
    \fill[white] (0,0.38) rectangle (1.52,1.14);
    \fill[NavyBlue!72] (0.38,0.38) rectangle (0.76,0.76);
    \fill[Goldenrod!75] (1.14,0.38) rectangle (1.52,0.76);
    \draw[step=0.38cm, draw=black!32, thick] (0,0) grid (1.52,1.52);
  \end{scope}
  \draw[NavyBlue!60, thick, rounded corners=1pt] (12.54,-0.79) rectangle (14.16,0.83);

  \node[font=\scriptsize, align=center] at ([yshift=-1.40cm]s0.center)
    {rendered grid\\$\render_\theta(s_0^\ell)=g$};
  \node[font=\scriptsize, align=center] at ([yshift=-1.40cm]s1.center)
    {rendered grid\\$\render_\theta(s_1^\ell)=g$};
  \node[font=\scriptsize, align=center] at ([yshift=-1.40cm]s2.center)
    {rendered grid\\$\render_\theta(s_2^\ell)=g_{\mathrm{goal}}$};

  \draw[arr] (s0.east) -- node[above, font=\scriptsize] {RIGHT}
    node[below, font=\scriptsize] {$\trans_\theta$} (s1.west);
  \draw[arr] (s1.east) -- node[above, font=\scriptsize] {RIGHT}
    node[below, font=\scriptsize] {$\trans_\theta$} (s2.west);

  \node[emission] (em01) at (3.3375,-0.92)
    {\textbf{Camera pan}\enspace$\emit_\theta$};
  \node[emission] (em12) at (10.0125,-0.92)
    {\textbf{Camera pan}\enspace$\emit_\theta$};

  \node[note] (nonmarkov) at (6.675,-2.72)
    {\textbf{Non-Markov witness.}\quad
     $\render_\theta(s_0^\ell)=\render_\theta(s_1^\ell)=g$, yet RIGHT yields $g$ from $s_0^\ell$
     and $g_{\mathrm{goal}}$ from $s_1^\ell$.  The current grid is therefore not a sufficient state.};
\end{tikzpicture}%
}
\caption{A non-Markov scrolling viewport.  Identical grids at two world offsets have different successors
under RIGHT; green boxes denote transient camera frames.}
\label{fig:rendered-moore-machine}
\end{figure}
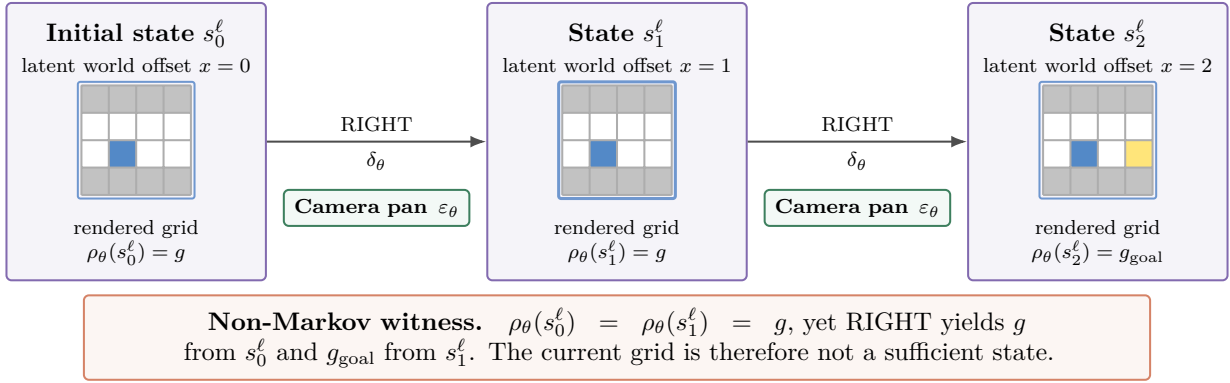

\paragraph{Determinism, hidden state, and level variation.}
ARC-AGI-3 is turn-based: the environment does not change asynchronously while the agent deliberates, and
the benchmark's environment qualification replays recorded winning and losing traces to verify faithful
re-execution \citep{arcprize2026arcagi3}.  We therefore represent the evaluated level instances with a
transition function rather than a probability distribution.  For a fixed environment version and initial
configuration, a complete underlying state and legal action determine one successor.

Determinism does not imply that the visible grid is Markov or that the agent knows the underlying state.  Every
fixed level property needed for later dynamics is part of $s_0^\ell$.  A counter, camera offset, or
off-screen object that changes during play is part of the evolving state $s_t$.  It may remain uncertain to
the agent even though the underlying transition is deterministic.  A stochastic environment could instead
use a transition kernel without changing the separation between underlying state, rendering, outcomes, and
protocol events.

The parameter $\theta$ denotes mechanics shared across the game, while the initial state $s_0^\ell$ varies
by level.  Tycho constructs its estimate
$\hat s_0^\ell=\hat\init(g_0^\ell,\ell-1)$ from the first rendered grid and the zero-based position of the
level in the game, then advances that state estimate through the recorded actions.  It does not reconstruct
state independently from every new grid.  As \Cref{fig:rendered-moore-machine} shows, doing so would erase
hidden changes whenever two points in the interaction have the same rendered frame.  The transition,
rendering, and outcome code is shared across levels, so evidence gathered earlier can reduce what must be
inferred again through scored interaction.

\paragraph{Why an outer game protocol?}
A level machine describes the mechanics within one attempt, but ARC-AGI-3 also governs resets, repeated
attempts, level advancement, and action accounting.  We represent these benchmark-wide rules as a small
outer protocol so that the same level model can be composed with the progression and scoring semantics.

\begin{definition}[Scored game protocol]
The ARC game protocol composes the level machines in order and adds attempts, recovery, advancement, and
action accounting.  A protocol configuration records the active level $\ell$, its underlying state $s$, whether the current
attempt is playable or fatal, and the scored-action counts $\mathbf a=(a_1,\ldots,a_{|\Levels|})$.
An ordinary action advances the current level machine $M_\theta^\ell$ and increments $a_\ell$.  A
$\mathrm{level\_complete}$ output closes the current level after preserving its terminal evidence, then
advances without action cost to $s_0^{\ell+1}$ or ends the game after the last level.  A
$\mathrm{game\_over}$ output closes only the current attempt.  The protocol control $\mathrm{RESET}$ is
available independently of $\avail_\theta(s)$, the set of ordinary game actions returned with the current
decision frame.  It costs one in-play action and starts a fresh attempt at $s_0^\ell$.  The
initialization RESET that creates a play is unscored.
\end{definition}

The outer protocol determines how level-machine outputs affect progression and score.  A terminal frame
documents how an attempt ended, whereas a reset or level advancement creates the next state from which an
action can be chosen.  Representing these events separately preserves this distinction and ensures that
attempts, completions, and scored actions are counted correctly.

If level $\ell$ has human baseline $h_\ell$, completion indicator $c_\ell$, and uses $a_\ell$ scored
actions, define its efficiency contribution as follows.  Under the protocol, completion follows an
ordinary action, so $c_\ell=1$ implies $a_\ell\geq 1$.
\[
e_\ell=
\begin{cases}
\min\!\left\{115,100(h_\ell/a_\ell)^2\right\},&c_\ell=1\text{ and }a_\ell>0,\\
0,&c_\ell=0.
\end{cases}
\]
Indexing levels from one, write $w_\ell=\ell$ for the official later-level weight.  Game RHAE is
\[
\min\!\left\{
 \frac{\sum_\ell w_\ell e_\ell}{\sum_\ell w_\ell},
 100\frac{\sum_\ell w_\ell c_\ell}{\sum_\ell w_\ell}
\right\}.
\]
Model identification is thus instrumental: it is valuable only when it improves weighted completion or
reduces scored interaction.

These protocol distinctions must be preserved in the evidence shown to the agent and verifier.  We
therefore record not only grids and actions, but also what role each observation played in the interaction.

\begin{definition}[Typed interaction history]
At each decision point, the history records the grid, available ordinary actions, and current
outcome.  After an action, it records the chosen action and its scored cost, any transient frames, the
successor decision output or terminal frame, and any reset or level boundary created by the protocol.
Each grid is tagged by its role: decision, transient, completion terminal, fatal terminal, reset
initialization, or next-level initialization.  A history is faithful when every recorded action is attached
to the decision frame from which it was chosen, never to a transient or terminal frame.
\end{definition}

Let $\Obs$ denote the space of these one-step records.  Their purpose is to preserve what the agent knew
when it chose an action and what the environment returned as a consequence.  This matters because identical
pixel grids can have different meanings when one is a decision frame and another is an animation or
terminal frame.  The tags prevent the model builder and verifier from fabricating transitions between
observations that were never consecutive decision frames.

\subsection{Identifiability and partial prediction}
\label{sec:partial-prediction}

Interaction can reveal relevant mechanics, but cannot distinguish internal implementations
that produce the same consequences for every future action sequence.  The operational target is therefore
to retain every distinction in the interaction history that could change what happens under a future action.
Two histories are \emph{interaction-distinguishable} if some common legal continuation yields different future
typed evidence, cost, outcome, or level progression.  A history encoding is \emph{interaction-sufficient}
when its value and a future action sequence determine these quantities.

\begin{proposition}[History-state lower bound]
Every interaction-sufficient encoding assigns different states to every pair of interaction-distinguishable
histories.  Consequently, the current rendered grid and action mask are insufficient whenever two such
histories share that visible output.
\end{proposition}

\begin{proof}
Suppose distinguishable histories $h$ and $h'$ receive the same encoded state.  Interaction sufficiency
requires every common legal continuation to induce the same typed evidence, cost, outcome, and progression
from that state.  This contradicts interaction distinguishability.  The second claim follows by choosing
$h,h'$ with equal current render and mask, as can occur with an unrendered switch, timer, inventory bit,
or off-screen structure.
\end{proof}

A useful model need not predict every visual detail.  To distinguish an abstraction from an
incorrect full-frame prediction, we allow a modeled renderer to abstain on pixels it has not determined.

\begin{definition}[Partial output prediction]
A learned renderer may return
\[
\hat\render(\hat s)\in(\{0,\ldots,15\}\cup\{\bot\})^{64\times64},
\]
where $\bot$ abstains on a pixel not determined by the model's encoded knowledge.  It may also return a
bounded set $V(\hat s)$ of observation variants for display-level ambiguity such as quantization.  The
output concretization $\Gamma_{\!O}(\hat\render,V)\subseteq\Grid$ contains the full grids agreeing with
the canonical render or a variant on every predicted cell.  For an observed grid $g$, the prediction is
\emph{accepted} when it claims at least one cell and $g\in\Gamma_{\!O}(\hat\render,V)$.  Its
\emph{coverage} is the fraction of cells claimed by the canonical render,
\[
\operatorname{cov}(\hat\render)=
\frac{1}{4096}\left|\{(i,j):\hat\render(\hat s)_{ij}\neq\bot\}\right|.
\]
\end{definition}

These are Tycho's two mechanisms for uncertain output: local abstention and a set of alternatives.  This
is intentionally simpler than maintaining a distribution over colors at every cell.  Alternatives are
used only to check predictions against evidence.  Planning uses the canonical rendering.

\begin{figure}[H]
\centering
\resizebox{0.99\linewidth}{!}{%
\begin{tikzpicture}[
  font=\small,
  arr/.style={-{Latex[length=2mm]}, thick, draw=black!65},
  checkarr/.style={-{Latex[length=2mm]}, thick, draw=ForestGreen!70!black},
  note/.style={font=\scriptsize, align=center}
]
  \node[anchor=west, font=\small\bfseries] at (0.10,-0.28)
    {(a) Abstain on newly exposed cells};

  \node[note] at (0.94,-0.88) {observed $g_t$};
  \node[note] at (3.39,-0.88) {predicted $\hat\render(\hat s_{t+1})$};
  \node[note] at (5.86,-0.88) {observed $g_{t+1}$};

  \begin{scope}[shift={(0.20,-3.00)}, x=0.245cm, y=0.245cm]
    \fill[white] (0,0) rectangle (6,6);
    \fill[black!24] (0,0) rectangle (6,1);
    \fill[black!24] (0,5) rectangle (6,6);
    \fill[black!45] (1,4) rectangle (2,5);
    \fill[Goldenrod!78] (4,2) rectangle (5,3);
    \fill[NavyBlue!76] (3,3) rectangle (4,4);
    \draw[step=1, draw=black!30, thick] (0,0) grid (6,6);
    \draw[NavyBlue!62, very thick, rounded corners=1pt] (-0.15,-0.15) rectangle (6.15,6.15);
  \end{scope}

  \begin{scope}[shift={(2.65,-3.00)}, x=0.245cm, y=0.245cm]
    \fill[white] (0,0) rectangle (6,6);
    \fill[black!24] (0,0) rectangle (5,1);
    \fill[black!24] (0,5) rectangle (5,6);
    \fill[black!45] (0,4) rectangle (1,5);
    \fill[Goldenrod!78] (3,2) rectangle (4,3);
    \fill[NavyBlue!76] (3,3) rectangle (4,4);
    \fill[black!10] (5,0) rectangle (6,6);
    \foreach \yy in {0,...,5} {
      \node[font=\scriptsize, text=black!65] at (5.5,\yy+0.5) {$\bot$};
    }
    \draw[step=1, draw=black!30, thick] (0,0) grid (6,6);
    \draw[NavyBlue!62, very thick, rounded corners=1pt] (-0.15,-0.15) rectangle (6.15,6.15);
  \end{scope}

  \begin{scope}[shift={(5.12,-3.00)}, x=0.245cm, y=0.245cm]
    \fill[white] (0,0) rectangle (6,6);
    \fill[black!24] (0,0) rectangle (6,1);
    \fill[black!24] (0,5) rectangle (6,6);
    \fill[black!45] (0,4) rectangle (1,5);
    \fill[Goldenrod!78] (3,2) rectangle (4,3);
    \fill[NavyBlue!76] (3,3) rectangle (4,4);
    \fill[Magenta!72] (5,4) rectangle (6,5);
    \draw[step=1, draw=black!30, thick] (0,0) grid (6,6);
    \draw[NavyBlue!62, very thick, rounded corners=1pt] (-0.15,-0.15) rectangle (6.15,6.15);
  \end{scope}

  \draw[arr] (1.81,-2.18) -- node[above, font=\scriptsize] {right} (2.53,-2.18);
  \draw[checkarr] (4.27,-2.18) -- node[above, font=\scriptsize] {verify} (5.00,-2.18);
  \node[note, text width=6.55cm] at (3.38,-4.48)
    {$\bot$ claims nothing about the unseen column;\\known cells remain checkable.};

  \node[anchor=west, font=\small\bfseries] at (7.05,-0.28)
    {(b) Return bounded alternatives};

  \node[inner sep=0pt, anchor=north west] (lsframe) at (7.15,-0.62)
    {\includegraphics[width=2.80cm]{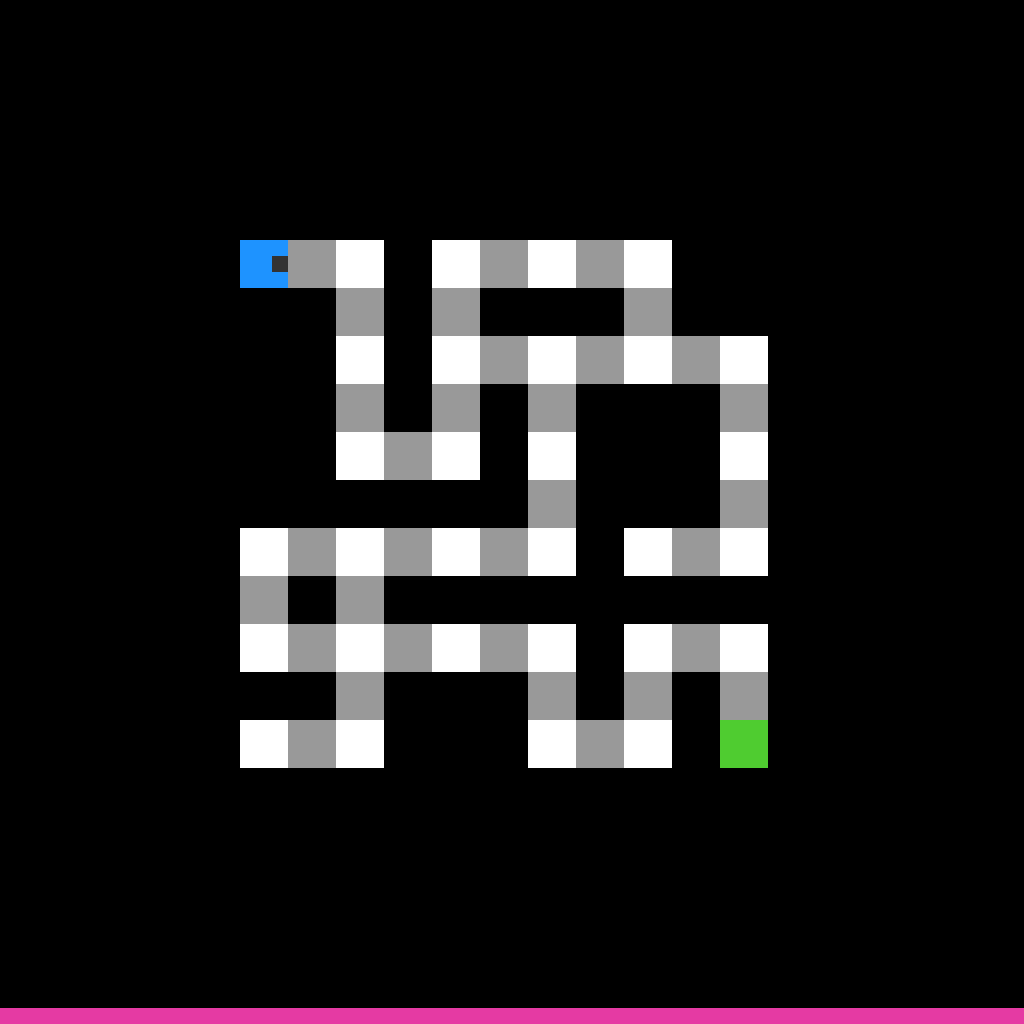}};
  \draw[BrickRed!75, very thick, rounded corners=1pt]
    (7.11,-3.37) rectangle (9.99,-3.46);
  \node[note] at (8.55,-3.67) {\texttt{tu93} start; true $M=50$};

  \node[note, font=\scriptsize\bfseries] at (12.75,-0.72) {Zoomed 64-cell bar};
  \node[inner sep=0pt, anchor=north west] (bar0) at (10.25,-0.91)
    {\includegraphics[trim=0 0 0 992,clip,width=5.00cm]{tu93_hud_t0.png}};
  \draw[BrickRed!75, thick, rounded corners=1pt]
    (10.21,-0.87) rectangle (15.29,-1.07);
  \draw[BrickRed!75, -{Latex[length=1.8mm]}, thick]
    (9.98,-3.41) to[out=10,in=190] (10.19,-0.98);
  \node[note] at (12.75,-1.25) {$n=0:\ p_0=64$ lit cells};

  \node[inner sep=0pt, anchor=north west] (bar1) at (10.25,-1.50)
    {\includegraphics[trim=0 0 0 992,clip,width=5.00cm]{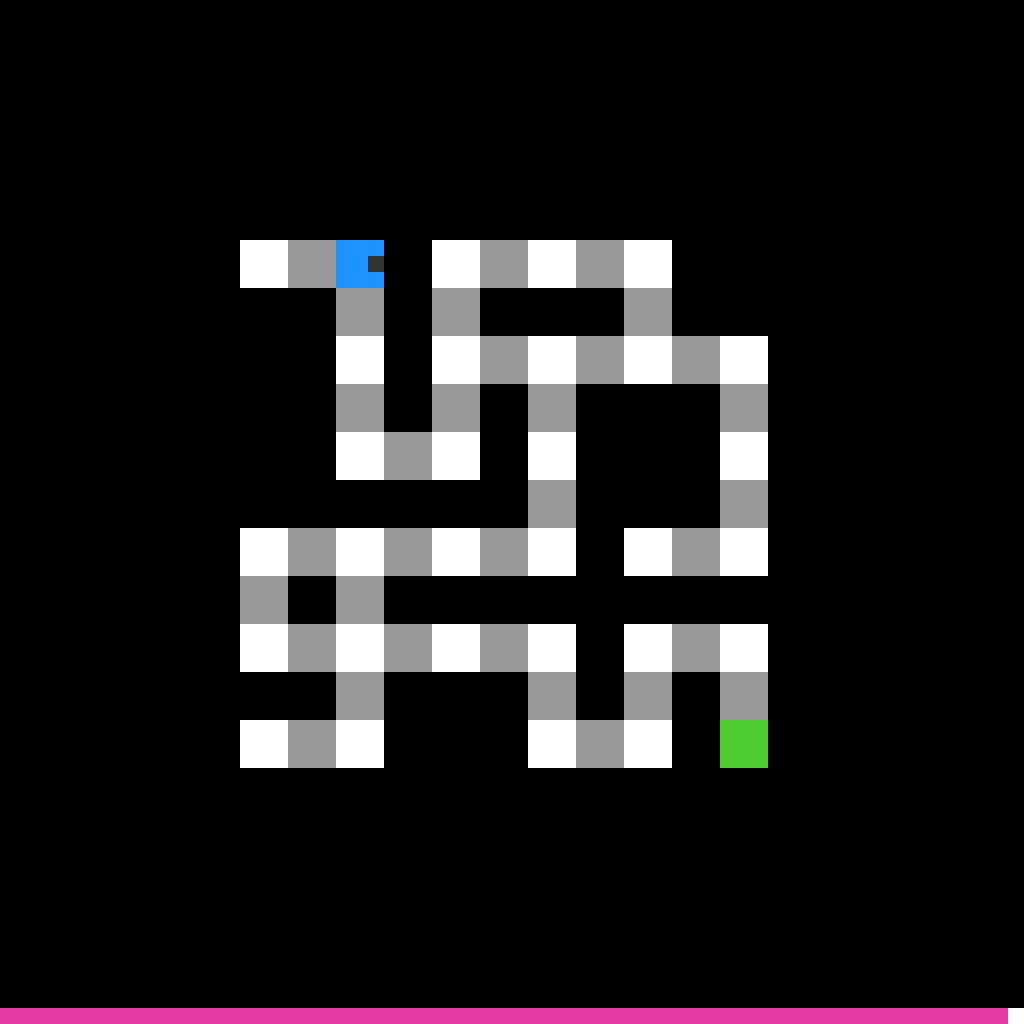}};
  \draw[black!45, thick, rounded corners=1pt]
    (10.21,-1.46) rectangle (15.29,-1.66);
  \node[note] at (12.75,-1.86) {$n=1:\ p_1=63$ lit cells};

  \node[note] at (12.75,-2.19)
    {$p_n=\operatorname{round}\!\left(64(M-n)/M\right)$};
  \node[note] at (12.75,-2.48)
    {$p_1=63\ \Longrightarrow\ M\in\{43,\ldots,127\}$};

  \node[note, font=\scriptsize\bfseries] at (12.75,-2.79)
    {Possible next outputs};

  \node[note, anchor=east] at (11.65,-3.11) {$p_2=61:$};
  \node[note] at (12.03,-3.11) {$\ldots$};
  \begin{scope}[shift={(12.33,-3.22)}, x=0.32cm, y=0.22cm]
    \fill[black] (0,0) rectangle (8,1);
    \fill[Magenta!72] (0,0) rectangle (5,1);
    \draw[step=1, draw=black!38, thick] (0,0) grid (8,1);
  \end{scope}

  \node[note, anchor=east] at (11.65,-3.46) {$p_2=62:$};
  \node[note] at (12.03,-3.46) {$\ldots$};
  \begin{scope}[shift={(12.33,-3.57)}, x=0.32cm, y=0.22cm]
    \fill[black] (0,0) rectangle (8,1);
    \fill[Magenta!72] (0,0) rectangle (6,1);
    \draw[step=1, draw=black!38, thick] (0,0) grid (8,1);
  \end{scope}

  \node[note, anchor=east] at (11.65,-3.81) {$p_2=63:$};
  \node[note] at (12.03,-3.81) {$\ldots$};
  \begin{scope}[shift={(12.33,-3.92)}, x=0.32cm, y=0.22cm]
    \fill[black] (0,0) rectangle (8,1);
    \fill[Magenta!72] (0,0) rectangle (7,1);
    \draw[step=1, draw=black!38, thick] (0,0) grid (8,1);
  \end{scope}

  \node[note] at (12.75,-4.17) {Observed next HUD: $p_2=61$};
  \node[inner sep=0pt, anchor=north west] (bar2) at (10.25,-4.37)
    {\includegraphics[trim=0 0 0 992,clip,width=5.00cm]{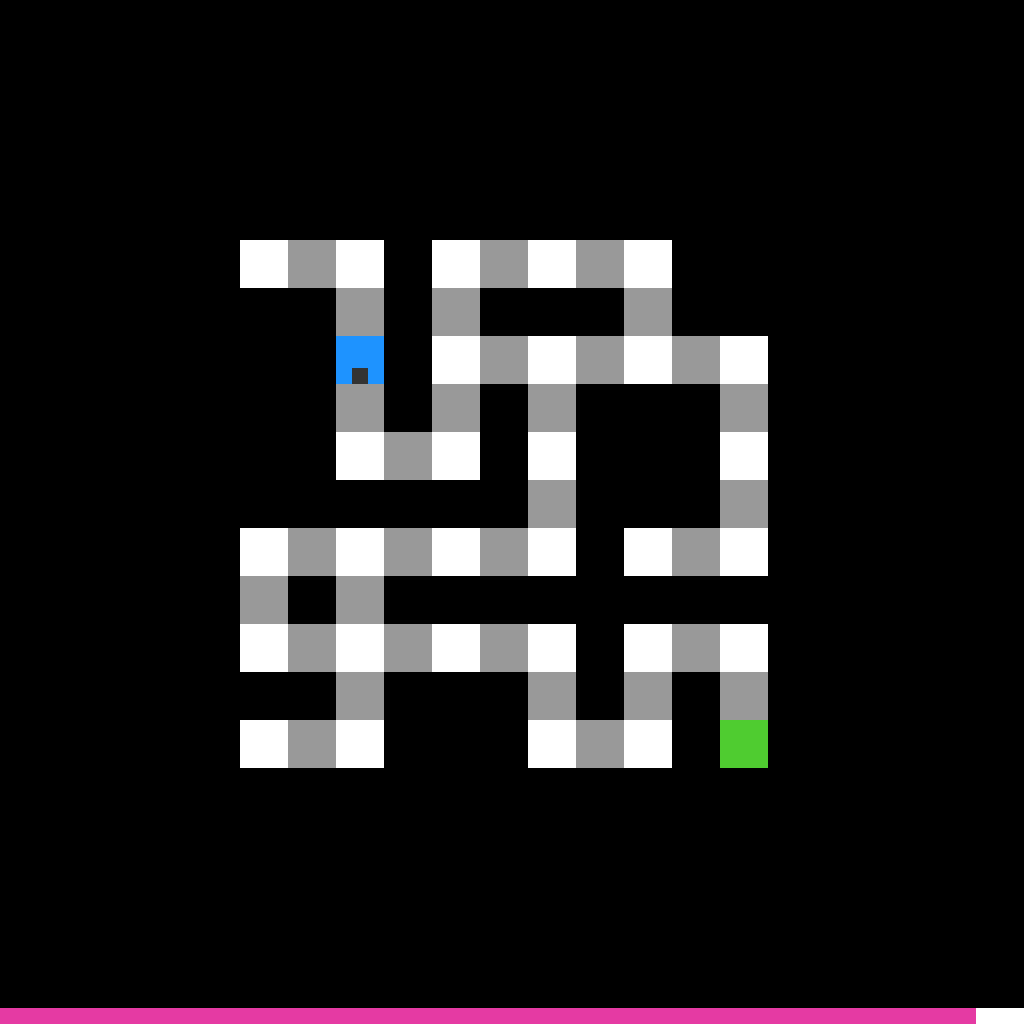}};
  \draw[ForestGreen!70!black, thick, rounded corners=1pt]
    (10.21,-4.33) rectangle (15.29,-4.53);
  \node[note, text=ForestGreen!60!black] at (12.75,-4.81)
    {Observation leaves $M\in\{43,\ldots,51\}$.};
\end{tikzpicture}%
}
\caption{Tycho's two mechanisms for uncertain output.  In (a), scrolling exposes cells the model has not
determined, so it abstains only there and receives lower coverage.  Panel (b) uses an actual
\texttt{tu93} frame.  Its 64-cell HUD quantizes a hidden move budget. After one action, three next bar
lengths remain compatible with the visible evidence.}
\label{fig:partial-output-prediction}
\end{figure}

\Cref{fig:partial-output-prediction} separates localized ignorance from finite display ambiguity.
Abstained cells make no correctness claim and reduce coverage.  Adding a complete observation variant
does not.  Panel (b) instantiates finite ambiguity with the actual \texttt{tu93} renderer.  If $M$ is
the unobserved initial move budget, then after $n$ ordinary actions it displays
$p_n=\operatorname{round}(64(M-n)/M)$ lit cells.  The observations $p_0=64$ and $p_1=63$ leave every
$M\in\{43,\ldots,127\}$ compatible, yielding $p_2\in\{61,62,63\}$.  Observing $p_2=61$ narrows the
range to $\{43,\ldots,51\}$ but still does not identify $M$.  The three complete next grids can be
returned as observation variants.  Variants describe alternative renderings of one model state, not
alternative underlying states.  Because the candidate budgets also imply different termination times, Tycho
keeps them as distinct hypotheses and plans from one selected executable hypothesis.  Neither mechanism
makes the environment stochastic.

\begin{remark}[Available-action masks]
Action legality is part of the machine output because it constrains action selection even when two rendered grids
are identical.  The public games inspected here did not vary the mask within a game, but the interface
also supports dynamic masks.  
\end{remark}

Finite histories often leave multiple hypotheses.  Tycho does not maintain an explicit version
space: the language-model policy can record competing hypotheses in its notes, choose one to encode as the
current executable model, and spend an action to distinguish them when useful.
\Cref{app:identification} formalizes this interpretation and the trade-off between information and scored
action cost.

\section{Tycho: Constructing Programmatic World Models}
\label{sec:programmatic}
\label{sec:tycho}

Tycho realizes the formal setting with four connected components.  The harness records observations and
submitted actions in a structured interaction record.  A free-form executable program represents a
revisable hypothesis about task-relevant state and dynamics.  Verification compares that hypothesis with
recorded experience, while planning exposes its consequences for future decisions.  A metareasoning
policy decides when to construct, repair, use, or bypass the model.  Only the actor can call
\code{take\_action} and thereby change the environment.

\subsection{Evidence and task memory}
\label{sec:interaction-interface}

Each game run has a separate on-disk workspace containing two kinds of information: observations and
events recorded by the harness, and notes and programs maintained by the actor or builder.  Recorded
experience is the empirical reference for verification.  A program may explain or predict that
experience, but its simulated states are not added to the interaction record.

\paragraph{Evidence interface.}
At each turn, the actor receives the current decision frame and its action mask, and may submit one control
through \code{take\_action}.  The harness saves the frame and metadata under
\code{level\_L/turn\_NNN.\{txt,png,json\}}.  Within Python code execution, the current grid is preloaded as \code{grid}, while a library (\code{wmlib}) provides parsed access to prior decision frames and ordinary transitions, archived reset attempts, completed-level terminal events, fatal events, and animation events.  Terminal, fatal, and animation frames remain separate from ordinary transitions because the actor never chooses an action from them.  The no-model policy receives the same evidence interface.  Model-using policies additionally maintain an editable world model and may verify it or search its predicted states.

\begin{figure}[H]
\centering
\includegraphics[width=\linewidth]{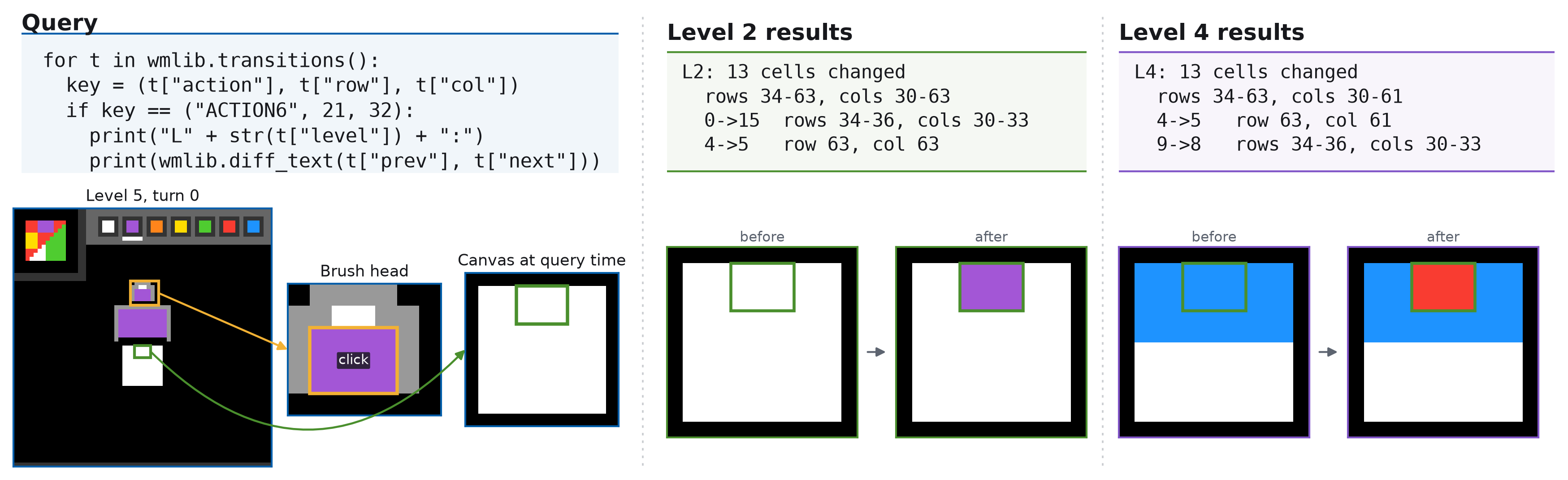}
\caption{Querying accumulated evidence in the GPT-5.6 Sol/Orchestrator run on \code{cd82}.  The displayed
filter simplifies a broader transition audit issued at the start of Level~5.  The focused view at left
identifies the brush head and canvas in the turn-0 frame.  The two matches below the results show the
brush stamping its color into the same $3\times4$ canvas region across Levels~2 and~4; each transition
also changes one HUD-counter cell.}
\label{fig:evidence-query}
\end{figure}

\paragraph{Persistent task memory.}
The workspace persists across levels of a game even though Tycho clears the conversation at each level
boundary.  Before clearing it, a consolidation pass writes level summaries.  Actor beliefs, builder
notes, helper programs, the executable model, and validated plans also remain in the workspace.  Resets
archive abandoned attempts rather than erasing them.  Independent runs use separate workspaces and never
share information across games.  Animation events have an index of selected keyframes.  The actor
loads exact frames only when needed.

\subsection{Executable hypothesis language}

A Tycho world model is a Python program that implements four required functions:
\begin{enumerate}
  \item \code{init\_state}, which maps a level's first observed frame and index to a modeled state.
  \item \code{transition}, which advances that state under an action.
  \item \code{render}, which maps the state back to an observed grid and may abstain on genuinely
  undetermined cells.
  \item \code{outcome}, which classifies the state as ongoing, level complete, or game over.
\end{enumerate}
It may also implement bounded \code{observation\_variants} for display-level quantization ambiguity and
planner-facing hooks for candidate actions, subgoals, a heuristic, a state key, or a custom planner.

The program constructs a separate start state for every level.  Layout, object positions, counters,
selected tools, and other setup values may therefore differ by level, while the transition, rendering,
and outcome functions express mechanics intended to generalize across the game.  This implements the
distinction between the level-specific state $s_0^\ell$ and shared mechanics $\theta$ introduced in
\Cref{sec:formal}.

The modeled state is deliberately game-specific.  It may store the grid directly, an object list, a
finite-state controller, a user-interface mode, a selected color or tool, a counter, geometric
constraints, or a combination of these.  Tycho therefore does not require an object-centric decomposition
in advance.  Although many ARC-AGI-3 games are naturally described in terms of objects, others are better
represented as user interfaces, cellular automata, editing tools, counters, hidden modes, or rule systems.
The agent chooses a state representation after observing the game.

Programs are not the only possible world-model representation.  Neural latent models can represent rich
continuous dynamics \citep{ha2018world,hafner2020dream}, and tree-search systems can use learned dynamics
for planning \citep{schrittwieser2020mastering}.  Programs are particularly suitable here because the
observations are discrete, many environments are algorithmic, individual prediction errors can be
localized exactly, scored actions reward reusable rules, and language-model agents can write and debug
code.  Tycho uses Python programs, but its evidence, verification, planning, and model-allocation
interfaces do not depend on that choice.

The program can carry information absent from the current grid, such as a counter, camera offset, or
selected tool, forward through the action history.  Re-parsing each frame independently would discard
this information.  When some pixels remain genuinely undetermined, the renderer may leave them unclaimed
as explained in Section~\ref{sec:formal}.  Planning operates on the model's explicit internal state,
while verification reports both accuracy on claimed pixels and prediction coverage.

\subsection{Verification and planning}

Verification and planning answer different questions.  Verification tests whether the learned program
replays interactions that have already been observed.  Planning starts from the replayed model state and
searches the program for an action sequence whose modeled endpoint is \code{level\_complete}.  When invoked
through \code{plan.py}, Tycho canonically replays the candidate and writes a validated artifact anchored
to the current model and starting frame, with an expected-frame hash for every step.  This establishes the
route's internal consistency under the learned program, not its correctness on unobserved environment
transitions.  The actor receives the validated route as conditional advice and commits one action at a
time.  On subsequent turns, Tycho surfaces the next action only while the model is unchanged, the observed
frame matches the stored hash, and the action remains available.
The automatic planner probe is separate and advisory.  It is level-local, requires an accepted replay of
the current-level prefix and no falsified current-level outcome evidence, and can surface a candidate first
action and short plan prefix.  It does not write the guarded artifact.  The actor must re-observe after
acting and re-plan from the resulting state.  Earlier-level mismatches remain warnings rather than
blocking search.
\Cref{app:identification} formalizes the stronger condition under which following a validated route
preserves the real outcome and action count.

\paragraph{Operational verification.}
Verification replays the most recent attempt for each level from its first frame and compares predicted
with recorded observations.  A transition enters the grading set if the recorded successor
differs from the recorded predecessor, or if the recorded grid is unchanged but the model predicts a
visible change.  The latter case penalizes false motion on a real no-op.  Transitions for which both
observation and model remain visually unchanged still advance the threaded model state but do not enter
this score.  Every claimed predicted cell is checked against the recorded successor.  Partial renders
therefore report both known-cell accuracy and coverage, while bounded observation variants affect only
compatibility testing.  Outcomes are verified separately because a model may reproduce motion while
misidentifying the goal or a hazard.  Recorded winning and fatal actions test whether the predicted
successor returns \code{level\_complete} or \code{game\_over} and, when available, renders the associated
terminal frame.  Ordinary decision states should remain \code{ongoing}.

\paragraph{Planning from threaded state.}
Planning starts from the state obtained by replaying the current level history through the model rather
than reparsing the latest grid, thereby preserving accumulated hidden variables.  For click-heavy games,
the lightweight planner relies on the model to propose a focused action set rather than expanding every
grid cell.

\Cref{fig:verification-planning-example} makes this data flow concrete in a toy $4\times4$ maze.  For the
recorded transition $(g_0,a_0,g_1)$, panel (a) instantiates the learned counterparts of the maps in
\Cref{sec:formal}: it constructs $\hat s_0$, applies $\hat\trans$ to obtain $\hat s_1$, renders
$\hat g_1=\hat\render(\hat s_1)$, and compares that prediction with $g_1$.  Starting from the threaded
state $\hat s_1$, the planner finds and canonically replays a four-action route to
\code{level\_complete}, recording a hash of the predicted observation after each action so that divergence
can be detected during execution.

\begin{figure}[H]
\centering
\resizebox{0.99\linewidth}{!}{%
\begin{tikzpicture}[
  font=\small,
  arr/.style={-{Latex[length=2mm]}, thick, draw=black!65},
  modelarr/.style={-{Latex[length=2mm]}, thick, draw=BrickRed!70},
  planarr/.style={-{Latex[length=2mm]}, thick, draw=ForestGreen!65!black},
  label/.style={font=\scriptsize, align=center},
  statebox/.style={draw=BrickRed!65, fill=BrickRed!4, rounded corners=2pt,
    font=\scriptsize, align=center, inner sep=4pt}
]
\newcommand{\toymaze}[4]{%
  \begin{scope}[shift={(#1,#2)}, x=0.31cm, y=0.31cm]
    \fill[white] (0,0) rectangle (4,4);
    \foreach \x/\y in {0/3,1/3,3/3,3/2,1/1,3/1,3/0}
      \fill[black!65] (\x,\y) rectangle ++(1,1);
    \fill[Goldenrod!78] (2,3) rectangle ++(1,1);
    \fill[NavyBlue!76] (#3+0.15,#4+0.15) rectangle ++(0.70,0.70);
    \draw[step=1, draw=black!35, thick] (0,0) grid (4,4);
    \draw[NavyBlue!55, thick, rounded corners=1pt] (-0.08,-0.08) rectangle (4.08,4.08);
  \end{scope}
}

\path[use as bounding box] (0,0) rectangle (16.15,4.65);

\node[anchor=north, font=\small\bfseries, align=center, text width=6.1cm]
  at (3.02,4.42) {(a) Learned transition check};
\node[label, anchor=east] at (0.80,3.18) {evidence};
\toymaze{0.95}{2.53}{0}{0}
\toymaze{4.05}{2.53}{1}{0}
\draw[arr] (2.27,3.15) -- node[above, label]{$a_0=\mathrm{RIGHT}$} (3.95,3.15);
\node[label] at (1.57,2.33) {$g_0$};
\node[label] at (4.67,2.33) {recorded $g_1$};

\node[label, anchor=east] at (0.80,1.25) {replay};
\node[statebox] (initbox) at (1.66,1.25)
  {$\hat s_0=$\\$\hat\init(g_0,\ell-1)$};
\toymaze{4.05}{0.63}{1}{0}
\draw[modelarr] (initbox.east) -- node[above, label]{$\hat\trans(\,\cdot\,,a_0)$} (3.95,1.25);
\node[label] at (4.67,0.43) {$\hat g_1=\hat\render(\hat s_1)$};
\draw[ForestGreen!65!black, thick, {Latex[length=1.7mm]}-{Latex[length=1.7mm]}]
  (5.56,1.83) -- (5.56,2.53);
\node[label, text=ForestGreen!55!black, fill=white, inner sep=1pt] at (5.56,2.18)
  {$=$};

\node[anchor=north west, font=\small\bfseries] at (6.43,4.42)
  {(b) Planner searches the threaded model state};
\node[label, text=ForestGreen!55!black] at (7.11,3.70)
  {\checkmark\ prefix accepted};
\draw[planarr] (7.11,3.58) -- (7.11,3.49);
\toymaze{6.49}{2.22}{1}{0}
\toymaze{8.49}{2.22}{2}{0}
\toymaze{10.49}{2.22}{2}{1}
\toymaze{12.49}{2.22}{2}{2}
\toymaze{14.49}{2.22}{2}{3}
\draw[planarr] (7.79,2.84) -- (8.39,2.84);
\draw[planarr] (9.79,2.84) -- (10.39,2.84);
\draw[planarr] (11.79,2.84) -- (12.39,2.84);
\draw[planarr] (13.79,2.84) -- (14.39,2.84);
\node[label] at (7.11,1.99) {threaded $\hat s_t$};
\node[label] at (9.11,1.99) {step 1};
\node[label] at (11.11,1.99) {step 2};
\node[label] at (13.11,1.99) {step 3};
\node[label, text=ForestGreen!55!black] at (15.11,1.94)
  {$\hat\outcome(\hat s_{t+4})=$\\\textbf{level complete}};

\node[font=\scriptsize, text=ForestGreen!55!black, align=center] at (11.18,0.82)
  {\textbf{Plan:}\quad RIGHT \(\longrightarrow\) UP \(\longrightarrow\) UP
   \(\longrightarrow\) UP \(\longrightarrow\) \textbf{level complete}};
\end{tikzpicture}%
}
\caption{Synthetic illustration of verification and planning.  Panel (a) connects the learned
initialization, transition, and rendering maps to the formal model in \Cref{sec:formal}.  Panel (b) starts
from the accepted current-level replay state and searches for a state classified as
\code{level\_complete}.}
\label{fig:verification-planning-example}
\end{figure}
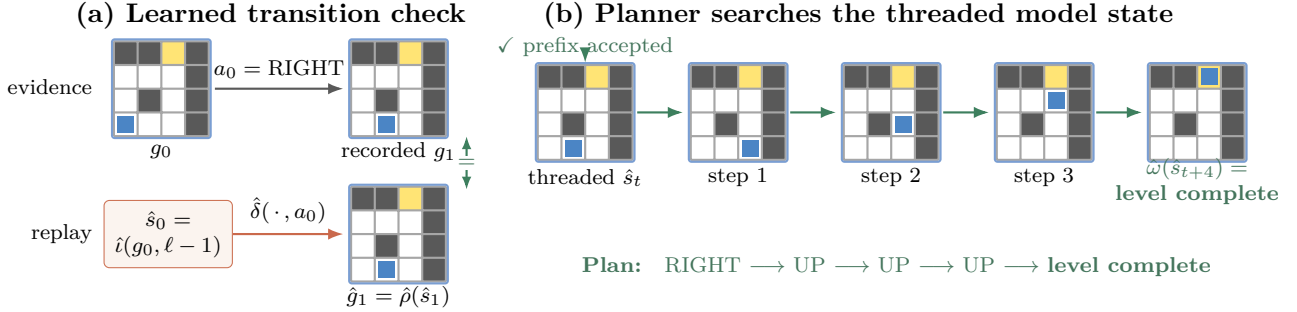

\subsection{Metareasoning over model use}

Constructing and maintaining a model consumes inference that could otherwise be spent on direct reasoning.
Tycho therefore makes this allocation policy an explicit experimental object through four benchmarkable
policies.  Under \emph{no-world-model}, the actor receives
the game interface without an executable model contract.  This is the direct-reasoning baseline for the
matched comparison.  Under \emph{single}, one actor both reasons and edits the model.  Under
\emph{orchestrator}, the actor may call a focused model-building subagent.  Under \emph{trigger},
the harness invokes the model builder when verification reports an unusable model, insufficient prediction
coverage, a transition mismatch, or an outcome inconsistency.  It also invokes the builder at new-level and
fatal-reset boundaries.  These policies test whether world modeling should remain actor-owned, be delegated
to a specialist, or be triggered automatically after verification failure.

\begin{table}[!ht]
\centering
\small
\renewcommand{\arraystretch}{1.12}
\begin{tabular}{@{}>{\raggedright\arraybackslash}p{0.17\linewidth}>{\raggedright\arraybackslash}p{0.17\linewidth}>{\raggedright\arraybackslash}p{0.31\linewidth}>{\raggedright\arraybackslash}p{0.27\linewidth}@{}}
\toprule
Policy & Model editor & When construction or repair runs & Role in the comparison \\
\midrule
No world model & --- & never & direct-reasoning baseline \\
Single & actor & at the actor's discretion & integrated reasoning and modeling \\
Orchestrator & builder & when requested by the actor & delegated specialist modeling (subagent) \\
Trigger & builder & after falsification or a lifecycle event & automatic model repair \\
\bottomrule
\end{tabular}
\caption{The four evaluated task-time model-maintenance policies.  All share the interaction record and
actor-only action protocol.  They jointly vary model ownership, specialist access, and realized inference
allocation.}
\label{tab:orchestration-policies}
\end{table}

The policies differ in who edits the executable model and what initiates an edit.  They never delegate
environment actions.  In builder-based policies, the builder may update the model and return advice, but
only the actor can act in the environment.

\FloatBarrier

\section{Evaluation}
\label{sec:evaluation}

\renewcommand{\topfraction}{0.92}
\renewcommand{\bottomfraction}{0.75}
\renewcommand{\textfraction}{0.08}
\renewcommand{\floatpagefraction}{0.80}

We evaluate Tycho on all 25 public ARC-AGI-3 demonstration games.\footnote{\url{https://arcprize.org/arc-agi/3}}
Five were randomly designated for harness development: \code{tr87}, \code{vc33}, \code{r11l},
\code{bp35}, and \code{ft09}.  Reporting the full set preserves scorecard comparability and avoids a
selectively defined evaluation subset.

We organize the analysis around three hypotheses:
\begin{enumerate}
  \item Selective executable world-model use improves completion and action efficiency over direct
  reasoning.
  \item Accurate transition prediction alone does not ensure strong gameplay: correct outcome
  inference and effective use of model advice also matter.
  \item A free-form model contract can accommodate task-specific representations, while automatic repair
  trades more exact models for additional builder calls and inference.
\end{enumerate}

We evaluate the four policies in \Cref{tab:orchestration-policies} in a matched comparison over the full public demonstration set of 25 ARC-AGI-3 games (183 levels in total) using Opus as backend model.  
We use this matched comparison to select the policy with the highest RHAE score, then evaluate that policy
with GPT-5.6 Sol and Opus~5 as separate complete-system evaluations under the reported configurations.
This section reports those first-version experiments.  Relative to the first version of this paper, the
inference costs and call counts below are restated under a corrected accounting, and one end-of-run
diagnostic has been withdrawn; \Cref{app:corrections} gives the details.  Scores, completed levels, and
action counts are unchanged.  \Cref{sec:vtwo} then repeats the four-policy comparison on an updated harness
with four current models and three runs per condition.

\paragraph{Shared configuration.}
\Cref{tab:runconfig} gives the matched and selected-policy configurations.
Games
are played against the locally cached ARC engine. Each completed trace is checked for frame, state, and
level-count agreement, then replayed through the ARC-AGI API in competition mode
\citep{arcagi3methodology}.  
Inference ceilings use a post-action boundary: the action cycle crossing the nominal limit remains in the
trace, and the game stops before the next request.  The live harness and canonical exporter use the same
accounting rule.

\subsection{Experimental results}

\begin{table}[t]
\centering
\small
\begin{tabular}{@{}l>{\raggedright\arraybackslash}p{0.35\linewidth}>{\raggedright\arraybackslash}p{0.35\linewidth}@{}}
\toprule
Parameter & Policy selection & Selected-policy evaluation \\
\midrule
Language model & Claude Opus~4.8 & GPT-5.6 Sol / Opus~5 \\
Reasoning effort & \code{xhigh} & \code{max} / \code{xhigh} \\
LM-call budget / game & 3500 & 3500 \\
Inference budget / game & \$750 & \$1{,}500 (not reached) \\
Tool steps / turn & 40 & 40 \\
Answer budget / LM call & 24{,}000 & 24{,}000 \\
Policy & four matched policies & orchestrator \\
Scoring & \multicolumn{2}{l}{official scorecards from deterministic competition replay} \\
\bottomrule
\end{tabular}
\caption{Two-stage public-set evaluation.  The first stage holds the listed configuration fixed and varies
the model-maintenance policy. The second evaluates the selected policy with two frontier models.  The
configured answer budget is not the provider request cap: requests add effort-dependent headroom because
hidden reasoning and visible output share the provider limit.  The \$750 ceiling of the matched study was
applied to the cost logged with the trajectory steps; under the corrected accounting, the retained cost of
some games exceeds it (\Cref{app:corrections}).}
\label{tab:runconfig}
\end{table}
\paragraph{Metrics.}
Our primary performance metric is \emph{mean game RHAE}: the arithmetic mean of the official per-game
RHAE scores over the 25 games.  We also report games completed, levels completed, and scored environment
actions.  RHAE uses the upper-median first-run human action count for each level as its efficiency
baseline \citep{arcagi3methodology}.  Separately, we compare each agent trajectory with the released
distribution of first-run human replays; the resulting game-balanced midrank reveals near-ceiling
performance that capped RHAE can obscure.

The model diagnostics ask three questions: can the model execute before an action, does it predict the
next rendering, and does it classify boundaries correctly?  They use the model version available
\emph{before} each action.  The \emph{pre-action
model-execution rate} is the fraction of non-RESET actions for which replay, the committed action's transition,
rendering, and a valid outcome all execute.  Among graded ordinary non-boundary actions with an executable
render, a cell is \emph{claimed} when the model predicts one of the 16 colors rather than returning
\code{-1} to abstain.  \emph{Accepted transition match} requires at least one claimed cell and requires
every claimed cell to match the next frame, using either the canonical render or one of at most five
bounded variants with the same unknown-cell mask.  \emph{Prediction coverage} is the fraction of frame
cells claimed by the canonical render; an unknown cell is an abstention, not a correct prediction.
\emph{Terminal-outcome
recall} measures correct \code{level\_complete} or \code{game\_over} predictions on recorded boundary
actions for which the pre-action model executes, while the \emph{terminal false-positive rate} measures
terminal predictions on evaluable ongoing actions.  We call an evaluable level, meaning one with at least
one graded transition, \emph{transition-exact under the accepted verifier} when every graded transition is
accepted.  This binary level summary does not require strict full-frame equality.  End-of-run model
results are labeled separately.  Strict full-frame match and other supporting diagnostics appear in
\Cref{app:diagnostics}.

A recommendation is \emph{exactly followed} when the same fully specified action, including click coordinates, is committed later in the turn in which the recommendation was surfaced.  Builder calls count explicit builder
invocations, and inference costs apply each provider's public list prices to the recorded token and cache
traffic of all calls that produced retained trajectory steps (\emph{retained-trajectory cost};
\Cref{app:corrections}).  The \emph{game-balanced human-replay midrank} compares an agent and a first-run human replay
first by levels completed and, when tied, by actions through the last completed level.  It then averages
the resulting same-game midranks uniformly over games.

\paragraph{Matched policy selection.}
The direct-reasoning \emph{no-world-model} baseline reaches $79.07$ RHAE, confirming
that a frontier model with Tycho's durable typed evidence and workspace interface is already a strong
ARC-AGI-3 agent.  The integrated \emph{single} policy reaches $85.36$, and the builder-delegating
\emph{orchestrator} reaches $88.49$.  The \emph{trigger} policy,
which calls the builder after verification failures rather than leaving model construction under direct
actor choice, reaches $83.07$: above the no-world-model baseline, but below both single and
orchestrator.  The four policies complete 19, 19, 21, and 18 games and 157, 162, 166, and 162 levels,
respectively.  Trigger produces much higher accepted transition match, but makes more builder calls, spends
more inference, and completes fewer levels.  Accurate transition prediction therefore does not by itself
ensure strong game-playing performance: when the model is built and how its advice is used also matter.
All four closed competition cards are public: \href{https://arcprize.org/scorecards/30bdf730-7aaf-49db-aae4-937df15bc5da}{no world model},
\href{https://arcprize.org/scorecards/3732640f-0e6c-44d4-8e2c-bae7c6476fad}{single},
\href{https://arcprize.org/scorecards/5477a5f0-efe9-43ae-83c1-cda8426c318c}{orchestrator}, and
\href{https://arcprize.org/scorecards/f31be13c-411d-4c32-a6c3-99c0a9c1bdbc}{trigger}.
With the orchestrator policy fixed by the Opus~4.8 study, GPT-5.6 Sol reaches $100.00$ public-set
RHAE under official scoring (\href{https://arcprize.org/scorecards/18d94e34-fee4-4fa9-9433-b6ab76c55554}{public scorecard}), completing all 25 games and all 183 levels in $7{,}766$ scored actions.
An Opus~5 run also reaches $100.00$ RHAE
(\href{https://arcprize.org/scorecards/08b98aa0-5df0-42c0-b501-856f553a21e9}{public scorecard}),
completing all 25 games and all 183 levels in $6{,}641$ actions.

The GPT-5.6 Sol run was Tycho's first evaluation with a GPT model: all preceding harness-development and
policy-selection experiments had used Anthropic models.  Reaching $100.00$ RHAE with OpenAI's GPT-5.6 Sol
therefore provides evidence that the system is not narrowly specialized to one model family or provider.
Opus~5 reaches the same score with fewer actions.

\begin{figure}[t]
\centering
\includegraphics[width=0.8\linewidth]{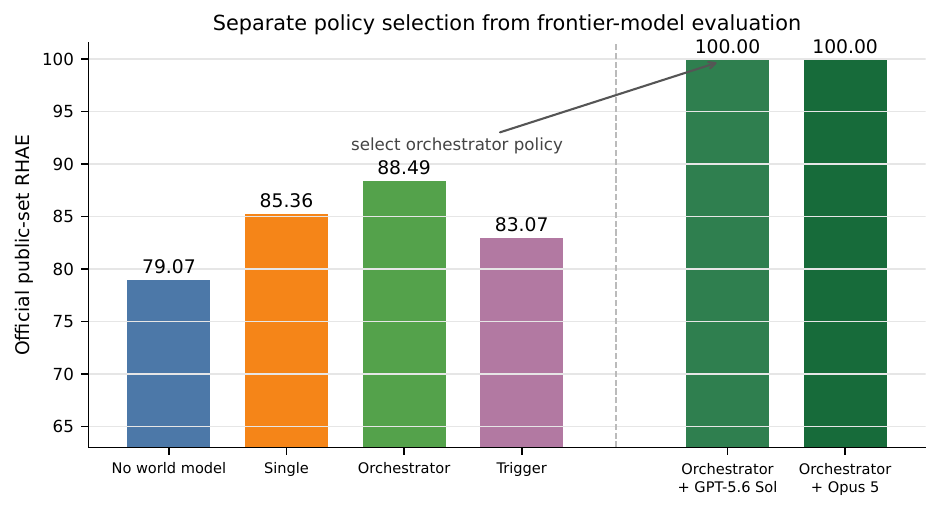}
\vspace{-0.5cm}
\caption{RHAE in the two-stage evaluation.  Four matched Opus~4.8 runs select the orchestrator policy,
which is then evaluated with GPT-5.6 Sol and Opus~5 as separate system results.}
\label{fig:policy-selection-transfer}
\end{figure}

\paragraph{Efficiency beyond capped RHAE.}
The headline score compresses substantial differences in action use.  Orchestrator completes the most
levels ($166$) while using $10{,}354$ actions, $20.3\%$ fewer than direct reasoning for nine additional
levels.  Trigger spends still fewer actions ($9{,}442$) and completes $162$ levels; single reaches the same
progress with $11{,}576$ actions.  Raw totals are not matched efficiency comparisons because the policies
reach different levels.  On the 157 levels completed by both orchestrator and direct reasoning,
orchestrator uses $7{,}187$ actions versus $8{,}857$, an $18.9\%$ reduction: it is more efficient on 81
levels, less efficient on 42, and tied on 34.  Against direct reasoning on their respective matched sets,
single reduces actions by $8.5\%$ over 153 levels and trigger by $11.9\%$ over 150 levels.  Trigger's
shorter unfinished tails are budget-censored.  It spends $1{,}361$ actions after the last completed level,
versus $4{,}140$ for direct reasoning, $2{,}403$ for single, and $1{,}842$ for orchestrator.  Every
unsuccessful matched-study game hits an imposed limit, including five trigger runs that hit the LM-call
limit (\Cref{tab:appendix-operational}).

The selected-policy GPT run moves the frontier in both dimensions: it completes 17 more levels than
the Opus~4.8 orchestrator while using $7{,}766$ actions, $25.0\%$ fewer overall.  Opus~5 also completes 17
more levels than the matched orchestrator in $6{,}641$ actions, $35.9\%$ fewer overall.  Relative to
GPT-5.6, it uses $14.5\%$ fewer actions at the same complete coverage.
RHAE only partially reflects this behavior.  Per-level efficiency is clipped at $115$, and game score
is further bounded by weighted completion.  Consequently, 114/157 direct, 133/162 single, 138/166
orchestrator, and 124/162 trigger completions already saturate the per-level cap.  The corresponding
counts are 167/183 for GPT-5.6 and 171/183 for Opus~5.  We therefore complement
RHAE with the completed-level actions, unfinished tails, and cap counts in
\Cref{fig:action-efficiency}.

\begin{figure}[!htbp]
\centering
\includegraphics[width=.9\linewidth]{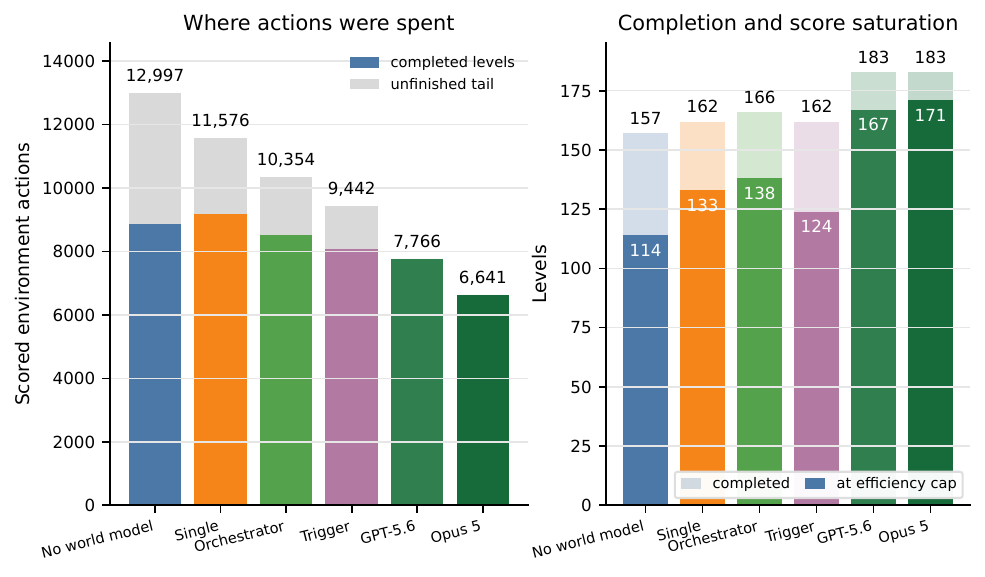}
\caption{Action use and score saturation.  Left: actions assigned to completed levels and to unfinished
tails.  Right: completed levels and the subset at the per-level RHAE efficiency ceiling.  Trigger uses the
fewest actions in the matched study, orchestrator completes the most levels, and both selected-policy
frontier runs advance the completion/action frontier.}
\label{fig:action-efficiency}
\end{figure}

\begin{figure}[!htbp]
\centering
\includegraphics[width=\linewidth]{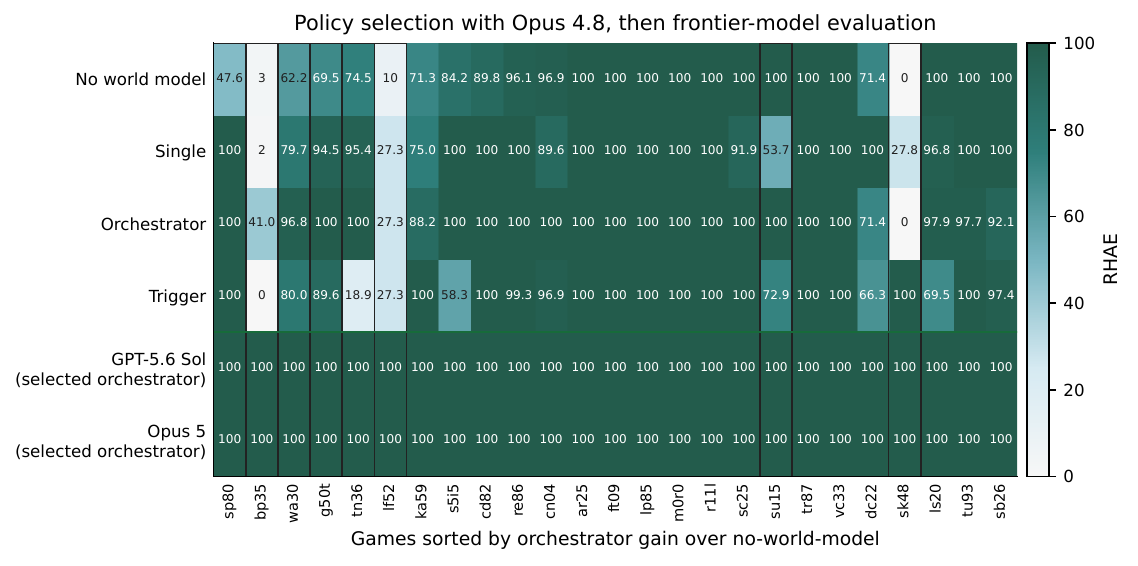}
\caption{Per-game public-set RHAE.  The first four rows are the matched Opus policy study; the separated
last two rows use the selected orchestrator policy with GPT-5.6 Sol and Opus~5.  Columns are sorted by the Opus
orchestrator's gain over direct reasoning.  Row colors match the policy color scheme used throughout the
paper, with paler cells indicating lower RHAE.  Black outlines mark games discussed as case studies or
diagnostics.}
\label{fig:mode-heatmap}
\end{figure}

\paragraph{Paired game-level effects.}
Because the public set contains only 25 games and many games saturate at $100$ RHAE, aggregate means alone
are not sufficient.  \Cref{fig:mode-heatmap} shows the saturation pattern, while
\Cref{tab:paired-deltas} reports paired per-game deltas with descriptive game-resampling intervals.
Resampling paired deltas over games asks whether a conclusion depends on the particular mix of games in
the public set.  It does not measure run-to-run model variability.  Within the observed trajectories,
orchestrator exceeds no world model by $9.42$ RHAE.  Its resampling interval excludes zero, and deleting
any one game leaves a difference between $+7.63$ and $+10.14$.  Differences among the three
model-using policies are less uniform.  Their intervals include zero because they trade wins across
games and tie on 9--12 games, often at saturation.  Thus, the observed delegation advantage is
distributed across the public games rather than driven by one game.  No model-maintenance policy
dominates every game.

\begin{table}[t]
\centering
\small
\begin{tabular}{@{}lrrr@{}}
\toprule
Paired comparison & Mean $\Delta$ RHAE & 95\% resampling interval & $+ / 0 / -$ games \\
\midrule
Single $-$ no world model & $+6.28$ & $[-0.67,\ 13.15]$ & 11 / 9 / 5 \\
Trigger $-$ no world model & $+4.00$ & $[-6.35,\ 15.68]$ & 8 / 9 / 8 \\
Orchestrator $-$ no world model & $+9.42$ & $[3.97,\ 15.76]$ & 11 / 11 / 3 \\
Orchestrator $-$ single & $+3.14$ & $[-2.61,\ 9.28]$ & 9 / 12 / 4 \\
Orchestrator $-$ trigger & $+5.43$ & $[-6.73,\ 16.44]$ & 10 / 11 / 4 \\
\bottomrule
\end{tabular}
\caption{Paired public-game deltas.  All comparisons have median zero because many games tie.  Descriptive
bootstrap intervals resample the 25 paired game deltas and measure sensitivity to benchmark composition,
not run-to-run model variability.}
\label{tab:paired-deltas}
\end{table}

\begin{figure}[!htbp]
\centering
\includegraphics[width=\linewidth]{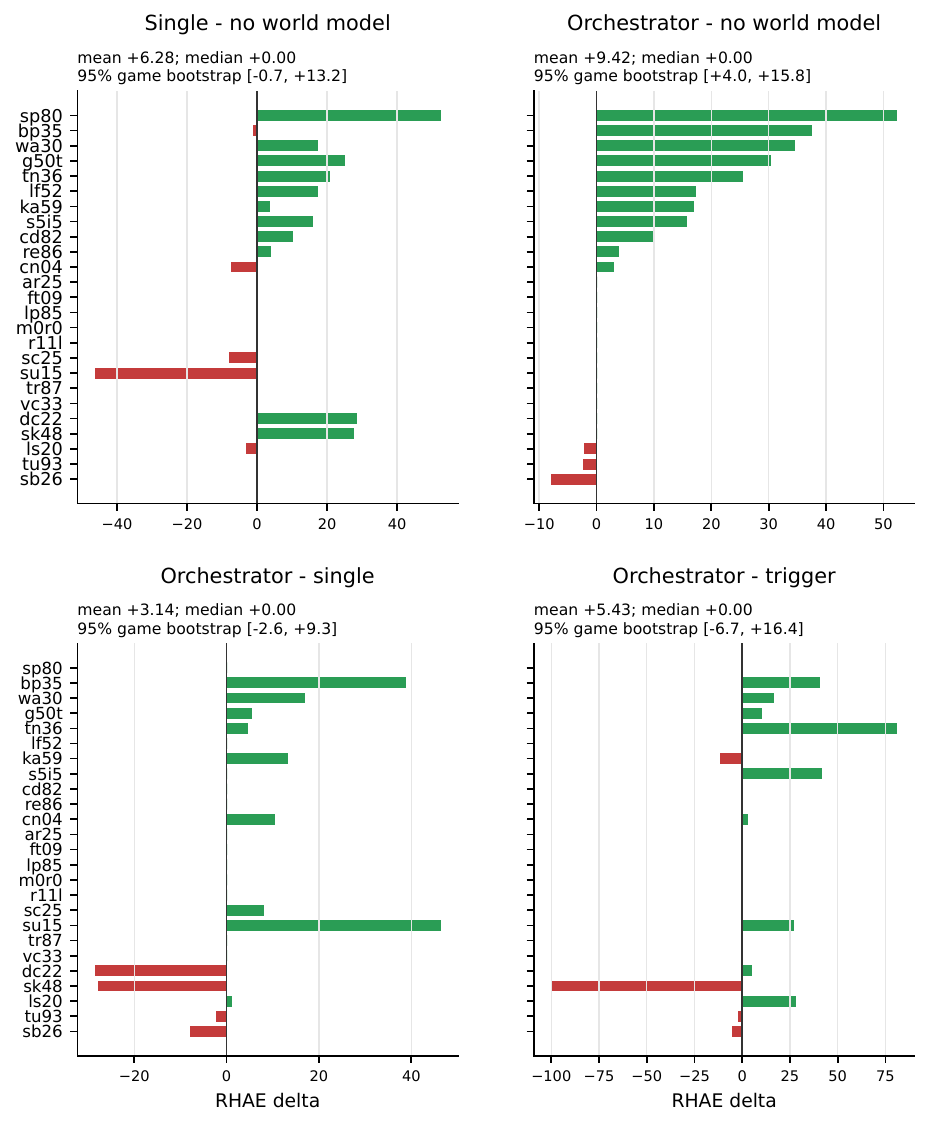}
\caption{Paired public-game RHAE deltas.  The no-world-model comparisons show large wins alongside many
saturated ties; the lower panels compare the three model-maintenance policies directly.}
\label{fig:paired-deltas}
\end{figure}

\paragraph{Human replay comparison.}
\label{sec:human-comparison}
We also compare against the ARC Prize public human replay release \citep{arcprize2026human}.  The release
contains 340 usable first-run replays.  We do not construct synthetic
25-game ``human players,'' because the public files do not expose stable participant identifiers across
games.  Instead, for each game we compare the agent to the empirical distribution of human replays for
that same game using lexicographic progress/action ordering: more completed levels is better, and among
runs with the same progress, fewer actions to the last completed level is better.  We then average the
resulting midrank uniformly over games: a strictly beaten replay contributes one, an exact tie one half,
and a replay ahead of the agent zero.  This is the expected midrank for a uniformly sampled public game
and first-run replay (it is not the percentile of one human participant across 25 games).  Human sample
sizes are uneven (10--54 usable replays per game), but each game receives equal weight in the
game-balanced midrank.

\begin{table}[t]
\centering
\small
\begin{tabular}{@{}>{\raggedright\arraybackslash}p{0.25\linewidth}
                    *{4}{>{\centering\arraybackslash}p{0.15\linewidth}}@{}}
\toprule
Policy / model & Game-balanced midrank & Pooled strict win & Top-decile games & Baseline action ratio \\
\midrule
No world model / Opus~4.8 & 79.86 & 81.18 & 13 / 25 & 1.43 \\
Single / Opus~4.8 & 83.23 & 84.12 & 14 / 25 & 1.54 \\
Orchestrator / Opus~4.8 & \textbf{86.14} & \textbf{87.06} & \textbf{17 / 25} & 1.67 \\
Trigger / Opus~4.8 & 81.97 & 82.94 & 15 / 25 & \textbf{1.75} \\
\addlinespace
Orchestrator / GPT-5.6 Sol & 98.47 & 98.53 & 24 / 25 & 2.21 \\
Orchestrator / Opus~5 & \textbf{100.00} & \textbf{100.00} & \textbf{25 / 25} & \textbf{2.58} \\
\bottomrule
\end{tabular}
\caption{Human replay comparison.  Game-balanced midrank is the exact empirical statistic over the
released first-run replays, with equal weight per game.  Pooled win is the fraction of replays strictly beaten, and ``top-decile'' counts games reaching the
empirical cutoff.  The baseline action ratio compares official human baselines with agent actions on
completed levels, so values above one mean the agent used fewer actions at its achieved progress.  The
two selected-policy frontier runs are shown separately.}
\label{tab:human-comparison}
\end{table}

The matched policies show that action efficiency and completion are distinct.  Trigger uses only
$57.2\%$ of the official human-baseline actions on the levels it completes, yet orchestrator completes
four more levels and attains the higher human midrank.  Both frontier runs are stronger on both
dimensions.  GPT-5.6 uses $45.3\%$ of the aggregate baseline over all 183 levels, reaches the empirical
top decile on 24 games, and meets or exceeds the best observed replay on 22.  Across all 183 levels,
Opus~5 uses $38.8\%$ of the aggregate baseline, reaches the top decile on all 25 games, and meets or
exceeds the best observed replay on all 25.

To place the human comparison in context, we applied the same ordering to Tycho's three orchestrator scorecards,
every nonhuman community-leaderboard entry whose linked scorecard exposes all 25 public games, and
contemporaneous systems that release equivalent 25-game progress and action data
(\Cref{tab:community-human-comparison}; \citealp{arcprize2026community}).  Official scorecards are parsed
directly.  For OPINE-World and Schema we sum released per-level action counts through the last completed
level; for Rodionov's component study we use its released ``Steps on Solved'' field, which is the same sum
\citep{courtis2026opine,schema2026,rodionov2026components}.  PRO-LONG releases 25 official single-game
scorecards for its first Fable~5 pass, while Continual Harness links a complete official scorecard
\citep{fox2026prolong,karten2026continual}.  NOOA likewise releases complete scorecards for its GPT-5.5
and GPT-5.6 Sol fleets \citep{furgale2026nooa}.

The resulting comparison changes the external reference point substantially.  Rodionov's fixed-verification
run reaches a human midrank of $96.38$, Schema's selected fallback result $95.27$, and PRO-LONG's first pass
$88.71$.  The older Rodionov row is the scorecard associated with the original executable-world-model paper
\citep{rodionov2026executable}; we identify it by the paper and method rather than its repository nickname.
Selection protocols remain material: Schema retains the better trajectory after conditionally rerunning
games, whereas the PRO-LONG row excludes the additional runs used for its $97.4$ headline.  The 25 released
first-pass scorecards average $94.71$, slightly above the paper's rounded $94.6$.  DreamTeam's row is its
released scorecard ($38.06$), not the paper's two-run mean of $38.4$
\citep{sarafian2026workspace}.  Continual Harness's released run reaches a $26.53$ human midrank at $20.54$
RHAE.  NOOA reaches human midranks of $81.89$ with GPT-5.6 Sol and $57.60$ with GPT-5.5 under its
two-hour fleet protocol.  Read-Grep-Bash exposes only one game.  We found no compatible current
public-25 artifact for Duck Harness; AERA does not provide directly comparable outcome-based counts for
the reason noted above; and graph-based exploration reports the earlier six-game preview.  These systems
are therefore not assigned a human midrank.

\begin{table}[t]
\centering
\small
\setlength{\tabcolsep}{5pt}
\begin{tabular}{@{}llrrr@{}}
\toprule
System & LLM & RHAE & Top-decile games & Human midrank$^\ast$ \\
\midrule
Tycho & Opus~5 & \textbf{100.00} & \textbf{25 / 25} & \textbf{100.00} \\
Tycho & GPT-5.6 Sol & \textbf{100.00} & 24 / 25 & 98.47 \\
Rodionov verification & GPT-5.6 Sol & 98.97 & 23 / 25 & 96.38 \\
Schema & Opus~4.8/Fable~5 & 98.98 & 20 / 25 & 95.27 \\
PRO-LONG (first pass) & Fable~5 & 94.71 & 15 / 25 & 88.71 \\
Tycho & Opus~4.8 & 88.49 & 17 / 25 & 86.14 \\
NOOA & GPT-5.6 Sol & 85.13 & 15 / 25 & 81.89 \\
OPINE-World & Opus~4.8 & 78.40 & 12 / 25 & 79.42 \\
Vision -- Continual Learning v1 & Vision Large & 63.15 & 11 / 25 & 68.28 \\
Rodionov executable WM (original) & GPT-5.5 & 63.74 & 9 / 25 & 63.08 \\
NOOA & GPT-5.5 & 50.22 & 5 / 25 & 57.60 \\
TELL & Opus~4.6 & 43.90 & 3 / 25 & 52.91 \\
DreamTeam (released run) & Opus~4.6/GPT-5.5 & 38.06 & 4 / 25 & 40.18 \\
Continual Harness & Gemini~3.1 Pro & 20.54 & 1 / 25 & 26.53 \\
a-evolve MAS Evolved & Opus~4.6 & 12.30 & 0 / 25 & 26.18 \\
OpenClaw & Opus~4.7 & 5.20 & 0 / 25 & 16.25 \\
\bottomrule
\end{tabular}
\caption{Human-replay comparison for publicly auditable 25-game trajectories using the comparison rule in \Cref{tab:human-comparison}.  Rows are sorted by human midrank, then RHAE.
They may differ in model and selection protocol; the text states the non-single-pass cases.
$^\ast$ Human midrank is shown only as cross-protocol context.}
\label{tab:community-human-comparison}
\end{table}

\paragraph{Cost and context.}
\Cref{tab:cost} reports inference cost and call counts.  For the matched study we use Anthropic's public
Claude Opus~4.8 list prices \citep{anthropicpricing2026}: \$5 per million fresh input tokens, \$25 per
million output tokens, \$0.50 per million cache-read tokens, and the five-minute cache-write rate for
writes.  The GPT transfer uses its corresponding public schedule \citep{openaigpt56sol2026}: \$5 fresh input, \$30 output, \$0.50
cache read, and \$6.25 cache write per million tokens.  
For Opus~5 we retain the same Opus-rate vector recorded by the run's budget accounting
(\$5/\$25/\$0.50/\$6.25 per million fresh-input/output/cache-read/cache-write tokens), and therefore label
its dollar value an estimate.  All first-version runs keep the list prices in force when
they ran (July 2026).
Costs are retained-trajectory costs: they cover every call that produced a step kept in the final
trajectory.  The first version of this paper summed only the calls logged with each step and thereby
missed calls on stretches replayed after a service interruption (\Cref{app:corrections}).
The four 25-game runs correspond to list-price totals of approximately \$6.11k (no world model), \$7.66k
(single), \$6.51k (orchestrator), and \$9.02k (trigger).  The last two are estimates with bounds of
\$6.24k--\$6.63k and \$8.62k--\$9.31k, because builder spend on replayed stretches is not recorded.
Orchestrator therefore obtains the highest score at an estimated cost $6\%$ above no world model ($2$--$9\%$
across the bounds), and trigger costs $39\%$ more than orchestrator.
Single's per-game costs remain heavy-tailed: its mean is \$306, but its median is \$166.
The GPT transfer is estimated at \$4.88k total, or \$195 mean and \$114 median per game.  Its largest game
cost is \$649, so the nominal \$1.5k ceiling never fires.  Despite using the stronger model, this run is
cheaper than the matched Opus model-using policies: in the logged traffic it consumes less than half
their output tokens and less than a third of their cache-read tokens.  Opus~5 is estimated at \$2.99k
total, or \$119 mean and \$97 median per game; its
largest game is \$303.  These figures exclude calls that an interruption discarded before any action was
committed, which amount to \$16 for the GPT transfer and \$150 for Opus~5 and cannot be measured for the
Opus~4.8 runs.

\paragraph{Cross-system cost normalization.}
\label{sec:cross-system-cost}
The available artifacts permit direct repricing for OPINE-World, sensitivity analysis for Rodionov, and a
partial audit for PRO-LONG, but not a dollar reconstruction for Schema.
Repricing OPINE-World's released actor and synthesis-agent counters at current Opus~4.8 API rates gives
\$12.4--15.2k across 25 games, compared with an estimated \$6.51k for Tycho's orchestrator.  The range is necessary
because the archive reports aggregate cache-creation tokens but not their TTL split (five minute or one hour rate).  On this normalized same-model comparison, Tycho
achieves the higher score at $43$--$53\%$ of the list-price inference cost.  

Directly converting the $90.75$M and $103.49$M cost tokens for Rodionov's GPT-5.6 Sol verification runs
gives \$2.72k and \$3.10k under the
observed subscription-backed cache mix, in which 97--98\% of input is recorded as cache reads and cache
writes are not reported.  The authors caution that preliminary API-key tests had lower cache-hit rates and
cost roughly five times this projection, yielding empirical reproduction sensitivities of approximately
\$13.6k and \$15.5k \citep{rodionov2026artifacts,rodionov2026components}.  These are $2.8$--$3.2\times$
Tycho's \$4.88k selected GPT transfer and $4.6$--$5.2\times$ its \$2.99k Opus~5 run.  Under Rodionov's
subscription-backed cache mix, however, the direct \$2.72k--\$3.10k projection is comparable to the
Opus~5 estimate.

PRO-LONG weights Fable~5 output, cache reads, and
one-hour cache writes at $5$, $0.1$, and $2$ fresh-input tokens, matching the public \$10/\$50/\$1/\$20
per-million schedule \citep{fox2026prolong,anthropicpricing2026}.  Its reported 150M billed-token
equivalents therefore yield \$1.50k for the first pass, while the selective best-of-up-to-two protocol is
reported at \$1.75k.  The latter is $1.7\times$ lower than Tycho's \$2.99k Opus~5 estimate.  This is a plausible
inference-cost advantage, helped by PRO-LONG's ability to execute up to 20 queued actions between model
invocations, whereas the first-version Tycho harness evaluated here observes and checks every committed
action (the updated harness of \Cref{sec:vtwo} commits guarded multi-action sequences).  The exact total cannot be
independently reconstructed because the additional selected trajectories and their token telemetry are
not included in the released artifacts \citep{fox2026prolongartifacts}.  Its different model,
2,000-action allowance, and selective
rerunning also preclude a controlled cost comparison.

\begin{table}[t]
\centering
\small
\begin{tabular}{@{}lrrr@{}}
\toprule
Policy / model & LM calls & API-equiv.\ cost & Mean / median per game \\
\midrule
No world model / Opus~4.8 & 24.5k & \$6.11k & \$244 / \$112 \\
Single / Opus~4.8 & 26.7k & \$7.66k & \$306 / \$166 \\
Orchestrator / Opus~4.8 & 26.9k & \$6.51k$^{\dagger}$ & \$260$^{\dagger}$ / \$169 \\
Trigger / Opus~4.8 & 50.3k & \$9.02k$^{\dagger}$ & \$361$^{\dagger}$ / \$329$^{\dagger}$ \\
\addlinespace
Orchestrator / GPT-5.6 Sol & 29.1k & \$4.88k & \$195 / \$114 \\
Orchestrator / Opus~5 & 18.7k & \$2.99k & \$119 / \$97 \\
\bottomrule
\end{tabular}
\caption{Call counts and model-specific list-price cost of the first-version runs under
retained-trajectory accounting (July 2026 list prices; calls discarded by interruptions excluded).
$^{\dagger}$Estimates, because builder spend on replayed stretches is not recorded: the totals lie within
\$6.24k--\$6.63k and \$8.62k--\$9.31k, the per-game means within \$250--\$265 and \$345--\$372, and the
trigger median within \$326--\$329.  The per-category token totals of the first version covered logged
calls only and are omitted.  \Cref{tab:cost-corrected} compares first-version and corrected values.}
\label{tab:cost}
\end{table}

\begin{figure}[!htbp]
\centering
\begin{minipage}[t]{0.49\linewidth}
  \centering
  \includegraphics[width=\linewidth]{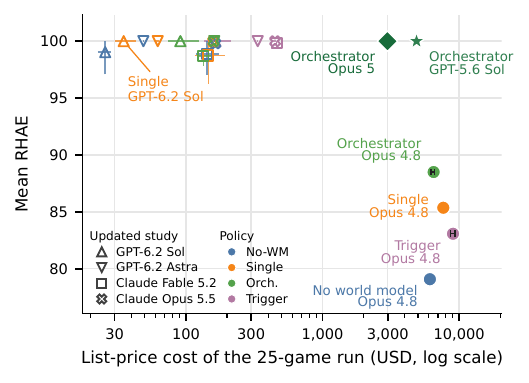}\\[-0.4em]
  {\small\textbf{(a)}}
\end{minipage}\hfill
\begin{minipage}[t]{0.49\linewidth}
  \centering
  \includegraphics[width=\linewidth]{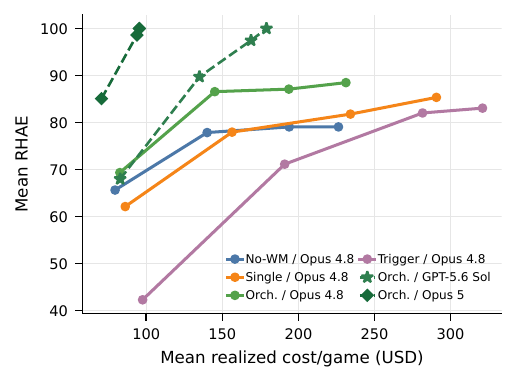}\\[-0.4em]
  {\small\textbf{(b)}}
\end{minipage}
\caption{Compute--performance tradeoffs.  \textbf{(a)} Mean public-set RHAE versus the list-price
cost of the 25-game run (log scale).  Filled markers are the six first-version runs at their corrected
cost (\Cref{tab:cost}).  Open markers are the 16 conditions of the updated study (\Cref{sec:vtwo}) at the
mean over three runs, with marker shape giving the model and colour the policy; thin bars span the
minimum and maximum over the runs.  \textbf{(b)} RHAE of canonical trace prefixes of the first-version
runs at \$100, \$250, \$500, and \$750 ceilings (left to right).  Solid curves are matched-study policies
and dashed curves are frontier runs; horizontal position includes the crossing action cycle; these are
prefix analyses, not reruns.  Panel (b) still uses the per-step logged costs of the first version, which
overstates scores at the lower ceilings (\Cref{app:corrections}); no score at the \$750 ceiling changes.}
\label{fig:cost-analysis}
\end{figure}

\paragraph{Sensitivity to per-game cost limits.}
We replay each recorded trajectory only through the action cycle that crosses a nominal \$100, \$250,
\$500, or \$750 per-game limit, then recompute its score.  The agents were not rerun or told about these lower
limits, so the curves show how much of the final score had accumulated by each point.
The right panel of \Cref{fig:cost-analysis} shows two main patterns.  Within the matched study, at \$100
direct reasoning leads single and trigger because model construction has not yet paid off.  Orchestrator
reaches $86.57$ RHAE and 163 completed levels by the \$250 limit, then gains only three levels and $1.92$
RHAE by \$750, while trigger needs more inference to recover from its builder-heavy start.  The frontier
traces saturate faster: Opus~5 reaches $98.62$ at \$250 and $100$ by \$500, while GPT-5.6 reaches $89.75$
at \$250 and $100$ by \$750.  The curves use the first version's logged costs; under the corrected
accounting scores at the lower ceilings are somewhat lower, and none changes at \$750.

\paragraph{Models available before action.}
The pre-action diagnostic evaluates the Python model available before each committed action using only
preceding history.  The resulting pre-action model-execution rates are $89.3\%$, $83.5\%$, and $52.0\%$
for single, orchestrator, and trigger, respectively.  If no usable model version is available in the run record, the model is counted as unavailable for that action. %
These diagnostics do not read the level-terminal or fatal frames affected by the recording error
described in \Cref{app:corrections}; recomputing them with the corrected frames leaves every value
unchanged.

\begin{table}[t]
\centering
\small
\setlength{\tabcolsep}{4pt}
\begin{tabular}{@{}lrrrrrr@{}}
\toprule
Policy & Model exec. & Accepted trans. & Coverage & LC recall ($n$) & GO recall ($n$) & Terminal FP \\
\midrule
Single & 89.3 & 13.9 & 99.7 & 21.1 (147) & 0.0 (12) & 0.97 \\
Orchestrator & 83.5 & 16.2 & 100.0 & 20.0 (155) & 0.0 (6) & 0.58 \\
Trigger & 52.0 & 88.1 & 99.3 & 77.7 (112) & 0.0 (4) & 0.59 \\
\bottomrule
\end{tabular}
\caption{Pre-action model diagnostics (percent), using the definitions above.  Accepted transition
match and coverage are micro-averaged over evaluable ordinary transitions; LC is
\code{level\_complete}, GO is \code{game\_over}, and FP is the false-positive rate.}
\label{tab:online-model-diagnostics}
\end{table}

The ordering rules out transition accuracy as a sufficient explanation of game-playing performance.
Trigger attains $88.1\%$ accepted transition match and recognizes most completion boundaries, yet scores
below single and orchestrator, whose accepted transition match is $13.9\%$ and $16.2\%$.  None predicts the small
set of observed deaths in the evaluable prefixes.  Single and orchestrator can still benefit from
structured scratch work, selective model use, later repair, and the actor's own reasoning.  These
pre-action diagnostics establish when an evaluated model was available and what it predicted. They do not
isolate its causal effect on the next action.  Macro-averaged, missing-data, strict-match, and repair-
recovery checks in \Cref{app:diagnostics} preserve the same qualitative contrast.

\begin{table}[t]
\centering
\small
\begin{tabular}{@{}lrrrr@{}}
\toprule
Policy & Builder calls & Recommendations & Exact followed & Action-name followed \\
\midrule
Single & 0 & 29 & 25 & 25 \\
Orchestrator & 147 & 68 & 35 & 54 \\
Trigger & 1{,}192 & 970 & 552 & 748 \\
\addlinespace
GPT transfer & 660 & 644 & 634 & 636 \\
Opus~5 & 130 & 129 & 109 & 122 \\
\bottomrule
\end{tabular}
\caption{Builder activity and behavior after pre-action model advice.  Exact following includes click
coordinates; action-name following compares only the action type.  Single recommendations are candidate
first actions from automatic planner probes, while the other rows use builder reports.  Builder calls are
those with a decision record in the trace; the harness counters report 728 and 162 builder runs for the
two frontier rows, the difference lying on stretches replayed after an interruption, which carry no
decision record.}
\label{tab:planner-adherence}
\end{table}

\paragraph{World-model and planner diagnostics.}
End-of-run diagnostics show the same separation.  The end-of-run models are transition-exact on 25 of 137
evaluable completed levels under single ($18.2\%$), 62 of 155 under orchestrator ($40.0\%$), and 154 of
162 under trigger ($95.1\%$), although trigger scores below both other model-using policies.  Full replay
in \Cref{tab:appendix-rendering,tab:appendix-outcomes} reports $99.97\%$ accepted transition match for
trigger, versus $45.56\%$ for orchestrator and $22.24\%$ for single.  These end-of-run measurements cannot
determine how earlier model versions affected action, but they reinforce that a more accurate simulator is
not necessarily a better game-playing policy.

The executable pathway also remains active in the frontier runs.  GPT-5.6's end-of-run summaries report
a transition-exact model on 182 of 183 completed levels ($99.5\%$), 660 builder calls, and planner-bearing
diagnostics on 34 completed levels.  Of 644 parseable builder recommendations surfaced before a GPT action,
the actor commits the exact recommended action 634 times in the same turn (636 match at the action-name
level).  This close temporal alignment shows that builder advice is usually carried into the next action,
but it does not isolate a causal effect because actor and builder share the same evidence.  Opus~5 reports
transition-exact end-of-run models on 150 of 183 completed levels ($82.0\%$), 130 builder calls, and
planner-bearing diagnostics on 32 levels.  It exactly follows 109 of 129 surfaced builder recommendations
(122 match at the action-name level).

\ArxivPlanTableVone

\paragraph{Plan-level statistics.}
\Cref{tab:plan-stats-v1} counts the plans that the harness planner returned in the matched study and
compares each, in order and from the frame on which it was computed, with the actions executed afterwards
on the same level.  The harness does not record that an action was taken from a plan, so these counts
measure agreement between plans and actions, not causation, and search code written by the agent itself
is not counted.  The counts are lower bounds because the first-version logs print at most the first eight
actions of a plan, truncate part of the builder output, and lack decision records on replayed stretches.
Within these limits the planner contributed little in the matched study.  Single received 25 plans on 16
of 168 started levels, 9 of which agreed with the actions up to the end of the level (7 of them inferred
from the level ending at exactly the predicted length), and orchestrator's builder returned 11 readable
plans, none of which was followed to the end of a level.  Trigger returned 198 plans on 81 of 169 levels; for 78 of them (36 inferred) the level ended while the
actions still agreed with the plan, and for 48 agreement ended after a recorded contradiction.  The share of all executed actions that
agree with a returned plan is $0.9\%$, $0.2\%$, and $5.9\%$ for single, orchestrator, and trigger.  The
same analysis gives $59.9\%$ for the GPT-5.6 Sol transfer run and $17.6\%$ for Opus~5.  \Cref{sec:vtwo}
reports the corresponding statistics for the updated harness.

\begin{figure}[!htbp]
\centering
\includegraphics[width=\linewidth]{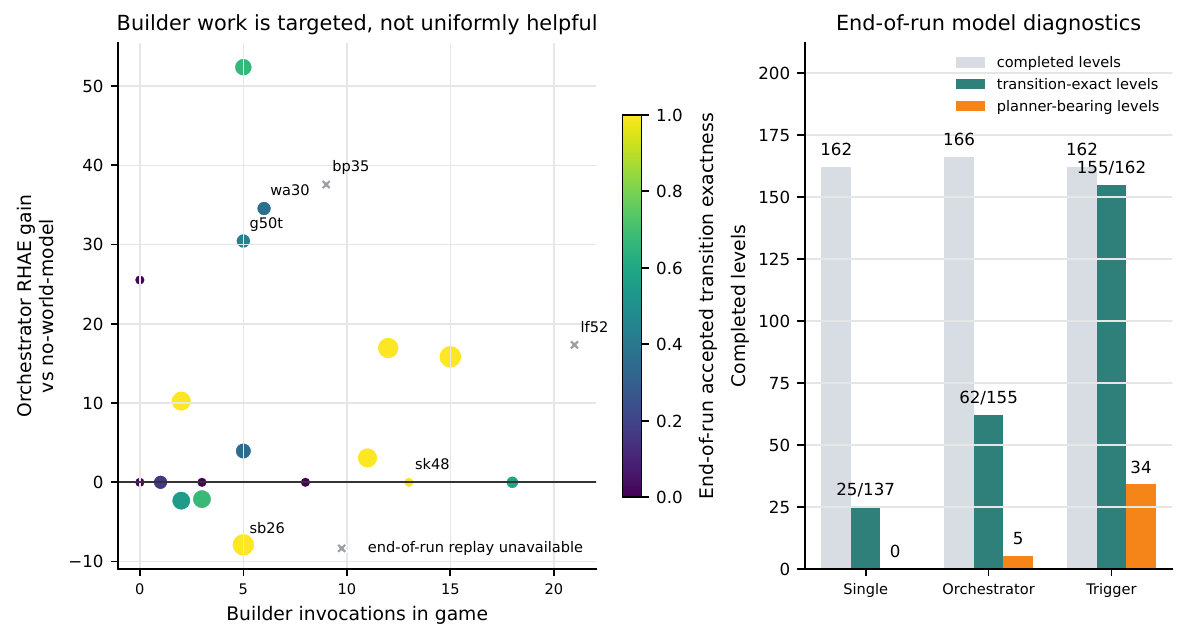}
\caption{World-model diagnostics for the matched Opus runs.  Left: builder calls versus orchestrator's
RHAE gain over direct reasoning; color encodes end-of-run accepted transition match, marker size
transition-exact levels, and crosses unavailable end-of-run replays.  Right: completed, transition-exact, and planner-bearing
levels for each model-using policy.}
\label{fig:wm-diagnostics}
\end{figure}

\paragraph{Hypothesis verdict.}
The evidence supports the three hypotheses with different strength.  First, the observed matched runs are
consistent with a benefit from selective executable-model use.  Single and orchestrator score above direct
reasoning, whereas automatic trigger repair spends more inference without attaining the best score.  Second,
accurate transition prediction is insufficient for strong gameplay.  Terminal-outcome recall remains
weaker, \code{sk48} models mechanics without the objective, and trigger combines the highest accepted
transition match with lower RHAE.  Third, the case studies illustrate how the free-form contract accommodates
task-specific combinations of representational devices, while
automatic repair trades additional builder calls and inference for better transition match.  The GPT result
provides evidence of cross-family transfer of the selected policy.  Opus~5 reaches the same public-set
score with fewer actions.  Each configuration in this section was run once, and the policies are complete
designs that differ in builder access, scheduling, and realized inference.  The ordering among the three
model-using policies and the contrast between transition match and score are therefore observations
about single runs of these designs; \Cref{sec:vtwo} repeats the comparison with three runs per condition.

\FloatBarrier

\subsection{Executable Models in Use}
\label{sec:case-studies}

Aggregate diagnostics do not show what an executable model contributes within a trajectory.  We therefore
inspect two late-level successes with complementary state representations, one failure in which model
construction became counterproductive, and one transfer-run episode that exposes repair and replanning.
To illustrate exact replay and model-based planning, we use \code{cd82}, where the canonical renders
match all recorded graded transitions exactly and the model predicts the level-completion outcome.
We then deliberately examine two different levels of \code{tu93} to show how an executable model can support search with incomplete rendering and be locally repaired when a previously unobserved interaction reveals a missing transition.
Finally, \code{sk48} is used to show that correct mechanics are insufficient when the objective remains unresolved.

\paragraph{Exact Replay and Planning with an Executable Model.}
In the matched Opus~4.8 Orchestrator run, \code{cd82} asks the agent to reproduce a target
pattern by moving one of two paint droppers over a canvas and selecting colors.  The world model
represents the canvas at its logical resolution, the droppers and their orientations, the selected color,
and the target pattern.  Its transition program composes the different paint operations, including
rotated droppers and overlapping colors.  Before acting in level~4, the model's canonical renders strictly
matched all 38 graded transitions from levels~0--3, rendered the new initial state exactly, and simulated a 13-action
program to the goal.  The actor accepted the first recommendation and then executed the entire sequence,
completing the level on its first attempt.  \Cref{fig:case-cd82} shows four intermediate simulated states and
reports their agreement with the corresponding observations.  The case illustrates how an executable model can replace trial and
error with a verified sequence.

\begin{figure}[!htbp]
  \centering
  \includegraphics[width=\linewidth]{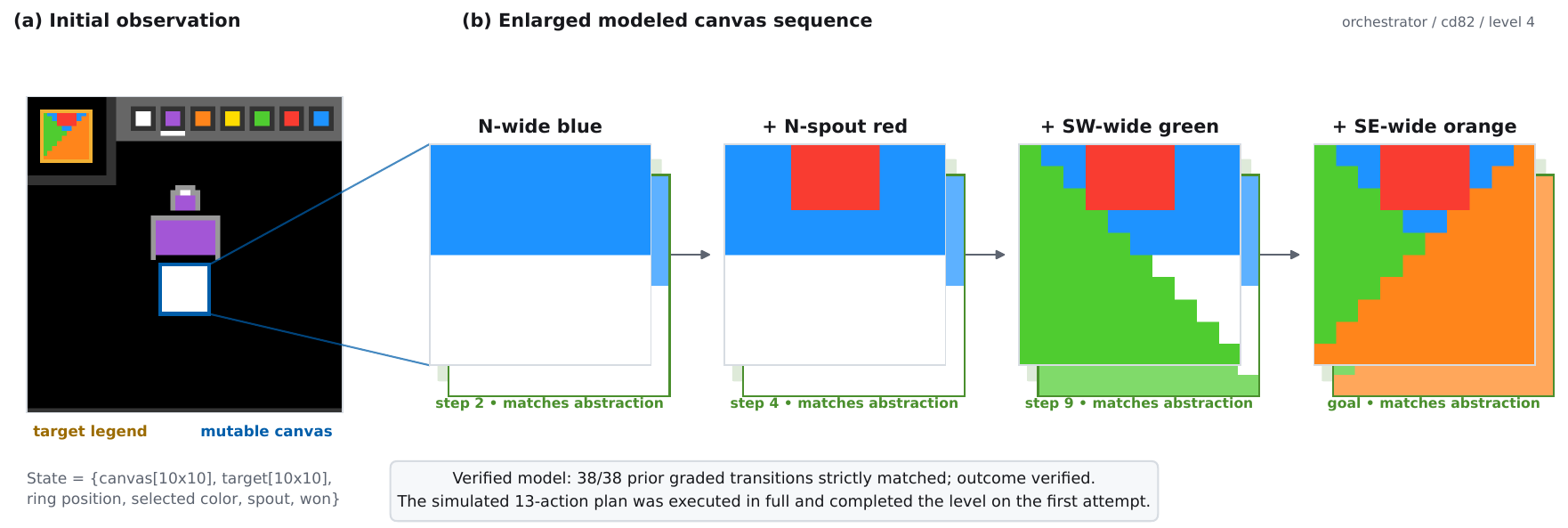}
  \caption{Planning with an Executable Model in \code{cd82}, level~4.  From left to right, the initial
  observation identifies the target and mutable canvas, the canvas is enlarged, and successive modeled
  states apply the planned actions until the final state reaches the target and returns
  \texttt{level\_complete}.  The diagonally offset cards behind the modeled states show the corresponding
  observed crops, with green shading marking their agreement.}
  \label{fig:case-cd82}
\end{figure}

\paragraph{Search over a partially rendered state.}
In the matched Opus~4.8 Single run, level~6 of \code{tu93} uses a different abstraction.  The
$64\times64$ rendered grid encodes a $9\times13$ room graph because each logical cell occupies a block
of pixels.  The modeled state tracks passability, player and goal positions, facing, hazards, move
budget, and death.  Following the partial-prediction contract in
\Cref{sec:partial-prediction}, the renderer abstains on the uncertain remaining-moves HUD row, yielding
$100\%$ accepted transition match at $98\%$ prediction coverage.  Predicting the HUD is unnecessary
for the route.  Over this state, A* expands 189 nodes and returns a 13-action plan that deliberately
removes a red hazard before reaching the goal.  The actor follows the recommended first action, and the
subsequent recorded trajectory matches the plan through completion (\Cref{fig:case-tu93}).  Thus a reduced-state
model can support reliable search without claiming every display pixel.  In this first-version harness,
Tycho still invokes a language model around every committed action.  The updated harness of
\Cref{sec:vtwo} instead lets the actor commit several actions under a guard that interrupts the sequence
on a prediction mismatch.

\paragraph{Repair after a previously unobserved interaction.}
In the GPT-5.6 Sol Orchestrator transfer run, level~8 of \code{tu93} combines moving orange patrols with
two directional red hazards and a trail-following chaser.  The model had explained all prior evidence and
predicted that the player could enter an orange patrol's cell because the patrol would move away.  The
action was safe, but for a different reason:
animation frames show the perpendicular contact expanding the patrol into an explosion ring and removing
it before the other entities move (\Cref{fig:case-tu93-repair}a--b).  The resulting frame therefore
contradicted the assumed patrol transition while revealing a more general combat rule.

After the builder encoded perpendicular patrol removal, the model's canonical renders strictly matched
all 168 graded transitions, including all 17 from the current level.  From the corrected turn-17 state, exact-state BFS
visited 33 states and validated a shortest 12-action continuation.  It removes the two red hazards in
sequence and then enters the green goal, whose completion rule had already been verified on levels~0--7.
The actor took the recommended first action and subsequently followed the full route to the recorded
completion (\Cref{fig:case-tu93-repair}c--d).  Here verification does more than score the model: it
localizes a mismatch, after which the repaired state can immediately support planning.

\begin{figure}[!htbp]
  \centering
  \includegraphics[width=\linewidth]{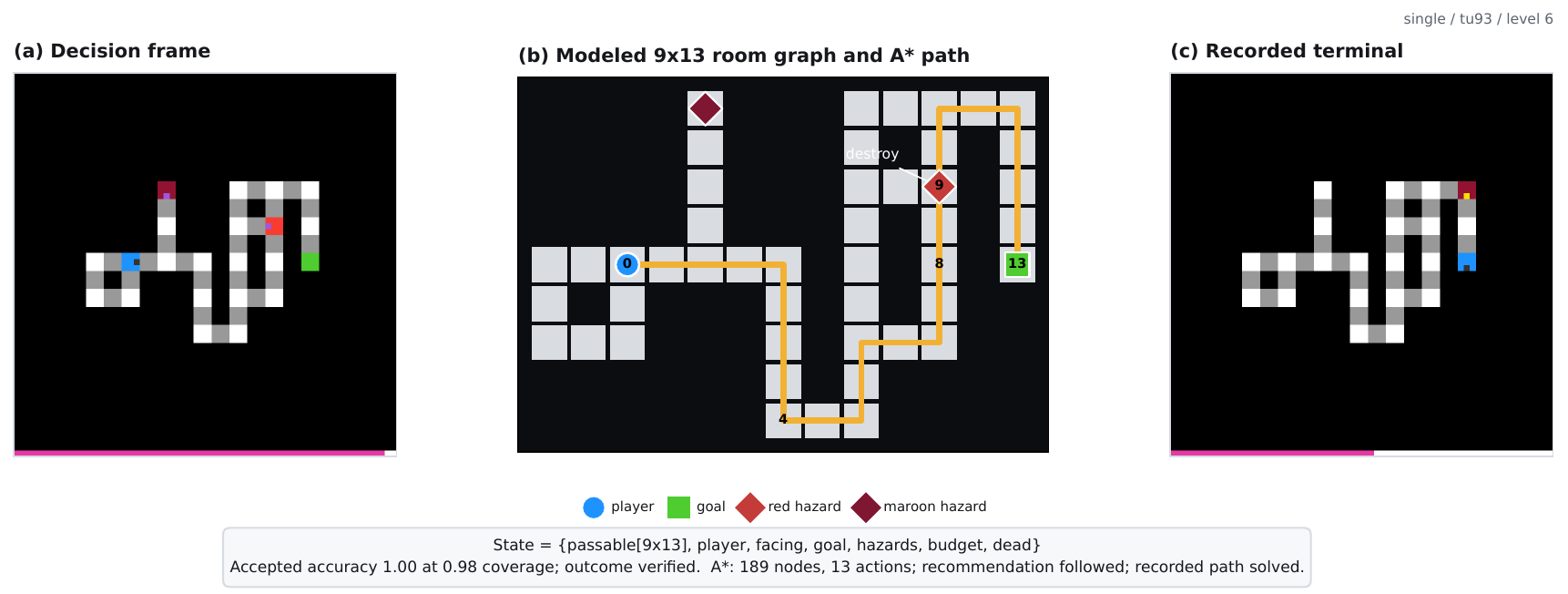}
  \caption{Planning over a reduced state representation in \code{tu93}, level~6.  The $64\times64$ observation
  is parsed into a $9\times13$ room graph in which diamonds mark hazards and the gold line is the
  13-action A* plan.  The plan removes the red hazard on its way to the goal, and the recorded trajectory
  follows it to completion.}
  \label{fig:case-tu93}
\end{figure}

\FloatBarrier
\begin{figure}[!htbp]
  \centering
  \includegraphics[width=\linewidth]{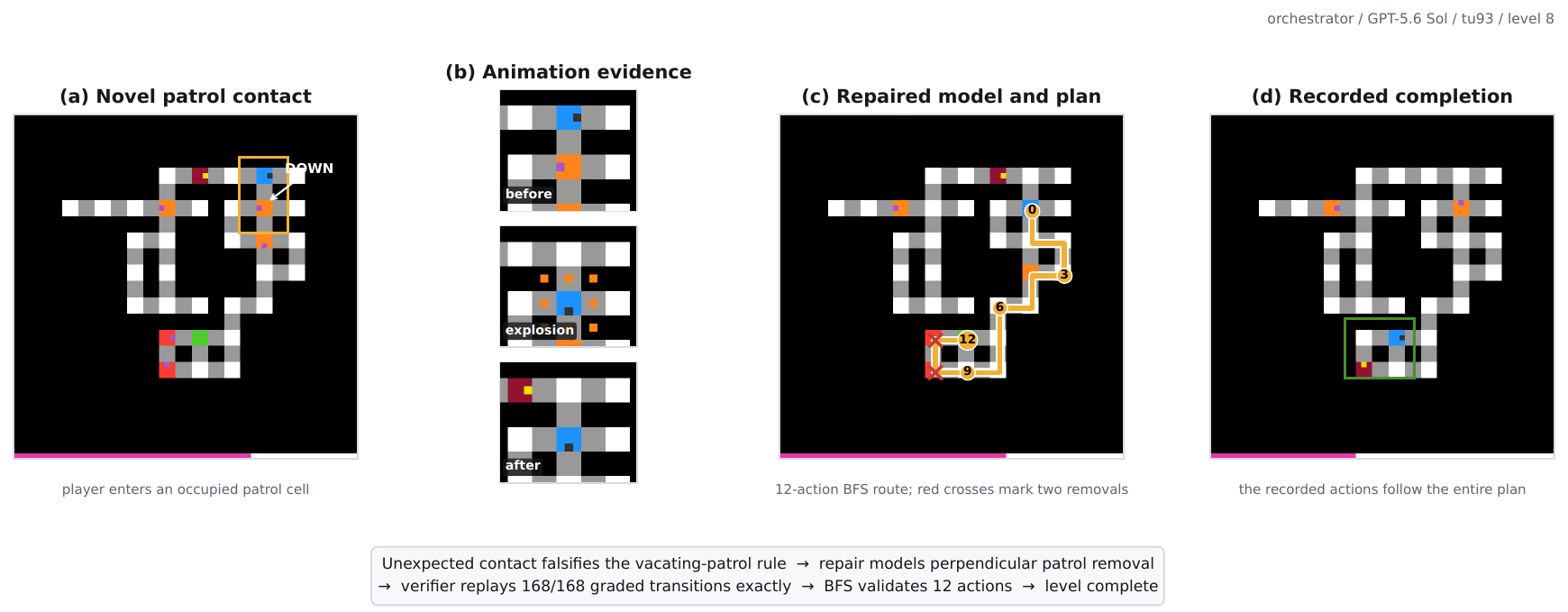}
  \caption{Transition mismatch, repair, and replanning in \code{tu93}, level~8.  The first observed player--patrol contact
  produces an observed explosion and removal.  The repaired model replays 168 graded transitions exactly,
  then validates the 12-action route that the recorded run follows to completion.}
  \label{fig:case-tu93-repair}
\end{figure}

\FloatBarrier

\paragraph{Correct mechanics, unresolved objective.}
The matched Opus~4.8 Orchestrator run on \code{sk48} demonstrates the opposite regime.  The builder
inferred a task-specific gripper-and-block simulator and repeatedly reported exact reproduction of observed
transitions, but no terminal had revealed the objective.  The actor and builder consequently treated the
pictorial HUD as a sequence of candidate goal predicates: a horizontal color train, a vertical dock, and
several variants that included gripper position or an attached arm.  Each candidate goal was reachable in the simulator and yielded a long plan, but executing the plans did not produce the predicted outcomes.  Thirteen builder calls and
314 official scored actions completed no level, whereas the trigger policy completed all eight levels of the
same game in this run.  The case shows how an accurate transition model can amplify a wrong or ungrounded
objective.

\begin{figure}[!htbp]
  \centering
  \includegraphics[width=\linewidth]{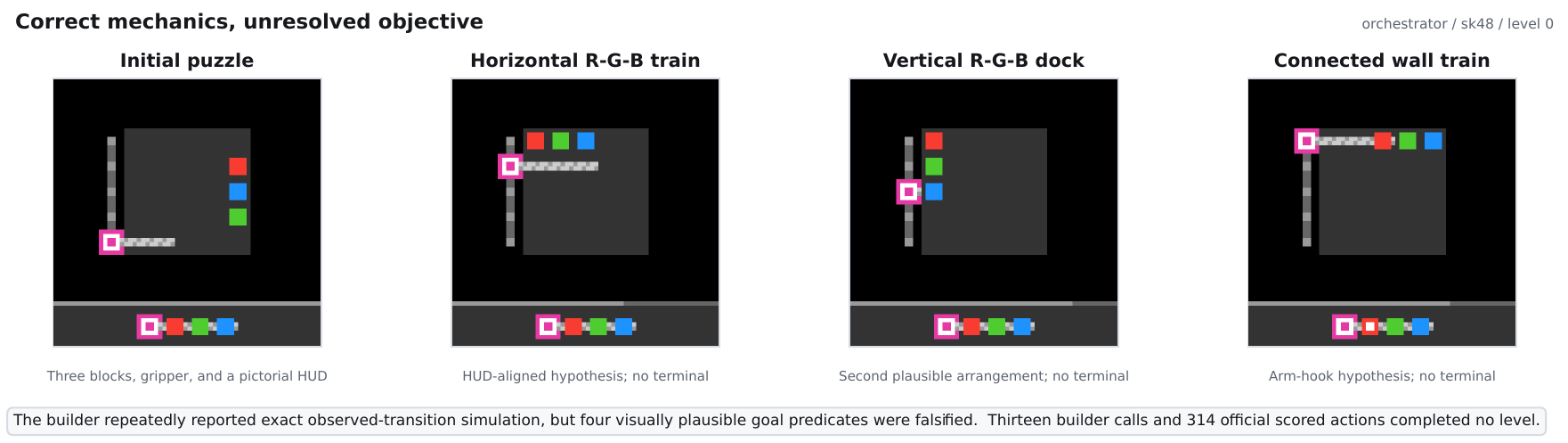}
  \caption{Outcome-identification failure in \code{sk48}, level~0.  The initial puzzle is followed by
  three simulator-reachable configurations suggested by the HUD; all remained nonterminal in the
  environment despite accurate modeled dynamics.}
  \label{fig:case-sk48}
\end{figure}

Together, the cases show both the flexibility and the risk of the model interface.  One model preserves a
nearly pixel-level UI state.  Another discards display resolution in favor of a hazard-aware graph.  The
transfer episode repairs a previously unobserved interaction, while the failure case models mechanics
without identifying the target.  What is shared is not a fixed object decomposition, but an executable
contract whose transition, outcome, and search components must be evaluated separately.

\FloatBarrier

\subsection{Behavioral Evidence of Task Simulation}
\label{sec:gpt-adaptation-cases}

A high score demonstrates acquired skill, but not by itself the process that produced it.  Chollet's
definition emphasizes skill-acquisition efficiency: how effectively a system turns
limited experience and prior knowledge into new competence \citep{chollet2019measure}.  ARC-AGI-3
operationalizes part of this idea through exploration, goal inference, dynamics learning, planning, and
action efficiency \citep{arcprize2026arcagi3}.  To look beyond aggregate scores, the ARC Prize Foundation
audited 160 frontier-model replays and traces against human-written game strategies.  The audit identified
three recurring failures: recognizing a local effect without deriving a global rule, importing the wrong
abstraction from a familiar game, and clearing a level without learning a transferable mechanic
\citep{arcprize2026failureanalysis}.  We use these distinctions to inspect the reported GPT-5.6 Sol
Orchestrator run based on the model-authored
persistent notes, executable programs and plans, recorded actions, and subsequent observations.
We select \code{tr87} to examine how mechanisms established in earlier levels are integrated into one
executable rule system, and \code{ka59} to show how useful actions can test uncertain mechanics and how
deliberation can precede execution.

\paragraph{Composing an executable rule system.}
\code{tr87} develops a visual rewrite language over six levels.  Its upper panels encode rules, while its
lower panels show a query and target.  Actions move a cursor between panels or cycle every glyph in the
selected panel through a seven-state visual alphabet.  Levels~1--3 establish the alphabets and 
variable-length rules: one query symbol can be represented by multiple target symbols and vice versa.  Level~4 first requires a fixed two-stage translation from cyan through pink to yellow, with only the answer reels editable.  Level~5 changes the interaction: the rule panels themselves become editable, and every glyph in a multi-glyph panel cycles in lockstep.

The builder handled these changes by repeatedly revising one persistent per-game program, not by
maintaining separate models for different levels.  The Level~4 revision added color-directed multi-hop
translation, while Level~5 added editable-panel state and lockstep transitions.  The final initializer
infers the applicable layout from the current grid rather than branching on the level index.

Level~6 combines these previously separate demands.  Each of its three rows contains one cyan-to-pink
rule and one pink-to-yellow rule, giving six rules represented by 12 source and output panels
(\Cref{fig:case-tr87-adaptation}a).  The agent must edit the rules so that translating the cyan query
produces a pink string which, in turn, translates to the fixed yellow target.  Its executable model parses
variable-length rules, predicts the intermediate pink string in \Cref{fig:case-tr87-adaptation}b, and
searches over cyclic shifts of the 12 panels.  The selected assignment requires 13 panel edits and 11
cursor moves.  The recorded edits follow this assignment and the level terminates on
action~24 (\Cref{fig:case-tr87-adaptation}c), compared with a 146-action human baseline.  This is
within-game compositional transfer: no new control action is introduced, but the agent must combine the
two-stage inference established in Level~4 with the editable-panel dynamics established in Level~5.

\begin{figure}[!htbp]
  \centering
  \includegraphics[width=\linewidth]{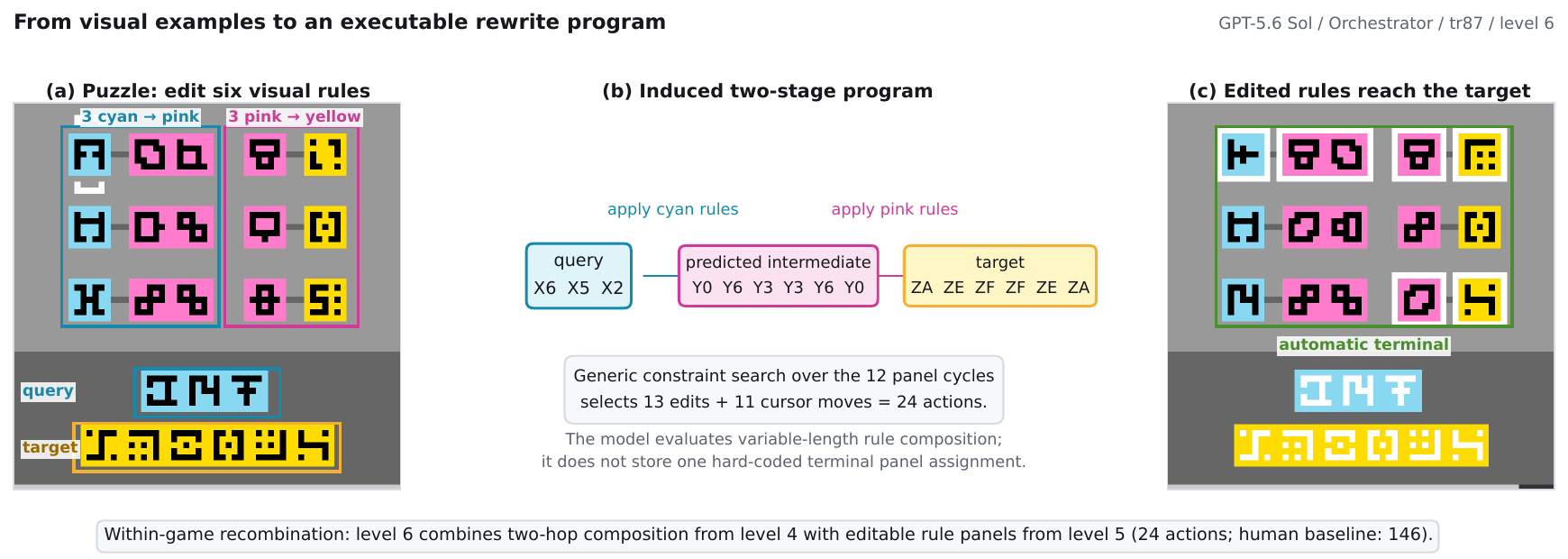}
  \caption{Within-game recombination in \code{tr87}, level~6.  Earlier levels separately introduce
  two-stage composition and editable rule panels; level~6 combines them.  The agent predicts the
  intermediate string and searches for a minimum-cost panel assignment, completing the level in 24
  actions.}
  \label{fig:case-tr87-adaptation}
\end{figure}

\paragraph{Turning a useful action into a test.}
\code{ka59} asks the agent to place colored pieces into matching sockets.  Later levels add mechanisms (pistons) in dark-red color
that periodically extend as actions advance, while purple bands block ordinary movement. 
The extension can be used for accelerated movement to bypass obstacles.  
Level~6 contains two ceiling-mounted pistons and one floor-mounted piston that can be moved between shafts
(\Cref{fig:case-ka59-adaptation}a).  Under the inherited assumption that every piston extends downward,
exhaustive search over all 6,888 reachable modeled states finds no goal.

The placement of the pistons suggests a different rule: a piston extends away from its supporting surface.  The first
action of the resulting shortest plan is also a discriminating test.  It advances the phase and reveals
orange on the top edge of both ceiling pistons but the bottom edge of the floor piston, as predicted
(\Cref{fig:case-ka59-adaptation}b).  The notes correctly preserve what this frame does \emph{not} establish:
orientation is observed, but upward transport remains a hypothesis.  Action~12 supplies the decisive
test by carrying the green token through the purple band to its predicted cell
(\Cref{fig:case-ka59-adaptation}c).  The remaining plan executes without correction and fills both sockets
in 45 actions, versus the 132-action baseline (\Cref{fig:case-ka59-adaptation}d).

\begin{figure}[!htbp]
  \centering
  \includegraphics[width=\linewidth]{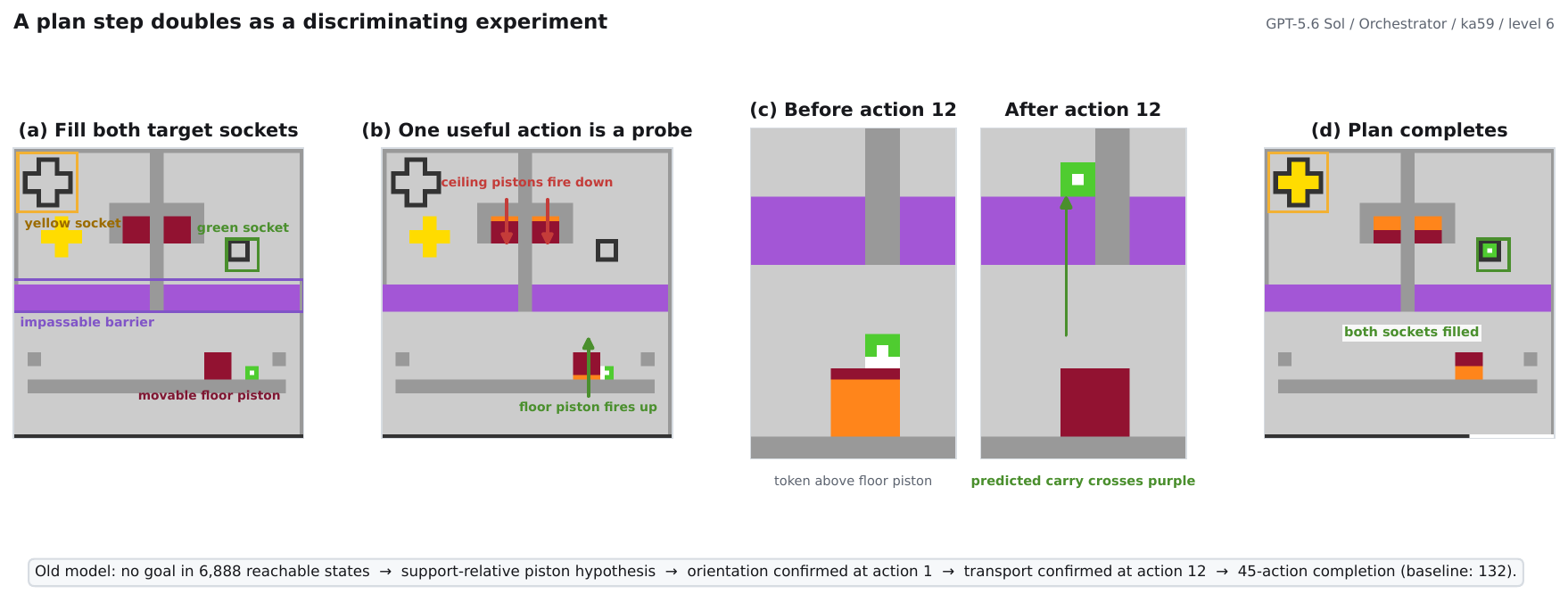}
  \caption{Hypothesis testing through useful actions in \code{ka59}, level~6.  One plan step distinguishes
  support-relative piston orientation from the inherited model; a later step tests the still-open
  transport hypothesis.  Both predictions are confirmed before the plan completes.}
  \label{fig:case-ka59-adaptation}
\end{figure}

\paragraph{Deliberation before execution.}
The final level of \code{ka59} shows a different separation between inference and action.
A prolonged actor--builder phase committed only three environment actions.  During it, the builder
corrected the transition rule for selection clicks, recovered relevant pushing and beam-transport
evidence from earlier levels, and validated an end-to-end route.  The actor then completed the level
through the remaining 88 actions in rapid succession, without another builder call.  This was not
batched control: actions were still issued through ordinary turns, but now followed the established
route with state checks and minor adjustments.  The trace therefore exposes substantial pre-action
computation that an action-based efficiency score does not measure.

Together, these traces show reusable abstraction, informative intervention, hypothesis revision,
cross-level transfer, and efficient execution once a useful model is found.  The evidence concerns the
coupled system---frontier model, Tycho harness, executable workspace, and search---on public games; it does
not separate the model's prior knowledge from online acquisition.  Its value is diagnostic: the retained
intermediate artifacts connect successful behavior to the tests that preceded it.

\FloatBarrier

\renewcommand{\topfraction}{0.70}
\renewcommand{\bottomfraction}{0.30}
\renewcommand{\textfraction}{0.20}
\renewcommand{\floatpagefraction}{0.50}

\section{From Task Adaptation in ARC-AGI-3 to Agent Adaptation: Learning from Limited Feedback}
\label{sec:agent-adaptation}

General-purpose agents must adapt to new tasks from limited task-specific evidence rather than depend on
data-heavy training for each new problem, a requirement central to skill-acquisition efficiency
\citep{chollet2019measure}.  ARC-AGI-3 makes this requirement concrete: an agent must identify hidden
mechanics and objectives through bounded interaction, while exploratory actions consume the same budget
used to complete the task.  Tycho studies one response to this constraint: constructing explicit, testable
task models during inference and using them to guide information gathering, planning, and action.

The same principle applies one level higher.  A run may expose limitations not only in its task model but
in the harness governing what the model can observe, retain, test, and execute.  Tycho records traces,
metrics, and meta-reflections to derive candidate changes for future runs.  This is adaptation from limited
evidence directed at the agent architecture rather than the task.

\begin{figure}[t]
\centering
\resizebox{0.98\linewidth}{!}{%
\begin{tikzpicture}[
  font=\small,
  box/.style={draw=NavyBlue!75, fill=NavyBlue!5, rounded corners=2pt, very thick,
    align=center, inner sep=5pt, text width=3.25cm, minimum height=1.22cm},
  modelbox/.style={draw=BrickRed!75, fill=BrickRed!5, rounded corners=2pt, very thick,
    align=center, inner sep=5pt, text width=3.25cm, minimum height=1.22cm},
  llmbox/.style={draw=black!65, fill=black!4, rounded corners=2pt, very thick,
    align=center, inner sep=5pt, text width=3.25cm, minimum height=1.22cm},
  harnessbox/.style={draw=NavyBlue!70!black, fill=NavyBlue!6, rounded corners=2pt, thick,
    align=center, inner sep=5pt, text width=3.25cm, minimum height=1.22cm},
  outerbox/.style={draw=ForestGreen!70!black, fill=ForestGreen!6, rounded corners=2pt, thick,
    align=center, inner sep=5pt, text width=3.25cm, minimum width=3.65cm, minimum height=1.65cm},
  gatebox/.style={draw=black!55, fill=black!4, rounded corners=2pt, thick,
    align=center, inner sep=5pt, text width=3.25cm, minimum width=3.65cm, minimum height=1.65cm},
  reflectbox/.style={draw=RoyalPurple!75, fill=RoyalPurple!6, rounded corners=2pt, thick,
    align=center, inner sep=5pt, text width=3.25cm, minimum height=1.22cm},
  arrow/.style={-{Latex[length=2.3mm]}, thick, draw=black!72},
  botharrow/.style={{Latex[length=2.3mm]}-{Latex[length=2.3mm]}, thick, draw=black!72},
  outerarrow/.style={-{Latex[length=2.3mm]}, thick, draw=ForestGreen!55!black},
  dashedarrow/.style={-{Latex[length=2.3mm]}, thick, dashed, draw=black!55}
]
  \node[llmbox] (llm) at (0,0) {\textbf{Foundation model}\\reasoning, tools;\\no task fine-tuning};
  \node[harnessbox] (harness) at (4.45,0) {\textbf{Frozen harness}\\tools, memory;\\verifiers};

  \node[box] (observe) at (0,-2.75) {\textbf{Observe}\\task state;\\feedback};
  \node[modelbox] (model) at (4.45,-2.75) {\textbf{World model}\\state, dynamics;\\objectives};
  \node[box] (plan) at (8.9,-2.75) {\textbf{Plan/probe}\\model-based\\search};
  \node[box] (act) at (13.35,-2.75) {\textbf{Act}\\environment\\step};
  \node[reflectbox] (reflect) at (4.45,-5.35) {\textbf{Meta-reflection}\\harness-aware;\\development note};

  \draw[botharrow] (llm) -- (harness);
  \draw[arrow] (harness.south) -- ++(0,-0.48) -| (observe.north);
  \draw[arrow] (observe) -- (model);
  \draw[arrow] (model) -- (plan);
  \draw[arrow] (plan) -- (act);
  \draw[arrow] (act.south) -- ++(0,-0.55) -| ([yshift=-0.55cm]observe.south) -- (observe.south);
  \draw[dashedarrow, draw=RoyalPurple!75!black, line width=0.9pt]
    (llm.south east) -- ++(0.22,0) |- ([xshift=-0.25cm]reflect.west) -- (reflect.west);

  \begin{scope}[on background layer]
    \node[draw=NavyBlue!45, fill=NavyBlue!2, rounded corners=4pt, inner sep=11pt,
      fit=(llm)(harness)(observe)(model)(plan)(act)(reflect), label={[font=\bfseries\small, text=NavyBlue!80!black]above:Frozen task-time agent run}] (innerfit) {};
  \end{scope}

  \node[outerbox] (artifacts) at (0.55,-8.05) {\textbf{Experience record}\\traces, metrics;\\meta-reflections};
  \node[outerbox] (redesign) at (4.8,-8.05) {\textbf{Generalize}\\failure mechanism;\\minimal change};
  \node[outerbox] (candidate) at (9.05,-8.05) {\textbf{Candidate agent}\\revised harness;\\future runs};
  \node[gatebox] (holdout) at (13.3,-8.05) {\textbf{Validation}\\transfers? clean?\\worthwhile?};

  \coordinate (runout) at (artifacts |- innerfit.south);
  \draw[dashedarrow] (runout) -- (artifacts.north);
  \draw[dashedarrow, draw=RoyalPurple!75!black, line width=0.9pt]
    (reflect.south) -- (artifacts.north east);
  \draw[outerarrow] (artifacts) -- (redesign);
  \draw[outerarrow] (redesign) -- (candidate);
  \draw[outerarrow] (candidate) -- (holdout);
  \draw[outerarrow] (holdout.east) -- ++(0.6,0) |- ([xshift=0.6cm]harness.east) -- (harness.east);
\end{tikzpicture}%
}
\caption{Two coupled adaptation loops.  During a frozen task run, the model and harness turn observations
into executable hypotheses, probes or plans, and actions.  The foundation model generates role-specific
meta-reflections from its harness-mediated context, and the harness stores these development notes with
traces and metrics in the experience record.  They do not affect the current trajectory.  Between runs,
those records motivate minimal harness changes whose transfer and integrity are checked before they enter
future agents.}
\label{fig:agent-adaptation}
\end{figure}
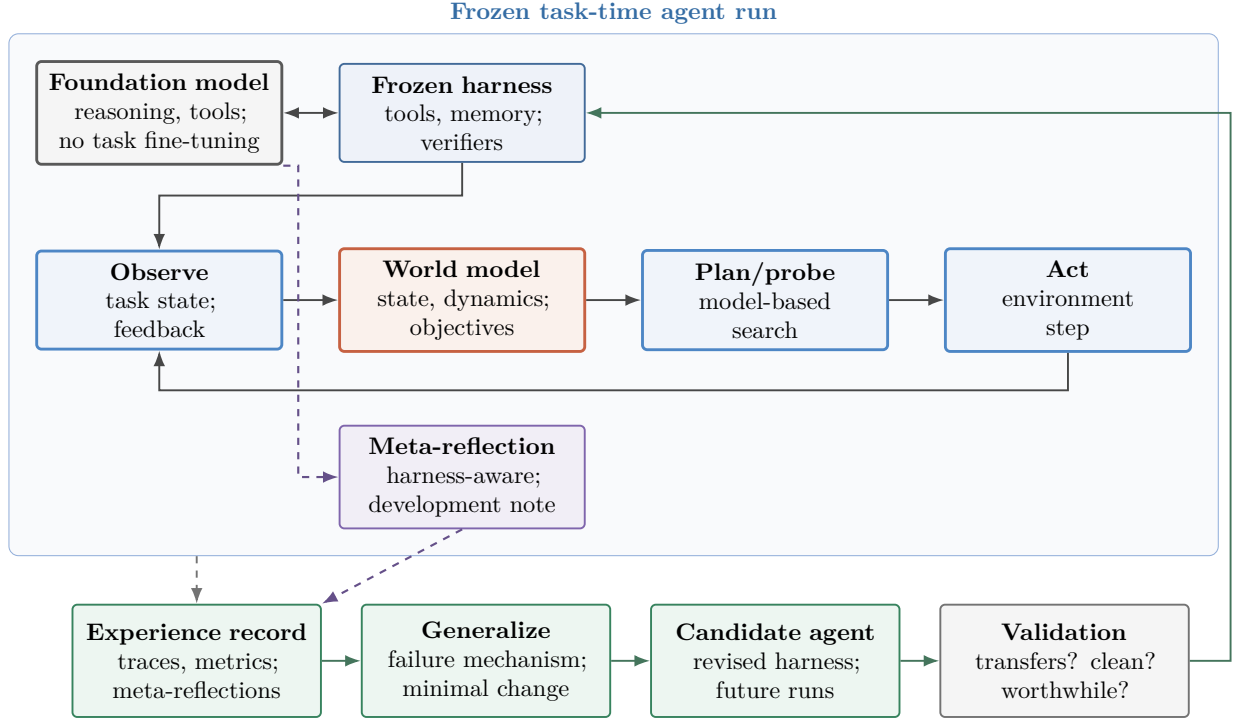

\Cref{fig:agent-adaptation} presents the same adaptation pattern at two scales.  Each loop turns limited
feedback into a testable hypothesis: the inner loop models the task, whereas the outer loop models a
limitation of the agent architecture.  Each then tests the hypothesis through an action or system change
and uses new evidence to retain, revise, or reject it.  Freezing the harness within a run separates the two
timescales: task adaptation can affect the current trajectory, while agent adaptation affects only future
runs.

\paragraph{The agent extends beyond the foundation model.}
An interactive agent couples a foundation model to an observation interface, memory, tools, action
protocols, verification, and recovery.  Changing the scaffolding or agent--computer interface can change
behavior with the model held fixed \citep{sghaier2026scaffolding,yang2024sweagent}.  System evaluations
must therefore report the architecture and protocol alongside the score
\citep{kapoor2024agentsmatter,zhang2026harnessdisclosure}.  Tycho is instantiated on ARC-AGI-3, but
follows a broader methodology: preserve experience, turn it into explicit hypotheses that can be tested,
and use those hypotheses selectively for action.  We therefore evaluate the compound agent rather than
the foundation model in isolation.  With the ARC Prize standard harness, Opus~4.8 scores $4.79$ on the 25 public games ($1.52\%$ on the
semi-private set),\footnote{\url{https://arcprize.org/scorecards/model/anthropic-opus-4-8-high}} whereas
Tycho's orchestrator reaches $88.49$ on the public games with the same foundation model.  This is not a controlled estimate of any
single component, but it shows how strongly system architecture can alter behavior.  The same methodology
can then be applied one level higher: failures in the harness become hypotheses for agent adaptation.
\FloatBarrier

\paragraph{Meta-reflection turns trajectories into candidate system changes.}
The actor and builder system prompts include a development-aid instruction inviting each role to record
one or two concrete sources of friction, such as a confusing instruction, unhelpful feedback, a missing
affordance, or a workspace obstacle.  Roles often provide these optional notes.  We call this
meta-reflection because its object is not the game but the system through which the game is approached.
The notes are written inside the run but read only afterward, so they cannot alter its trajectory or score.
They serve as compact hypotheses about system limitations that can be checked against the associated
trace.

Across development runs, meta-reflections and trace review exposed recurring limitations in Tycho.
Representative examples included low-resolution renders, missing level-boundary frames, file-operation
caps, unavailable Python libraries, incorrect reset semantics, missing verification and planning metrics,
imprecise diagnostic reports, and insufficient isolation from network or external-filesystem access.  This
list is illustrative rather than exhaustive.  None of these observations described an individual game's
mechanics.  Instead, they identified transferable deficiencies in how future agents could observe, retain,
verify, report, and act.

Related work uses trajectory feedback to revise language-model programs \citep{agrawal2025gepa}, or
searches and repairs agent architectures \citep{hu2024adas,chen2026failedtrajectories}.  In Tycho, we still
mediate this outer loop.  We cluster notes across runs, inspect the corresponding traces, and filter the
feedback for generality.  Requests for additional task hints, for example, are not treated as candidate
improvements.  For recurring issues, we infer a general failure mechanism and implement a minimal
candidate change.  We reject changes that leak task information, lack a transferable mechanism, or add
unjustified complexity, and test the remaining candidates on development games before freezing the
architecture.  This process turns local friction into an auditable architectural hypothesis and
distinguishes agent adaptation from benchmark-specific patching.

\paragraph{World models make limited experience reusable.}
The inner loop requires a corresponding abstraction over task experience.  A useful world model turns a
trajectory into a prospective computational object: it preserves task-relevant state, predicts relevant
consequences, exposes uncertainty, and supports testing candidate actions before executing them in the
environment.  ARC-AGI-3 makes Python programs and deterministic Moore machines a natural choice because
its observations and actions are discrete and exact.  The broader claim is not that useful world models
must be symbolic programs or exact simulators.  It is that an agent can make limited experience reusable
by organizing it into a predictive representation for information gathering, planning, and action.

The functional requirement is representation-agnostic.  Programmatic models can extend beyond
deterministic grid worlds \citep{piriyakulkij2025poeworld}, while object-centric states
\citep{locatello2020objectcentric}, causal models \citep{scholkopf2021toward}, and learned latent dynamics
provide alternative formalisms.  MuZero further shows that a planning model need only predict
action-relevant quantities rather than complete observations \citep{schrittwieser2020mastering}.  Any such
model must organize history for prospective action selection: what is known, what remains uncertain, which
action would resolve it, and which plan should work if the model is right.  At greater generality, choosing
the representation becomes part of agent adaptation itself.  An unfamiliar domain may call for a
symbolic simulator, belief state, relational graph, learned neural dynamics model, hybrid representation,
or no explicit model at all.
\looseness=-1

\paragraph{Toward closing the outer loop.}
Limited feedback makes audit and validation more important: a single trajectory can support a
plausible but spurious diagnosis.  Replayable traces, verifier metrics, and explicit artifacts make proposed
improvements inspectable.  Separate validation can then test whether changes transfer, remain
uncontaminated, and justify their added complexity.  In Tycho, task-time adaptation occurs through
inference-time hypotheses, tools, and memory rather than task-specific fine-tuning.  The outer loop remains
human-mediated: the actor and builder generate diagnostic notes, while we select, refine, implement, and
validate the resulting changes.

The next step is to automate these stages without relaxing their controls: compare evidence across runs,
formulate recurring failure mechanisms, implement candidate changes in isolation, and test them on held-out
tasks.  Our broader hypothesis is architectural: generality depends not only on the pretrained model, but
also on machinery that can organize limited experience into task models and revise the system supporting
that process.  Under a strict standard for general intelligence, an agent should eventually construct or
adapt whatever harness and model formalism a novel problem requires rather than depend on a human engineer
to do so.  More capable models should require fewer externally engineered iterations and progressively
close this outer loop themselves.

\section{Updated Study: Policy Comparison after Saturation}
\label{sec:vtwo}

The first-version study asked which model-maintenance policy to select under a fixed model.  Since then,
newer models have saturated public-set RHAE, and the practical question has moved to what a saturated
score costs in environment actions and in inference.  This section repeats the four-policy comparison on
an updated harness with four current models and three runs per condition.  It is a separate study: the
harness, the models, and the limits differ from \Cref{sec:evaluation}, whose results are kept as published
with the corrections of \Cref{app:corrections}.

\subsection{Harness and design}

\paragraph{Changes to the harness.}
The evidence interface, the executable-hypothesis contract, the verification semantics, and the four
policy definitions are those of \Cref{sec:tycho}.  Three changes matter for reading the results.
\begin{itemize}
  \item \emph{Multi-action commits.}  The first-version harness committed one environment action per
  decision.  The actor now commits an ordered sequence of actions in one decision.  Execution stops at the
  end of the sequence, at a level completion or game over, or when an optional guard fails.  The actor
  chooses the guard per commit: expected cell values that it names for individual steps, or the full
  grid predicted by the executable model, in which case the sequence stops at the first observation the
  model did not predict.  This implements the guarded execution listed as future work in the first
  version.
  \item \emph{Terminal frames and trigger.}  The recorded terminal of a level and the recorded frame of a
  fatal action are selected by the engine's frame-stack rule, which corrects the recording error
  described in \Cref{app:corrections}.  The level-start invocation of the builder fires only after a real
  level completion, and an indeterminate verification probe no longer counts as a model error.
  \item \emph{Limits.}  A game stops at \$400 and a level at \$50 of inference at list prices.  Both limits
  were fixed before the study as protection against runaway inference.  A decision that receives eight
  consecutive replies without a tool call is forced to act.  The per-game call limit (3{,}500) and the
  per-level action budget of five times the human baseline are unchanged.
\end{itemize}
The other mechanisms of the first version are retained: history that preserves the providers' prompt
caches, the summary written at each level boundary, model-written summaries of selected animations, and
the automatic planner search after each model edit.  Cache reuse was improved, and the planner search
now runs with a larger budget (50{,}000 nodes and 30 seconds instead of 5{,}000 and 3).

\paragraph{Conditions.}
We run the four policies with GPT-6.2 Sol, GPT-6.2 Astra, Claude Fable~5.2, and Claude Opus~5.5 on all 25
public games, three runs each: 48 runs of 25 games, or 1{,}200 games.  The three runs of a condition are
independent stochastic trajectories; each plays every game once, with no retries and no selection among
trajectories.  Reasoning effort was fixed per model in a preliminary ten-game sweep with the single
policy, as the lowest setting that completed all ten games at $100.00$: \code{max} for GPT-6.2 Sol,
\code{high} for GPT-6.2 Astra, \code{medium} for Claude Fable~5.2, and \code{high} for Claude Opus~5.5.
The sweep data are not part of the study.

\paragraph{Metrics and accounting.}
RHAE is at or near its ceiling in every condition, so next to RHAE and completed games we report scored
environment actions, inference cost, and actions relative to the \emph{best known route}.  This is the
route of a public competition-mode scorecard of 19 November 2026 \citep{starga2026scorecard}, with
\RESBestKnownPath{} actions over the 25 games against $17{,}135$ for the official human baselines.  We
replayed its recorded actions in the offline engine against the official game sources: they complete all
25 games and 183 levels in \RESBestKnownPath{} actions, and every state and frame matches the recording.
The route is therefore feasible and an upper bound on the optimum.  How it was produced is not
documented, and we use it only as a route.
Reported cost is retained-trajectory cost (\Cref{app:corrections}) at the model providers' public list
prices as of November 2026 \citep{openaipricing2026nov,anthropicpricing2026nov}, computed from recorded
token and cache traffic.\footnote{The accounting uses, per million fresh-input, output, cache-read, and
cache-write tokens, \$2/\$10/\$0.10/\$2.50 for GPT-6.2 Sol, \$8/\$40/\$0.40/\$10 for GPT-6.2 Astra,
\$8/\$40/\$0.20/\$10 for Claude Fable~5.2, and \$4/\$20/\$0.20/\$5 for Claude Opus~5.5.}
It excludes the end-of-game meta-reflection (\Cref{sec:agent-adaptation}), a separate call after the last
action that amounts to $3\%$ of the total, and calls repeated after a service interruption, which total
\$27 over all 48 runs.  %

\paragraph{Verification of trajectories.}
Every reported trace is replayed offline with per-action checks of state, completed levels, and frame
hash, and then through the ARC-AGI API in competition mode.  Each resulting scorecard reproduces the
locally recorded score, completed levels, and action count; the scorecards are listed with the released
code.

\subsection{Results}

\begin{table}[t]
\centering
\small
\setlength{\tabcolsep}{3.6pt}
\begin{adjustbox}{max width=\linewidth}
\begin{tabular}{@{}llrrrrrr@{}}
\toprule
Model & Policy & RHAE & Games & \shortstack{Scored\\actions} & \shortstack{Actions /\\best route} &
\shortstack{Cost\\(USD)} & \shortstack{Builder\\calls} \\
\midrule
GPT-6.2 Sol & No world model & \RESSolNoneRhae & \RESSolNoneSolved & \RESSolNoneActions & \RESSolNoneVsBest & \RESSolNoneCost & -- \\
(\code{max}) & Single & \RESSolSingleRhae & \RESSolSingleSolved & \RESSolSingleActions & \RESSolSingleVsBest & \RESSolSingleCost & -- \\
 & Orchestrator & \RESSolOrchRhae & \RESSolOrchSolved & \RESSolOrchActions & \RESSolOrchVsBest & \RESSolOrchCost & \RESSolOrchBuilder \\
 & Trigger & \RESSolTrigRhae & \RESSolTrigSolved & \RESSolTrigActions & \RESSolTrigVsBest & \RESSolTrigCost & \RESSolTrigBuilder \\
\addlinespace
GPT-6.2 Astra & No world model & \RESAstraNoneRhae & \RESAstraNoneSolved & \RESAstraNoneActions & \RESAstraNoneVsBest & \RESAstraNoneCost & -- \\
(\code{high}) & Single & \RESAstraSingleRhae & \RESAstraSingleSolved & \RESAstraSingleActions & \RESAstraSingleVsBest & \RESAstraSingleCost & -- \\
 & Orchestrator & \RESAstraOrchRhae & \RESAstraOrchSolved & \RESAstraOrchActions & \RESAstraOrchVsBest & \RESAstraOrchCost & \RESAstraOrchBuilder \\
 & Trigger & \RESAstraTrigRhae & \RESAstraTrigSolved & \RESAstraTrigActions & \RESAstraTrigVsBest & \RESAstraTrigCost & \RESAstraTrigBuilder \\
\addlinespace
Claude Fable~5.2 & No world model & \RESFableNoneRhae & \RESFableNoneSolved & \RESFableNoneActions & \RESFableNoneVsBest & \RESFableNoneCost & -- \\
(\code{medium}) & Single & \RESFableSingleRhae & \RESFableSingleSolved & \RESFableSingleActions & \RESFableSingleVsBest & \RESFableSingleCost & -- \\
 & Orchestrator & \RESFableOrchRhae & \RESFableOrchSolved & \RESFableOrchActions & \RESFableOrchVsBest & \RESFableOrchCost & \RESFableOrchBuilder \\
 & Trigger & \RESFableTrigRhae & \RESFableTrigSolved & \RESFableTrigActions & \RESFableTrigVsBest & \RESFableTrigCost & \RESFableTrigBuilder \\
\addlinespace
Claude Opus~5.5 & No world model & \RESOpusNoneRhae & \RESOpusNoneSolved & \RESOpusNoneActions & \RESOpusNoneVsBest & \RESOpusNoneCost & -- \\
(\code{high}) & Single & \RESOpusSingleRhae & \RESOpusSingleSolved & \RESOpusSingleActions & \RESOpusSingleVsBest & \RESOpusSingleCost & -- \\
 & Orchestrator & \RESOpusOrchRhae & \RESOpusOrchSolved & \RESOpusOrchActions & \RESOpusOrchVsBest & \RESOpusOrchCost & \RESOpusOrchBuilder \\
 & Trigger & \RESOpusTrigRhae & \RESOpusTrigSolved & \RESOpusTrigActions & \RESOpusTrigVsBest & \RESOpusTrigCost & \RESOpusTrigBuilder \\
\bottomrule
\end{tabular}
\end{adjustbox}
\caption{Updated study on the 25 public games: means over three runs per condition, with reasoning effort
in parentheses.  Games: completed games out of 25; Actions / best route: scored actions of completed games
divided by the best known route (\RESBestKnownPath{} actions over all games); Cost: retained-trajectory
cost at November 2026 list prices without the end-of-game reflection.}
\label{tab:vtwo-results}
\end{table}

\paragraph{Score no longer separates the policies.}
Mean public-set RHAE lies between \RESFableOrchRhae{} and $100.00$ in all sixteen conditions
(\Cref{tab:vtwo-results}), and individual runs lie between $96.25$ and $100.00$
(\Cref{tab:vtwo-per-run}).  %
Of the 1{,}200 games, 1{,}196 were completed; 44 of the 48 runs completed all 25 games, and 37 runs
score exactly $100.00$.  The four exceptions are \code{lf52} for GPT-6.2 Sol without a model in its first
run (per-level action budget, 5 of 10 levels), \code{lf52} for Claude Fable~5.2 without a model and under
single in its third run (per-level cost limit, 5 of 10 levels each), and \code{wa30} for Claude Fable~5.2
under orchestrator in its third run (per-level action budget, 6 of 9 levels).  No game ended at the
model-call limit or the per-game cost limit.  Where the first-version study separated the policies by up
to $9.4$ RHAE under Claude Opus~4.8, the current models complete the public set under every policy,
including direct reasoning without an executable model.  The comparison therefore rests on actions and
cost.  This is a statement about the public set with these models; it does not show that the policies
are equivalent on harder or unseen games.

\paragraph{Actions.}
On average, the conditions use between $37\%$ and $48\%$ of the official human-baseline actions.  %
They remain above the best known route: the mean ratios of the sixteen conditions range from
\RESAstraTrigVsBest{} to \RESFableNoneVsBest.  The condition with the fewest actions, GPT-6.2 Astra with
trigger, uses \RESAstraTrigActions{} actions on average
(\RESAstraTrigActionsMin--\RESAstraTrigActionsMax{} across runs).  For reference, the first-version
orchestrator runs used $6{,}641$ (Opus~5) and $7{,}766$ (GPT-5.6 Sol) actions.

Run totals of the same condition differ by up to several percent (for example $6{,}509$ to $6{,}952$ for
GPT-6.2 Astra under single), which is the size of most differences between policies.  We therefore
compare policies per game.  For each model and each pair of policies, every game completed under both
policies in a run gives one paired ratio of action counts; ratios of the same game are averaged over runs
on the log scale, so the unit of analysis is the game.  \Cref{tab:vtwo-paired} in
\Cref{app:vtwo-runs} reports the geometric mean ratio with a $95\%$ percentile bootstrap interval that
resamples games.  It contains 24 comparisons without correction for multiple testing, so we interpret
only the clear effects:
\begin{itemize}
  \item \emph{An executable model reduces actions for the two GPT-6.2 models.}  Against direct reasoning,
  the action ratio is $0.948$ (GPT-6.2 Sol) and $0.955$ (GPT-6.2 Astra) for orchestrator and $0.955$ and
  $0.942$ for trigger; these intervals exclude one.  The single policy also uses fewer actions ($0.961$
  and $0.975$), with intervals that end just below one.
  \item \emph{Automatic repair reduces actions for the two Claude models.}  With Claude Fable~5.2, trigger
  uses $12$--$14\%$ fewer actions than each of the other three policies (ratios $0.856$ to $0.884$), with
  fewer actions in 22 to 24 of the 25 games.  With Claude Opus~5.5, trigger uses $6\%$ fewer actions than
  direct reasoning ($0.939$); its $4\%$ advantage over single and orchestrator has intervals that reach
  one.
  \item \emph{The three model-using policies are not distinguishable for the GPT-6.2 models}, and
  \emph{single and orchestrator are not distinguishable from direct reasoning for the Claude models}.
  All of these intervals include one.
\end{itemize}

\ArxivPlanTableVtwo

\paragraph{Plan use.}
The plan-level statistics indicate where these differences come from.  \Cref{tab:plan-stats-v2} applies
the analysis of \Cref{tab:plan-stats-v1} to the updated study.  As there, the counts measure agreement
between a plan returned by the harness planner and the actions executed afterwards, not that the plan
caused the actions, and search code written by the agent itself is not counted.  Because actions are now
committed as sequences, the comparison is more direct than in the first version: in $73.8\%$ of all plan
decisions the sequence that the actor committed in the same decision contained the entire printed plan.

Under trigger, every model plays a large part of the game from plans: the share of executed actions that
agree with a returned plan is $79\%$ for GPT-6.2 Astra, $65\%$ for Claude Opus~5.5, $58\%$ for GPT-6.2 Sol,
and $42\%$ for Claude Fable~5.2, against $5.9\%$ for trigger in the first-version study, and the level
ends on a plan for $90\%$, $75\%$, $79\%$, and $61\%$ of completed levels.  %
A recorded contradiction between the model and an observation does not usually end a plan under trigger:
in 31 to 45 percent of the plans on which a level ended, a contradiction was recorded on the way, the
builder repaired the model, and the same action sequence continued.

Without the trigger, the two model families behave differently.  The two Claude models do not use the
executable-model channel on their own: over the twelve runs of the single and orchestrator policies they
made one builder call, received one plan from the harness planner, and made no model edit through the
edit tools.  For them, single and orchestrator are in effect repetitions of direct reasoning, which is
consistent with the paired comparison; they do write search code of their own, which these counts do not
cover.  The two GPT-6.2 models use the channel voluntarily in every run: GPT-6.2 Sol obtains 103 to 184
plans per run under single and plays $34\%$ of its actions from them, GPT-6.2 Astra 15 to 28 plans and
$4\%$; under orchestrator both invoke the builder throughout (\RESSolOrchBuilder{} and
\RESAstraOrchBuilder{} calls per run on average) and play $40\%$ and $46\%$ of their actions from plans.
The policies under which actions follow plans are thus the ones with fewer actions, for both model
families.  This is agreement between two measurements; reasoning effort differs by model, and we did not
test whether plan use causes the reduction.

\paragraph{Cost of the builder policies.}
Relative to the single policy, orchestrator costs $2.6$ times as much for both GPT-6.2 models, which use
the builder, and the same as single for the two Claude models, which do not.  %
Trigger costs $4.7$ (GPT-6.2 Sol), $5.4$ (GPT-6.2 Astra), $3.2$ (Claude Fable~5.2), and $2.8$ (Claude
Opus~5.5) times as much as single, with \RESSolTrigBuilder, \RESAstraTrigBuilder, \RESFableTrigBuilder,
and \RESOpusTrigBuilder{} builder calls per run.  %
The action reductions that the builder policies buy over direct reasoning are $4$--$6\%$ for the GPT-6.2
models and $6$--$14\%$ for the Claude models under trigger.  As in the first version, the comparison is
between complete designs: builder access, scheduling, and realized inference change together, and the two
builder policies are not matched in builder calls.  The complete study of 48 runs cost about \$8.2k of
inference, less than the single trigger run of the first-version study (\$9.02k).  %

\paragraph{A change in GPT-6.2 Sol between runs.}
The third GPT-6.2 Sol run (21 November) differs from the first two (18 and 20 November) although code
and configuration were identical.  Under the single policy, output tokens per actor call roughly doubled,
from about 250 to about 500, with unchanged call counts and cache-read share.  Reported cost rose
accordingly: \$31, \$30, and \$43 for single in the three runs, \$74, \$75, and \$125 for orchestrator, and
\$135, \$144, and \$214 for trigger.  The third run also used somewhat fewer actions under every policy,
and under single the actor obtained 184 plans against 103 and 124.  We report this as an observed change
in the behaviour of the served model.  With one affected run we cannot separate it from ordinary
run-to-run variation, and we therefore give GPT-6.2 Sol's cost per run with its date
(\Cref{tab:vtwo-per-run}) next to the mean.  GPT-6.2 Astra shows no such change in its output per call.

\paragraph{How far the bar has moved.}
ARC-AGI-3 was released on 25 March 2026 with a reported frontier-model score of $0.51\%$ against $100\%$
for humans \citep{arcprize2026launch}.  That figure is the verified result of Claude Opus~4.6 (\$8.9k) on
the semi-private evaluation set of 55 games with the ARC Prize standard harness, a minimal
provider-neutral interface \citep{arcprize2026leaderboard,arcprize2026policy}.  Under this condition the
highest verified score among models released up to June 2026 was $1.52\%$ (Claude Opus~4.8, \$10.0k).
Later releases raised it to $7.78\%$ (GPT-5.6 Sol, \$25.1k), $30.16\%$ (Claude Opus~5, \$20.7k), $62.7\%$
(GPT-6 Astra, \$26.1k), and $81.2\%$ (Claude Fable~5.2, \$21.9k), the highest standard-harness score
listed on 22 November 2026 \citep{arcprize2026leaderboard,arcprize2026astra}.  Since September 2026 ARC
Prize also reports a provider-adapter harness, which preserves the provider's reasoning state between
requests and compacts long conversations \citep{arcprize2026astra,arcprize2026policy}.  With it, GPT-6
Astra scored $99.95\%$ (\$18.8k), GPT-6.2 Astra $100.0\%$ (\$4.1k), and GPT-6.2 Sol $97.4\%$ (\$1.5k) on
the semi-private set \citep{arcprize2026leaderboard}.  These verified scores and costs refer to the
semi-private set.  For the 25 public games used in this paper ARC Prize publishes per-model scorecards
without cost; on these, the provider-adapter run of GPT-6.2 Astra at maximum reasoning effort completes
all 183 levels with a score of 100 in $6{,}147$ actions \citep{arcprize2026scorecardgpt62astramaxpa}.

Community harnesses moved in parallel on the public games.  The community leaderboard repository
\citep{arcprize2026communityrepo} collects systems with code and competition-mode scorecards; its costs
are self-reported.  Scorecards of 15 and 20 July 2026 show $98.97$ and $99.86$, the latter at a reported
\$654 \citep{baseline12026scorecard,retrodict2026scorecard}.  Tycho was the first harness to reach
$100.00$ public-set RHAE, with its scorecard of 29 July 2026 (Opus~5, \$2.99k, $6{,}641$ actions;
\citealp{tycho2026scorecardv1,arcprize2026communityrepo}).  Four later entries show $100.00$.  Three report costs of \$728, \$778, and \$415; the last, with GPT-6
Astra in $6{,}732$ actions, is the lowest reported cost for a single live run.  The fourth, with GPT-6.1
Sol in $6{,}781$ actions, reports an estimate of \$158 from a partial record for routes found over
repeated attempts
\citep{arcskill2026scorecard,kepler2026scorecard,oy12026scorecard,asterion2026scorecard}.

Our own runs trace the same movement.  With Claude Opus~4.8 the first-version orchestrator reached
$88.49$ public-set RHAE for an estimated \$6.51k, and the two first-version frontier runs reached
$100.00$ for \$4.88k (GPT-5.6 Sol, $7{,}766$ actions) and \$2.99k (Opus~5, $6{,}641$ actions).  The
GPT-6.2 Sol run of 18 November under the single policy reaches the same public-set score in $6{,}626$
actions for about \$31 at list prices, excluding the reflection call (\$33 with it; scorecard
\href{https://arcprize.org/scorecards/3363ff30-d956-4932-a6e5-ee01fe33b063}{\texttt{3363ff30}}).  That is
about one percent of the cost of the Opus~5 run at nearly the same action count.  Averaged over the three
runs the single policy costs \$\RESSolSingleCost{} with GPT-6.2 Sol.  The reduction combines a lower-priced
model, a cache-read share of about $90\%$ of input tokens, and fewer model calls per action through
multi-action commits.  Harness and model changed together, so the factor is not attributable to the
harness alone.  On action efficiency, the condition with the fewest actions (GPT-6.2 Astra with trigger,
\RESAstraTrigActionsMin{} to \RESAstraTrigActionsMax{} actions) uses
\RESAstraTrigVsBestMin--\RESAstraTrigVsBestMax{} times the best known route, at about ten times the cost
of the GPT-6.2 Sol single run.  The ARC Prize run of GPT-6.2 Astra cited above uses fewer actions, and we
do not claim the lowest action count on the public set.  Within eight months of the release, the public
games thus went from essentially unsolved to solved under ARC Prize's own provider-adapter harness and
under several community harnesses.  Completion no longer separates systems; what remains informative is
the cost of a run and the number of actions above the best known route.

\paragraph{Limits of this study.}
All conditions are on the public set, on which the harness was developed and the efforts were chosen, and
all four models were released after that set became public; we have no contamination audit for them.
The policies are complete designs and are not matched in builder budget.  Reasoning effort differs by
model, so differences between models are differences between model--effort pairs.  Three runs give a
range, not a distribution; the paired comparison rests on the 25 games.  Reported costs are list-price
equivalents of recorded traffic and depend on cache behaviour and on the behaviour of the served model on
the day of the run.  The best known route is an upper bound on the optimum, and the ratios would rise if
shorter routes were found.

\FloatBarrier

\section{Limitations}
\label{sec:limitations}

\paragraph{Empirical scope.}
The first-version experiments evaluate one stochastic run of each reported configuration on the 25
public games (183 levels), five of which informed harness development; the updated study
(\Cref{sec:vtwo}) has three runs per condition on the same games, with reasoning efforts chosen on ten of
them.  These results therefore characterize performance on a fixed public testbed, not generalization to
unseen games, and only the updated study gives an indication of run-to-run variation.  Benchmark exposure also
cannot be excluded: three games were public before the full 2026 release, and provider cutoff dates do not
constitute a benchmark-specific contamination audit.\footnote{The full public set
\href{https://arcprize.org/blog/arc-agi-3-launch}{launched on March 25, 2026}, but \code{ft09},
\code{ls20}, and \code{vc33} had been public since the
\href{https://arcprize.org/blog/arc-agi-3-preview-30-day-learnings}{July 2025 developer preview}.
The reported cutoffs for Opus~4.8 (January 2026) and
GPT-5.6 Sol (February 16, 2026) predate the full release but postdate the three preview games; Opus~5 has
a May 2026 cutoff, after the full release \citep{anthropicmodels2026,openaigpt56sol2026}.  The four
models of the updated study were released after the full public set, and we have no contamination audit
for them.}  A strict
generalization test requires a frozen-harness evaluation on unseen games.  The Opus~4.8 policy comparison
was not repeated, so the run-to-run variability of its ordering is unknown; the repeated runs of the
updated study concern newer models and a changed harness.

\paragraph{Benchmark-specific priors.}
Tycho is not benchmark-agnostic. 
However, its task contract, belief-memory schema, action interface, and modeling tools encode ARC-AGI-3--specific structural priors.
The observations are $64\times64$ grids over 16 colors, and the games comprise multiple levels that generally increase in complexity and share mechanics.
The framework supplies conventions for RESET and numbered actions.
Action efficiency is the benchmark objective, and the levels are assumed to be human-solvable through manageable interaction.
HUD or interface cells are interpreted using possible categories such as budget lives, timers, selected tools, progress indicators, validation lights, and status bars.  
These priors do not reveal any individual game's mechanics, objective, or solution, but they narrow the interpretation and modeling problem and would require revision for transfer to other domains.

\paragraph{Agent adaptation remains human-mediated.}
The initial creation of Tycho and the reflection-based changes described in
Section~\ref{sec:agent-adaptation} are author-steered.  The actor and builder produce diagnostic notes,
and LLMs assist development, but the authors select, refine, implement, and validate the resulting harness
changes.  These results therefore demonstrate a human-mediated adaptation process, not autonomous harness
construction or self-improvement.  A stronger test would require the system itself to select, implement,
and validate generic harness changes from limited feedback.

\paragraph{Human comparisons are reference distributions.}
The released human data contain independent first-run replays rather than stable participant identities
across games.  Our game-balanced statistic therefore compares a fixed agent trajectory with a random
replay on the same public game; it does not rank one agent against a population of 25-game human
participants.  Humans also did not receive the iterative public-set development process available to the
harness.

\paragraph{Action efficiency excludes inference cost.}
ARC-AGI-3 uses action count as a common proxy for resources including data, time, compute, and risk, while
internal reasoning and tool operations are not scored \citep{arcprize2026arcagi3}.  RHAE therefore records
neither the inference behind an action nor, once an agent is more efficient than the human reference, how
far it remains from the shortest route.  In the first version this hid a large trade: Opus~5 uses $61\%$
fewer environment actions than the aggregate human baselines, but the run entails 18.7k model calls and
an estimated \$2.99k in list-price inference cost across the public set.  For an order-of-magnitude
comparison, integrating the brain's estimated 20\,W metabolic power over the median successful human
replay for each public game yields roughly 0.1\,kWh \citep{arcprize2026human,yu2018brainenergy}.  Applying
published estimates of 0.24\,Wh for a production text request and 3.91\,Wh for a long test-time-scaled
reasoning request to Opus~5's 18.7k calls yields a 4--73\,kWh envelope
\citep{elsworth2025googleenergy,oviedo2026inferenceenergy}, roughly $45$--$730\times$ the human
brain-energy estimate.  This projection compares inference with brain metabolism rather than lifecycle
energy.

The updated study reports actions and cost side by side and shows that this trade has shifted without
disappearing (\Cref{sec:vtwo}).  Its sixteen conditions average between \RESFableOrchRhae{} and $100.00$
public-set RHAE, while their mean action counts differ by up to a third and their mean costs by more than
an order of magnitude.
A $100.00$ public-set run now costs about \$31 of inference at list prices (GPT-6.2 Sol under the single
policy, 18 November), about a hundredth of the Opus~5 run at nearly the same action count, yet it still
makes about 3{,}200 model calls, which the same per-request estimates place at roughly 0.8--13\,kWh.  %
Fewer actions also still cost more inference: the trigger policy has the lowest mean action count for
every model at about three to five times the cost of the single policy.  Finally, dollar cost is a less
stable measure than action count.  It depends on provider prices and on prompt-cache behaviour, and it
changed for one model between runs of the same configuration.  We therefore report it as the list-price
cost of recorded traffic at a stated date, not as a property of the system.

\paragraph{Inference latency and service reliability affect evaluation.}
Individual model calls can vary from seconds to tens of minutes during extended reasoning, while hosted
inference services can throttle, time out, or fail transiently.  Evaluation therefore depends on
operational choices such as per-call and per-game timeouts, retry and backoff policies, and durable
checkpointing for resumption.  Aggressive limits can truncate useful reasoning; permissive limits can
leave games stalled or exhaust the evaluation budget.  These choices can affect completion and
reproducibility even when the agent policy is unchanged.  Two of the four games not completed in the
updated study were ended by its per-level cost limit while the agent was still acting
(\Cref{sec:vtwo}).  The behaviour of a hosted model can also change between runs without any change
on our side, as observed for GPT-6.2 Sol (\Cref{sec:vtwo}), and reported costs exclude calls repeated
after a service interruption (\Cref{app:corrections}).

\paragraph{Model and evidence limitations.}
Programmatic models can be verbose, brittle, and costly when direct spatial reasoning would suffice.
Agreement with observed transitions also leaves off-trajectory behavior unidentified, while sparse
terminal events make objective inference harder than dynamics fitting.  Tycho therefore treats replay
consistency as evidence rather than proof and relies on discriminating probes, cross-level transfer,
planned-rollout tests, and separate measurements of action and language-model cost.

\section{Conclusion and Future Work}
\label{sec:conclusion}
\label{sec:future-work}

Tycho formalizes an ARC-AGI-3 game as a rendered deterministic Moore machine rather than a stream of
screenshots.  State and action determine transitions; states produce observations, available actions, and
outcomes; and transient and boundary evidence remain distinct.  The resulting evidence contract lets an
executable hypothesis be checked against experience, used for planning, and left partial where evidence
does not support a prediction.

The experiments yield three general conclusions.  First, executable modeling is useful selectively:
direct reasoning is already strong, while actor-requested delegation has the strongest observed
performance.  Second, transition accuracy and decision quality are distinct: automatic repair produces
the highest accepted transition match without the strongest gameplay because objective inference, model allocation, and
effective use of advice remain separate problems.  Third, a useful abstraction need not reconstruct the
privileged engine or every pixel; multiple task-specific representations and partial models can support
successful action when they preserve decision-relevant distinctions along a proposed route.  This is
\emph{active abstraction}: deciding whether, when, and how to model is part of the control problem, and a
model is valuable when it improves evidence gathering, planning, or action.  The first two conclusions
rest on one run per configuration with Claude Opus~4.8 on the public set and are observations about
complete policy designs.

The updated study changes the picture in two ways.  With four current models, public-set RHAE is saturated
under every policy, including direct reasoning, so the policies can only be told apart by actions, plan
use, and cost.  There, the executable model still earns its place, but differently for the two model
families: the GPT-6.2 models build and plan with it on their own and use $2$--$6\%$ fewer actions than
without it, while the Claude models use it only when the harness schedules its construction, in which case
they use $6$--$14\%$ fewer actions at roughly three times the inference cost.  Second, the cost of a
saturated public-set score has fallen by about two orders of magnitude since the first version, to about
\$31 in one GPT-6.2 Sol run, with an action count about $20\%$ above the best known route.  Whether
the policies separate again on harder or unseen games is open.

The decisive next test is frozen-harness generalization to environments unavailable during development,
including newly released ARC-AGI-3 games, independently authored rendered environments, and other
interactive domains.  Comparisons across frontier and open-weight models, and repeated runs on games that are not saturated,
can separate stable system properties from model-specific behavior.  Guarded reuse of verified plans and
lower inference cost, listed as priorities in the first version, are part of the updated harness.  The
remaining technical priorities include explicit uncertainty and better allocation between reasoning and
model construction.  More fundamentally, meta-reflection should become an agent capability:
proposing, implementing, and validating harness changes from limited feedback while preserving
independent checks.

\paragraph{A note on timing.}
The final weeks of this work coincided with an exceptional pace of change in ARC-AGI-3: new systems and
LLMs appeared, and the state of the art shifted repeatedly.  This changing empirical context led us to
revisit comparisons, analyses, and experimental priorities several times while finalizing the manuscript.
We chose to prioritize methodological quality over release speed, accepting that the paper would enter the public discussion at a later point.

\bibliographystyle{abbrvnat}
\bibliography{arxiv/refs}

\appendix

\section{Illustration of Tools and Prompt Excerpts}
\label{app:interface}

This section gives the complete tool surface and selected passages from the actor and builder system
prompts in the GPT-5.6 Sol orchestrator run.  The public release fixes the corresponding
\href{https://github.com/NIMI-research/Tycho/blob/f68912a764372ead0a610db2e1c011d41ce5197e/tycho/prompts/actor.system.j2}{actor template},
\href{https://github.com/NIMI-research/Tycho/blob/f68912a764372ead0a610db2e1c011d41ce5197e/tycho/prompts/builder.system.j2}{builder template}, and
\href{https://github.com/NIMI-research/Tycho/blob/f68912a764372ead0a610db2e1c011d41ce5197e/configs/paper/gpt56_sol_orchestrator_max.yaml}{run configuration}
at commit \code{f68912a}.  Both prompts also included the development-only meta-reflection instruction
described in Section~\ref{sec:agent-adaptation}; we omit that instruction here for brevity.  Both
agents received the same color vocabulary and access to the per-game workspace and recorded-history
library.  We show those shared operational instructions only where they clarify a role-specific
contract.  Bracketed ellipses mark omitted prompt text.

\paragraph*{Tool surface}

\Cref{tab:actor-tools} summarizes the complete actor tool surface.  A question
mark denotes an optional field.  File operations were confined to the per-game workspace.  Only
\code{take\_action} changed the environment and consumed a scored action. The other tools supported
computation, memory, context management, or delegation.

{
\footnotesize
\setlength{\LTpre}{0.35\baselineskip}
\setlength{\LTpost}{0.35\baselineskip}
\begin{longtable}{@{}
  >{\raggedright\arraybackslash}p{0.28\linewidth}
  >{\raggedright\arraybackslash}p{0.67\linewidth}@{}}
\caption{Actor tools in the GPT-5.6 Sol orchestrator run.}
\label{tab:actor-tools}\\
\toprule
Tool and arguments & Function and benchmark effect \\
\midrule
\endfirsthead
\multicolumn{2}{@{}l}{\footnotesize\itshape Table~\thetable\ continued}\\
\toprule
Tool and arguments & Function and benchmark effect \\
\midrule
\endhead
\bottomrule
\endfoot
\code{ls(path?)} &
List a workspace directory; read-only. \\
\code{read\_file(path)} &
Read a workspace file; read-only. \\
\code{write\_file(path, content)} &
Create or overwrite a workspace file; persists without an environment action. \\
\code{edit\_file(path, old, new)} &
Replace one exact, unique string in a workspace file; persists without an environment action. \\
\code{run\_python(code, timeout\_s?)} &
Run Python with the current grid, NumPy, and recorded-history library preloaded; the optional timeout is
1--300 seconds (15 by default); no environment action. \\
\code{set\_verbosity(grid?, diff?)} &
Control whether later turns inline the grid and an exact, summarized, or omitted diff; recorded evidence
is unchanged. \\
\code{invoke\_builder}\\[-1pt]\code{(reason?)} &
Ask the builder to construct or revise the executable hypothesis; invokes a separate model call but no
environment action. \\
\code{take\_action(action, row?, col?)} &
Commit a currently available game control; ends the turn and consumes one scored environment action. \\
\end{longtable}
}
\FloatBarrier

The \code{action} field of \code{take\_action} was constrained each turn to \code{RESET} and the controls
declared by the current frame.  The builder had a narrower surface: file listing, reading, writing, and
exact-string editing; \code{run\_python}; and \code{edit\_function(path, name, code)} for replacing one
top-level Python function.  It had neither \code{take\_action} nor recursive delegation, so the actor
retained control over every scored intervention.

\paragraph*{Annotated actor prompt excerpts}

\colorlet{actorpromptaccent}{NavyBlue!82!black}
\tikzset{actor prompt bracket/.style={
  draw=actorpromptaccent,
  line width=1.15pt,
  line cap=rect
}}
\mdfdefinestyle{actorprompt}{
  hidealllines=true,
  leftline=true,
  linecolor=actorpromptaccent,
  middlelinewidth=1.15pt,
  leftmargin=2.0em,
  rightmargin=0.2em,
  innerleftmargin=0.9em,
  innerrightmargin=0.2em,
  innertopmargin=0.35em,
  innerbottommargin=0.35em,
  skipabove=0.55\baselineskip,
  skipbelow=0.65\baselineskip,
  splittopskip=0.35\baselineskip,
  splitbottomskip=0.35\baselineskip
}
\newcommand{\actorpromptlabelnode}[1]{%
  \node[
    rotate=90,
    anchor=center,
    inner sep=1pt,
    align=center,
    font=\scriptsize\sffamily\bfseries,
    text=actorpromptaccent
  ] at ($(O)!0.5!(O|-P)+(-0.9em,0)$) {#1};%
}
\newcommand{\promptcut}{%
  \textcolor{actorpromptaccent!72!black}{\bfseries[\ldots]}%
}
\newcommand{\actorprompttophook}{%
  \draw[actor prompt bracket] (O|-P) -- ++(0.55em,0);%
}
\newcommand{\actorpromptbottomhook}{%
  \draw[actor prompt bracket] (O) -- ++(0.55em,0);%
}
\newenvironment{actorpromptblock}[1]{%
  \def\actorpromptblocklabel{#1}%
  \begin{mdframed}[
    style=actorprompt,
    singleextra={%
      \actorpromptlabelnode{\actorpromptblocklabel}%
      \actorprompttophook
      \actorpromptbottomhook
    },
    firstextra={%
      \actorpromptlabelnode{\actorpromptblocklabel}%
      \actorprompttophook
    },
    middleextra={},
    secondextra={\actorpromptbottomhook}
  ]%
  \small
  \setlength{\emergencystretch}{2em}%
}{%
  \end{mdframed}%
}

The selected actor passages establish the task, delegation policy, and durable belief channel.

\begin{actorpromptblock}{TASK\\CONTRACT}
\promptcut{}
You are playing an unfamiliar turn-based game on a 64x64 grid of colored cells (colors 0--15). Its
objects, mechanics, and goal are not given - you need to discover them by exploring and observing, like a
scientist. Each game has multiple levels of usually increasing complexity; a level ends when you reach
its (unknown) goal. Your objective: complete each level in as few actions as possible. Benchmark prior:
every level is solvable by human players through a manageable amount of interaction and reasoning.
\promptcut{}
The current turn header lists the valid \code{take\_action} choices: RESET plus the frame-declared game
actions. Tool calls are free for scoring; each \code{take\_action}, including RESET, spends one scored
environment action and ends the turn. Prefer RESET when the current attempt is unwinnable or clearly
more expensive than restarting. \promptcut
\end{actorpromptblock}

\begin{actorpromptblock}{BUILDER\\DELEGATION}
\promptcut{}
World modeling is delegated to a focused subagent; you do NOT write \code{world\_model.py} yourself.
Once the dynamics look learnable, call \code{invoke\_builder} with your beliefs about state, action
effects, objectives, and useful probes. \promptcut{}
The builder constructs and verifies \code{world\_model.py},
then reports its confidence, outcome hypothesis, and recommended action or subgoal. You retain action
control: discount uncertain advice, and do not treat failure to find a plan as proof that the level is
impossible. Re-invoke after informative new evidence or a failed prediction, then reuse the model.
\promptcut
\end{actorpromptblock}

\begin{actorpromptblock}{BELIEF\\MEMORY}
\promptcut{}
Keep your evolving hypotheses in \code{notes/actor\_beliefs.md}; this is both memory across turns and
the handoff read by the builder. \promptcut{}
Treat HUD and interface cells as possible observations of game state, not decoration, and infer their
role from transitions and terminal evidence. \promptcut
\end{actorpromptblock}

\paragraph*{Annotated builder prompt excerpts}

The selected builder passages contain the role-specific scientific instructions: latent-state
representation, bounded prediction, falsification feedback, cross-level generalization, outcome
inference, and validated planning.  Shared operational instructions and implementation-level cautions
are omitted.

\begin{actorpromptblock}{ROLE}
You are the WORLD-MODEL BUILDER for an agent playing an unfamiliar 64x64 grid game. You do NOT take game
actions. Your only job: construct or refine \code{world\_model.py} so it predicts the game's dynamics,
and report what you found to the ACTING agent.

The ACTING agent hands you, in the message below, its current beliefs about the mechanics (how the grid
parses into objects, how actions change the grid, the objective/outcome/subgoals). Treat that as your
starting hypothesis---verify and encode it; correct it where the observed frames disagree. \promptcut
\end{actorpromptblock}

\begin{actorpromptblock}{STATE AND\\OBSERVATION}
\promptcut{}
Define whatever latent fields the game requires in \code{State}; a single frame need not determine
the next. \promptcut{}
\code{render(state)} must project that state to the full visible grid. Use \code{-1} only
for genuinely undetermined cells. For bounded HUD quantization, \code{observation\_variants} may return
at most five full-grid alternatives; do not use it to hide uncertain dynamics, locations, off-screen
terrain, or outcomes. \promptcut
\end{actorpromptblock}

\begin{actorpromptblock}{FALSIFICATION\\FEEDBACK}
\promptcut{}
The kickoff message includes a \code{[verify state]} block with initial-render results, simulation
accuracy, the first divergence, and lossless observed and predicted deltas. Compare those deltas to
locate the error, edit the model, and use the next verification report as feedback. \promptcut
\end{actorpromptblock}

\begin{actorpromptblock}{VERIFY AND\\GENERALIZE}
\promptcut{}
Every save auto-verifies three channels: the initial render, transitions replayed from the level start,
and outcomes on solved, failed, and ongoing states. \promptcut{}
Iterate: edit $\rightarrow$ read feedback
$\rightarrow$ refine. There is ONE game engine across levels, so encode shared mechanics rather than
hardcoding level branches in \code{transition()}; use level-specific initialization or outcomes only
when the evidence requires them. \promptcut
\end{actorpromptblock}

\begin{actorpromptblock}{OUTCOME\\INFERENCE}
\promptcut{}
OUTCOME is a FIRST-CLASS inference, not an afterthought. The objective is never told to you; infer it
from what changed when a level ended. \promptcut{}
A draining move or life counter is a cost, not success. While
the objective is uncertain, retain competing terminal hypotheses and the probe that would separate them.
\promptcut
\end{actorpromptblock}

\begin{actorpromptblock}{PLANNING}
\promptcut{}
Use the generic planner first. \promptcut{}
A plan is validated only when replay through \code{world\_model.py} reaches level\_complete.
\promptcut
\end{actorpromptblock}

\section{Task-Relevant Identification and Plan Sufficiency}
\label{app:identification}

Finite interaction generally leaves many programs and underlying states compatible with the evidence.  The
relevant question is which remaining distinctions can affect benchmark performance.  We formalize two
answers: identification only up to benchmark behavior, and correctness only along a proposed route.

\subsection{What must be identified?}

After history $h_t$, let $\Bel_t$ contain the complete machine--state hypotheses consistent with it.  Each
hypothesis fixes future observations, available actions, costs, outcomes, and level progression.  A
partial executable model with abstained cells or observation alternatives denotes a set of these
completions; it does not make the environment stochastic.

\begin{definition}[Benchmark equivalence]
Two hypotheses $b,b'\in\Bel_t$ are \emph{benchmark-equivalent}, written
$b\equiv_{\mathrm{bench}}b'$, when every adaptive policy induces the same available-action sets,
terminal outcomes, per-level completion indicators, and scored-action counts under both.
\end{definition}

Since RHAE is determined by completion and action counts, replacing the game by any member of its
equivalence class preserves every policy's score.  Identification therefore need not recover source code,
privileged state, or even a unique program.  History may nevertheless leave several classes possible.
A probe is decision-relevant only if its outcomes separate classes that favor different continuations,
and its potential benefit must justify its scored cost.  Tycho does not enumerate $\Bel_t$; the actor
approximates this comparison by recording alternative hypotheses and deciding whether another interaction
or a model request is worthwhile.

\subsection{How much correctness does a plan require?}

Let $P=(a_t,\ldots,a_{t+m-1})$ be a model-generated plan.  Starting from the true protocol state $x_t$
and modeled state $\hat s_t$, consider the true and modeled executions of $P$.

\begin{definition}[Route-sufficient model]
The model is \emph{route-sufficient} for $P$ when, after every prefix, the true rendered grid belongs to
the output concretization $\Gamma_{\!O}$ of the model's predicted render and variants, and the true and
modeled outcomes agree.  At every nonfinal prefix, the next action must be available in both executions.
\end{definition}

\begin{proposition}[Route preservation]
If a model is route-sufficient for $P$, then $P$ remains executable and incurs $m$ scored actions.  A
modeled \emph{level\_complete} or \emph{game\_over} endpoint is also the true endpoint.
\end{proposition}

\begin{proof}
Induction over prefixes preserves action availability, the rendered-grid relation, and outcome agreement.
Applying final outcome agreement gives the claim.
\end{proof}

Route sufficiency requires no correctness away from $P$ and is thus weaker than global transition
match.  Conversely, exact replay constrains past transitions, not the unobserved transitions on a new
route.  Accepted transition match is therefore neither necessary nor sufficient for a high RHAE.
This clarifies the matched-policy result: repair can improve replay accuracy without improving RHAE,
while an abstract model can support a successful route despite errors elsewhere.  For an explicit
validated plan, Tycho surfaces one action at a time and withholds further plan actions when the observed
frame differs from the stored canonical prediction.

\section{Public Implementation and Audit Interface}
\label{app:implementation}

The released artifact contains the agent and harness source, prompt templates, paper configurations,
tests, aggregate inputs, and the six competition-scorecard manifests of the first version.  Release
0.2.0 is the harness version of the updated study (\Cref{sec:vtwo}); it adds that study's sixteen
configurations, a per-run summary with per-game action counts, and the links to its 48 official
scorecards.  \code{make validate} exercises
the package and its configuration and contract tests without model or ARC credentials, supporting fresh
stochastic runs and aggregate verification.  Before creating each official scorecard, deterministic
competition replay checked the recorded actions against the returned states and rejected any mismatch;
the released scorecard manifests contain the resulting public scorecard links and aggregate outcomes.

\subsection{Agent-facing evidence API}

The harness interface is the evidence layer shared by all policies, including no-world-model runs.  It is not
itself a learned model.  It is the typed view of the environment history that the actor can inspect,
reason over, and, when enabled, use to build an executable model.

\begin{description}[leftmargin=1.9em,itemsep=1pt]
  \item[Decision frames.] Each turn exposes the current playable grid and action interface.  In the
  workspace, \code{grid} and \code{wmlib.current\_grid()} give the current rendered frame as a numeric
  array; prompts may additionally carry an image, a text grid, and a text diff.
  \item[Transition and boundary history.] \code{wmlib.frames()} and \code{wmlib.transitions()} return
  decision states and ordinary action transitions.  Completed-level and fatal boundaries are represented
  separately by \code{terminal\_events()} and \code{death\_events()}, preserving the terminal frame and
  causal action needed to test \code{outcome(state)}.
  \item[Transient animation evidence.] \code{wmlib.animation\_index()} lists summary metadata for saved
  non-decision animation events.  \code{animation\_grids()} retrieves exact numeric frames for a selected
  event, optionally restricted to keyframes or requested indices.
  \item[Perceptual helpers.] \code{wmlib.diff\_text()}, \code{wmlib.segment()}, and
  \code{wmlib.segment\_summary()} provide lossless diffs and connected-component summaries.
  They are convenience views over recorded grids, not fixed representational commitments.
\end{description}

This API matters for interpretation of results.  A no-world-model actor still receives durable evidence
about current state, history, boundaries, and animations; it is not a bare language model asked to infer
from a lossy transcript.  Conversely, when executable world modeling is enabled, the model is judged
against this same record.  The policy comparison therefore asks whether model-using workflows improve the
complete agent on top of a faithful Moore-machine interface.

\subsection{Executable-model verification}

The required Python functions are \code{init\_state}, \code{transition}, \code{render}, and
\code{outcome}.  The optional \code{observation\_variants} hook represents display uncertainty;
\code{actions}, \code{subgoals}, \code{heuristic}, \code{planner\_key}, and a custom planner can support
tractable planning.  The verifier threads each level from its first frame,
checks every claimed cell, reports prediction coverage separately, and evaluates outcomes against recorded
completion and fatal boundaries.  It applies no cosmetic HUD mask, and completion animations and the next
level's initial frame remain outside the completed level's ordinary transition sequence.

The trigger configuration uses a fixed verifier gate.  It requests repair when accepted transition match is
below $0.999$, when prediction coverage is below $0.75$ or vacuous, when the model is absent or invalid, or
when the outcome classifier contradicts observed ordinary, completion, or fatal states.  New levels after the
first and fatal resets invoke the builder independently of this gate.  The coverage guard prevents a
substantially abstaining renderer from silencing repair merely because all cells it does claim are correct.
The value $0.75$ was fixed during harness development and used unchanged in the reported trigger run; it
was not selected by a threshold sweep or varied in an ablation.

\section{Corrections to the First Version}
\label{app:corrections}

Two corrections apply to the experiments of \Cref{sec:evaluation} as reported in the first version of
this paper (July 29, 2026).  Neither changes a score, a completed level, or an action count.  The affected
numbers are updated in place throughout this version; this appendix records what changed and why.

\paragraph{Terminal frame.}
The environment returns a stack of frames per action.  On a level-completing action the stack ends with
the goal-reached board followed by the first board of the next level; the evaluated harness instead
stored the first frame of the stack as the level's recorded terminal, and likewise the first frame rather
than the final one after a fatal action.  Replaying every published action through the engine shows that
the stored terminal differs from the correct one on 75--82 of the 157--166 level completions in each
Opus~4.8 run and on 97 of 183 in each of the two frontier runs, and that the stored fatal frame differs
on 5--13 of 12--15 fatal actions in the Opus~4.8 runs and on 7 of 10 and 6 of 6 in the frontier runs.
Ordinary transitions were read correctly.

Official scores, level counts, action counts, and costs do not read these frames.  The pre-action
diagnostics in \Cref{tab:online-model-diagnostics} (model execution, accepted transition match, coverage,
terminal-outcome recall, and the terminal false-positive rate) do not read them either and are numerically unchanged when the corrected
frames are substituted.  These quantities can depend on the error only through agent behaviour: the
affected frames were shown to the agent at level boundaries and after fatal actions under every policy,
and they were the reference of the terminal-render check.  Under the trigger policy that check caused 3 of
the 1{,}192 recorded builder invocations and was passed against an incorrect frame in 3{,}030 further
repair decisions across 14 games; in the 1{,}565 of these decisions that could be re-evaluated, the
correct frame would have requested a repair.  We cannot measure what a corrected run would have done.

Two reported items did depend on the stored terminals.  The ``verified games'' column of the end-of-run
outcome table (\Cref{tab:appendix-outcomes}) was computed against them and has been removed.  Three
case-study panels labelled as a recorded terminal or completion (\code{tu93} level~6 and level~8, and
\code{tr87}) showed the first frame of the completing action; they have been redrawn from the engine's
terminal frame, and the statements made about them in the text hold for the redrawn panels.  The harness
of the updated study (\Cref{sec:vtwo}) selects the terminal by the engine's frame-stack rule.

\paragraph{Inference cost and call counts.}
The first version computed cost as the sum over the calls logged with each trajectory step.  When a game
resumed after a service interruption, the steps replayed from the checkpoint kept their actions but not
their call logs, so calls that produced retained actions were missing from the sum.  We now report
\emph{retained-trajectory cost}: the spend on all calls that produced steps kept in the final trajectory,
taken from the run's cumulative meter or, where the archive has no meter, from the cumulative per-step
token counters.  Calls discarded by an interruption (work after the last checkpoint that was redone) are
excluded by convention; they amount to \$15.55 for the GPT-5.6 Sol run and \$150.19 for the Opus~5 run,
and cannot be measured for the four Opus~4.8 runs, whose archives contain no per-call event log.
\Cref{tab:cost-corrected} gives the first-version and corrected values.  The Opus~5 cost is unchanged
because it already used the run meter; its call count was the logged count and is corrected.  For the
two Opus~4.8 policies with a builder, builder spend on replayed steps is not recorded and is estimated,
so those two figures carry bounds.  All first-version runs keep the public list prices in force when
they ran (July 2026).  The ordering of the four policies by cost is unchanged, and trigger still costs
$39\%$ more than orchestrator.

The correction has three further consequences.  First, the \$750 per-game ceiling of the matched study
was applied after the fact to the logged cost, so the retained cost of several games exceeds \$750
although their trajectories were cut, or left uncut, according to the logged cost: up to \$966 without a
world model, \$1{,}075 under single, and an estimated \$1{,}157 under orchestrator; no trigger game exceeds
it.  The published trajectories, scores, and action counts are kept
as they are.  Second, the cost-ceiling analysis in panel (b) of \Cref{fig:cost-analysis} uses the logged
per-step costs.  The cost of steps replayed after a resume is not recorded per step in every run, so the
corrected curves are bounded rather than determined: scores at the \$100 and \$250 ceilings are lower by up
to ten points for the runs with the most replayed steps, and no score at the \$750 ceiling changes.  Third, the per-category token totals reported in the first
version covered logged calls only and are omitted from \Cref{tab:cost}.  Builder-call counts in
\Cref{tab:planner-adherence} are likewise counts of logged builder runs.

\begin{table}[t]
\centering
\small
\setlength{\tabcolsep}{4pt}
\begin{tabular}{@{}lrrrr@{}}
\toprule
 & \multicolumn{2}{c}{API-equiv.\ cost} & \multicolumn{2}{c}{LM calls} \\
\cmidrule(lr){2-3}\cmidrule(l){4-5}
Policy / model & first version & corrected & first version & corrected \\
\midrule
No world model / Opus~4.8 & \$5.66k & \$6.11k & 22.9k & 24.5k \\
Single / Opus~4.8 & \$7.27k & \$7.66k & 25.6k & 26.7k \\
Orchestrator / Opus~4.8 & \$5.78k & \$6.51k$^{\dagger}$ & 24.1k & 26.9k \\
Trigger / Opus~4.8 & \$8.03k & \$9.02k$^{\ddagger}$ & 44.4k & 50.3k \\
\addlinespace
Orchestrator / GPT-5.6 Sol & \$4.47k & \$4.88k & 26.5k & 29.1k \\
Orchestrator / Opus~5 & \$2.99k & \$2.99k & 15.1k & 18.7k \\
\bottomrule
\end{tabular}
\caption{First-version and corrected inference cost and call counts (retained-trajectory accounting;
July 2026 list prices; calls discarded by interruptions excluded).
$^{\dagger}$Estimate, bounds \$6.24k--\$6.63k.  $^{\ddagger}$Estimate, bounds \$8.62k--\$9.31k.}
\label{tab:cost-corrected}
\end{table}

\section{Individual Runs and Paired Comparisons of the Updated Study}
\label{app:vtwo-runs}

\Cref{tab:vtwo-per-run} lists the 48 runs of the updated study (\Cref{sec:vtwo}) individually, and
\Cref{tab:vtwo-paired} gives the paired per-game comparison of scored actions.  Every run was replayed
through the ARC-AGI API to an official competition-mode scorecard that reproduces the locally recorded
score, completed levels, and action count; the scorecards are listed with the released code.

\begin{table}[!htbp]
\centering
\footnotesize
\setlength{\tabcolsep}{4pt}
\begin{adjustbox}{max width=\linewidth}
\begin{tabular}{@{}llcrrrrrr@{}}
\toprule
Model & Policy & Run (start) & RHAE & Solved & Actions & vs.\ best route & Cost (USD) & Builder calls \\
\midrule
GPT-6.2 Sol (max) & No world model & 1 (18 Nov) & 97.09 & 24 & 7,525 & 1.34 & 26 & -- \\
 &  & 2 (20 Nov) & 100.00 & 25 & 7,121 & 1.29 & 23 & -- \\
 &  & 3 (21 Nov) & 100.00 & 25 & 6,495 & 1.18 & 28 & -- \\
 & Single & 1 (18 Nov) & 100.00 & 25 & 6,626 & 1.20 & 31 & -- \\
 &  & 2 (20 Nov) & 100.00 & 25 & 6,675 & 1.21 & 30 & -- \\
 &  & 3 (21 Nov) & 100.00 & 25 & 6,569 & 1.19 & 43 & -- \\
 & Orchestrator & 1 (18 Nov) & 100.00 & 25 & 6,637 & 1.21 & 74 & 194 \\
 &  & 2 (20 Nov) & 100.00 & 25 & 6,574 & 1.20 & 75 & 192 \\
 &  & 3 (21 Nov) & 100.00 & 25 & 6,381 & 1.16 & 125 & 218 \\
 & Trigger & 1 (18 Nov) & 100.00 & 25 & 6,658 & 1.21 & 135 & 724 \\
 &  & 2 (20 Nov) & 100.00 & 25 & 6,618 & 1.20 & 144 & 725 \\
 &  & 3 (21 Nov) & 100.00 & 25 & 6,272 & 1.14 & 214 & 702 \\
\addlinespace
GPT-6.2 Astra (high) & No world model & 1 (19 Nov) & 100.00 & 25 & 7,074 & 1.29 & 49 & -- \\
 &  & 2 (20 Nov) & 100.00 & 25 & 6,989 & 1.27 & 49 & -- \\
 &  & 3 (21 Nov) & 100.00 & 25 & 6,978 & 1.27 & 49 & -- \\
 & Single & 1 (19 Nov) & 100.00 & 25 & 6,509 & 1.18 & 61 & -- \\
 &  & 2 (20 Nov) & 100.00 & 25 & 6,828 & 1.24 & 63 & -- \\
 &  & 3 (21 Nov) & 100.00 & 25 & 6,952 & 1.26 & 62 & -- \\
 & Orchestrator & 1 (19 Nov) & 100.00 & 25 & 6,503 & 1.18 & 165 & 146 \\
 &  & 2 (20 Nov) & 100.00 & 25 & 6,396 & 1.16 & 164 & 145 \\
 &  & 3 (21 Nov) & 100.00 & 25 & 6,449 & 1.17 & 154 & 132 \\
 & Trigger & 1 (19 Nov) & 100.00 & 25 & 6,216 & 1.13 & 338 & 615 \\
 &  & 2 (20 Nov) & 100.00 & 25 & 6,313 & 1.15 & 339 & 594 \\
 &  & 3 (21 Nov) & 100.00 & 25 & 6,303 & 1.15 & 329 & 598 \\
\addlinespace
Claude Fable 5.2 (medium) & No world model & 1 (19 Nov) & 100.00 & 25 & 8,357 & 1.52 & 118 & -- \\
 &  & 2 (20 Nov) & 99.64 & 25 & 8,420 & 1.53 & 134 & -- \\
 &  & 3 (21 Nov) & 97.03 & 24 & 8,125 & 1.50 & 175 & -- \\
 & Single & 1 (19 Nov) & 100.00 & 25 & 7,804 & 1.42 & 116 & -- \\
 &  & 2 (20 Nov) & 99.98 & 25 & 8,003 & 1.46 & 127 & -- \\
 &  & 3 (21 Nov) & 96.25 & 24 & 8,469 & 1.57 & 194 & -- \\
 & Orchestrator & 1 (19 Nov) & 99.33 & 25 & 8,158 & 1.48 & 123 & 0 \\
 &  & 2 (20 Nov) & 98.89 & 25 & 8,470 & 1.54 & 127 & 0 \\
 &  & 3 (21 Nov) & 97.83 & 24 & 8,114 & 1.44 & 153 & 1 \\
 & Trigger & 1 (19 Nov) & 100.00 & 25 & 6,923 & 1.26 & 461 & 768 \\
 &  & 2 (20 Nov) & 100.00 & 25 & 7,170 & 1.30 & 461 & 755 \\
 &  & 3 (21 Nov) & 99.53 & 25 & 7,058 & 1.28 & 459 & 755 \\
\addlinespace
Claude Opus 5.5 (high) & No world model & 1 (19 Nov) & 100.00 & 25 & 7,209 & 1.31 & 160 & -- \\
 &  & 2 (20 Nov) & 99.40 & 25 & 7,494 & 1.36 & 164 & -- \\
 &  & 3 (21 Nov) & 100.00 & 25 & 7,214 & 1.31 & 168 & -- \\
 & Single & 1 (19 Nov) & 100.00 & 25 & 7,208 & 1.31 & 164 & -- \\
 &  & 2 (20 Nov) & 100.00 & 25 & 7,365 & 1.34 & 159 & -- \\
 &  & 3 (21 Nov) & 99.87 & 25 & 7,130 & 1.30 & 151 & -- \\
 & Orchestrator & 1 (19 Nov) & 100.00 & 25 & 7,123 & 1.30 & 157 & 0 \\
 &  & 2 (20 Nov) & 100.00 & 25 & 7,130 & 1.30 & 156 & 0 \\
 &  & 3 (21 Nov) & 100.00 & 25 & 7,277 & 1.32 & 165 & 0 \\
 & Trigger & 1 (19 Nov) & 100.00 & 25 & 6,743 & 1.23 & 429 & 722 \\
 &  & 2 (20 Nov) & 100.00 & 25 & 6,741 & 1.23 & 458 & 781 \\
 &  & 3 (21 Nov) & 100.00 & 25 & 6,889 & 1.25 & 460 & 744 \\
\bottomrule
\end{tabular}

\end{adjustbox}
\caption{Individual runs of the updated study.  Run: run number and start date (November 2026).  RHAE is
public-set RHAE; Solved: completed games out of 25; vs.\ best route: scored actions of completed games
divided by the best known route length for those games; Cost: retained-trajectory cost at list prices,
excluding the end-of-game reflection.}
\label{tab:vtwo-per-run}
\end{table}

\begin{table}[!htbp]
\centering
\small
\begin{tabular}{@{}llrrr@{}}
\toprule
Model & Comparison (A vs.\ B) & Actions A/B & 95\% interval & A fewer \\
\midrule
GPT-6.2 Sol & single vs.\ no world model & 0.961$^{*}$ & 0.921--0.998 & 15/25 \\
 & orchestrator vs.\ no world model & 0.948$^{*}$ & 0.921--0.974 & 19/25 \\
 & orchestrator vs.\ single & 0.987 & 0.961--1.015 & 15/25 \\
 & trigger vs.\ no world model & 0.955$^{*}$ & 0.918--0.996 & 17/25 \\
 & trigger vs.\ single & 0.994 & 0.961--1.031 & 14/25 \\
 & trigger vs.\ orchestrator & 1.006 & 0.976--1.040 & 14/25 \\
\addlinespace
GPT-6.2 Astra & single vs.\ no world model & 0.975$^{*}$ & 0.951--0.997 & 19/25 \\
 & orchestrator vs.\ no world model & 0.955$^{*}$ & 0.917--0.988 & 18/25 \\
 & orchestrator vs.\ single & 0.980 & 0.947--1.006 & 16/25 \\
 & trigger vs.\ no world model & 0.942$^{*}$ & 0.895--0.988 & 20/25 \\
 & trigger vs.\ single & 0.966 & 0.926--1.007 & 18/25 \\
 & trigger vs.\ orchestrator & 0.986 & 0.968--1.007 & 16/25 \\
\addlinespace
Claude Fable 5.2 & single vs.\ no world model & 0.968 & 0.926--1.011 & 15/25 \\
 & orchestrator vs.\ no world model & 0.987 & 0.943--1.034 & 14/25 \\
 & orchestrator vs.\ single & 1.022 & 0.975--1.078 & 11/25 \\
 & trigger vs.\ no world model & 0.856$^{*}$ & 0.811--0.898 & 24/25 \\
 & trigger vs.\ single & 0.884$^{*}$ & 0.850--0.919 & 24/25 \\
 & trigger vs.\ orchestrator & 0.866$^{*}$ & 0.818--0.912 & 22/25 \\
\addlinespace
Claude Opus 5.5 & single vs.\ no world model & 0.979 & 0.946--1.010 & 13/25 \\
 & orchestrator vs.\ no world model & 0.977 & 0.931--1.017 & 13/25 \\
 & orchestrator vs.\ single & 0.998 & 0.970--1.027 & 15/25 \\
 & trigger vs.\ no world model & 0.939$^{*}$ & 0.900--0.985 & 21/25 \\
 & trigger vs.\ single & 0.959 & 0.924--1.003 & 18/25 \\
 & trigger vs.\ orchestrator & 0.961 & 0.925--1.003 & 20/25 \\
\bottomrule
\end{tabular}

\caption{Paired per-game comparison of scored actions in the updated study.  Actions A/B is the geometric
mean over games of the ratio of action counts under policy A and policy B, with ratios of the same game
averaged over the three runs; only games completed under both policies in a run enter.  The interval is a
$95\%$ percentile bootstrap over the 25 games; $^{*}$ marks intervals that exclude one.  ``A fewer''
counts games in which A used fewer actions on average.  No correction for multiple comparisons is
applied.}
\label{tab:vtwo-paired}
\end{table}

\FloatBarrier

\section{Run Diagnostics}
\label{app:diagnostics}

\newcommand{\AppendixOperationalTable}{%
\begin{table}[H]
\centering
\scriptsize
\setlength{\tabcolsep}{3.2pt}
\begin{tabular}{@{}lrrrrrrr@{}}
\toprule
Policy & Wins & Level limit & LM limit & Cost limit & Context comp. & Peak prompt & LM calls \\
\midrule
No world model & 19 & 3 & 0 & 3 & 3 & 850k & 24.5k \\
Single & 19 & 2 & 0 & 4 & 4 & 851k & 26.7k \\
Orchestrator & 21 & 1 & 1 & 2 & 3 & 852k & 26.9k \\
Trigger & 18 & 2 & 5 & 0 & 1 & 852k & 50.3k \\
GPT-5.6 & 25 & 0 & 0 & 0 & 5 & 219k & 29.1k \\
Opus 5 & 25 & 0 & 0 & 0 & 0 & 300k & 18.7k \\
\bottomrule
\end{tabular}
\caption{Operational diagnostics for the six first-version runs.  LM calls are the corrected counts of
calls that produced retained trajectory steps (\Cref{app:corrections}).}
\label{tab:appendix-operational}
\end{table}
}

\newcommand{\AppendixRenderingTable}{%
\begin{table}[H]
\centering
\scriptsize
\setlength{\tabcolsep}{3.2pt}
\begin{tabular}{@{}lrrrrrrr@{}}
\toprule
Policy & Trans. & Accepted & Strict & Known-cell & Coverage & UNKNOWN & Variant \\
\midrule
No world model & -- & -- & -- & -- & -- & -- & -- \\
Single & 4,510 & 22.24 & 6.9 & 90.2 & 99.5 & 30.8 & 0.5 \\
Orchestrator & 5,355 & 45.56 & 44.1 & 97.2 & 97.8 & 3.8 & 0.8 \\
Trigger & 6,105 & 99.97 & 92.9 & 100.0 & 98.9 & 5.5 & 1.6 \\
GPT-5.6 & 6,486 & 99.49 & 98.8 & 100.0 & 100.0 & 0.3 & 0.4 \\
Opus 5 & 6,114 & 70.26 & 63.2 & 95.7 & 99.5 & 7.8 & 2.7 \\
\bottomrule
\end{tabular}
\caption{End-of-run rendering diagnostics over graded transitions (percent except Trans.).}
\label{tab:appendix-rendering}
\end{table}
}

\newcommand{\AppendixOutcomeTable}{%
\begin{table}[H]
\centering
\scriptsize
\setlength{\tabcolsep}{3.2pt}
\begin{tabular}{@{}lrrrr@{}}
\toprule
Policy & LC recall ($n$) & GO recall ($n$) & LC false + & GO false + \\
\midrule
No world model & -- & -- & -- & -- \\
Single & 28.5 (137) & 25.0 (12) & 0.2 & 0.6 \\
Orchestrator & 46.5 (155) & 66.7 (6) & 0.3 & 0.0 \\
Trigger & 98.1 (156) & 100.0 (20) & 0.0 & 0.0 \\
GPT-5.6 & 99.5 (183) & 100.0 (10) & 0.0 & 0.0 \\
Opus 5 & 87.8 (181) & 100.0 (6) & 3.4 & 0.0 \\
\bottomrule
\end{tabular}
\caption{End-of-run outcome diagnostics (percent except event counts).  The first version's
``verified games'' column was computed against incorrectly stored terminal frames and has been removed
(\Cref{app:corrections}).}
\label{tab:appendix-outcomes}
\end{table}
}

\newcommand{\AppendixPlannerTable}{%
\begin{table}[H]
\centering
\scriptsize
\setlength{\tabcolsep}{3.2pt}
\begin{tabular}{@{}lrrrrrr@{}}
\toprule
Policy & Builder calls & End-run plan lvls. & Auto first & Manual first & Builder first & Plan prefix \\
\midrule
No world model & 0 & 0 & -- & -- & -- & -- \\
Single & 0 & 0 & 25/29 & -- & -- & 116/217 \\
Orchestrator & 147 & 5 & -- & -- & 35/68 & -- \\
Trigger & 1192 & 34 & -- & -- & 552/970 & -- \\
GPT-5.6 & 660 & 34 & -- & 4/5 & 634/644 & 49/150 \\
Opus 5 & 130 & 32 & -- & -- & 109/129 & -- \\
\bottomrule
\end{tabular}
\caption{Builder activity and executable-model use.}
\label{tab:appendix-planner}
\end{table}
}

These diagnostics complement the per-game RHAE, level-completion, and action counts available from the
official scorecards.  The scorecard links and per-game manifests are released in
\path{artifacts/scorecards/}; machine-readable diagnostic values and aggregate figure inputs are released
in \path{artifacts/appendix_metrics.json} and \path{artifacts/figure_data.json}.

\subsection{Operational and end-of-run checks}

\Cref{tab:appendix-operational} separates successful engine \code{win} terminations from stops imposed by
the level-action, language-model-call, and inference-cost limits.  Context reductions count emergency
prompt reductions; peak prompt is the largest recorded prompt-token estimate; and LM calls count
language-model requests, restated here as the calls that produced retained trajectory steps
(\Cref{app:corrections}).  No run ended through the all-levels-accounted fallback, an exception, or the
output-length cap.

\ArxivOperationalTable

The rendering diagnostics in \Cref{tab:appendix-rendering} grade a transition when either the observed
grid changes or the model predicts a change.  Accepted transition match permits a bounded observation
variant and ignores cells explicitly marked \code{UNKNOWN}; strict match requires the canonical full
render.  Known-cell accuracy is weighted by claimed cells, coverage by transitions, and the last two
columns report how often each uncertainty mechanism was used.  The no-world-model policy has no
executable renderer.

\AppendixRenderingTable

\Cref{tab:appendix-outcomes} evaluates \code{level\_complete} (LC) and \code{game\_over} (GO).  Recall is
measured on recorded terminal evidence, false-positive rates on ordinary decision states, and parentheses
give terminal-event counts.  Inference-cost-capped games are excluded from all end-of-run model
aggregates.  The first version additionally reported a count of ``verified games'', which also required
the model's rendering of each recorded level terminal to match.  That count was computed against
incorrectly stored terminal frames and has been removed (\Cref{app:corrections}).

\ArxivOutcomeTable

\subsection{Model use and pre-action repair}

We count only explicit harness events with stable trace-level definitions; arbitrary Python calls and file
rewrites are excluded.  In \Cref{tab:appendix-planner}, end-run plan levels count completed levels for
which model replay exposes at least one plan.  The remaining columns use recommendations surfaced before
the committed action.  A first action is followed only when the actor commits the same fully specified
action later in that harness turn.  Manual recommendations are explicit \code{plan.py} results, automatic
recommendations come from verifier feedback, and builder advice is separate because it need not originate
in search.  An explicit \code{plan.py} invocation creates the validated cross-turn artifact.  Automatic
feedback never does so.  Plan prefix counts surfaced actions until the first deviation or boundary.  A dash
means that the policy exposed no parseable recommendation through that channel, not that no planning
occurred.

\AppendixPlannerTable

The main text reports micro-averaged accepted transition match.  Game-macro accepted transition match is $16.6\%$,
$17.8\%$, and $83.0\%$ for single, orchestrator, and trigger, preserving the same ordering.  A conservative
lower bound that counts every action without an available pre-action model as inexact is $12.4\%$,
$13.5\%$, and $45.5\%$.

Repair-recovery measurements give a second view.  After a rejected prediction, orchestrator restores an
accepted prediction on the first post-builder action in 29 of 73 evaluable cases ($39.7\%$) and within
five predictions in 35 ($47.9\%$).  Trigger does so in 354 of 531 ($66.7\%$) and 503 of 531 ($94.7\%$),
respectively.  The evaluable subset contains 103 of 147 orchestrator builder calls and 716 of 1{,}192
trigger calls.  These rates are associational: a builder call can coincide with an already informative
trajectory, and the five-step window can include another call.

\FloatBarrier

\section{Generated Program versus Game Engine: \code{ls20}}
\label{app:ls20-generated-program-audit}

Replay metrics show whether a final simulator reproduces recorded transitions,
but not which engine rules it encodes or where its behavior diverges off the
recorded trajectory. We therefore compare the generated program, notes,
auxiliary files, and action trace from one completed \code{ls20} run with the
public engine source. This retrospective case study asks what Tycho constructed
and how builder recommendations appear in the action stream.

The run uses the Orchestrator policy described in
\Cref{tab:orchestration-policies}. 
We did not supply the actor or builder with an explicit description of
\code{ls20}'s mechanics. We did supply our ARC-AGI-3-specific prompts and
initial workspace, documented in \Cref{app:interface}. These specify the action interface,
interaction history retained across levels, explicit
state/transition/render/outcome interfaces, attention to hidden state and
heads-up display (HUD) variables, hypothesis-discriminating probes, replay
verification, uncertainty handling, and simulator-based search. The
\code{ls20} hypotheses and code changes were produced during interaction
within this workflow.

For ground truth, we use \path|ls20.py| from public engine package version
\code{9607627b}. 
The alphanumeric tags shown below are non-semantic identifiers
from the obfuscated engine source, not names assigned by Tycho. We include them only as searchable source
anchors.

\subsection{Program-engine correspondence}

\begin{table}[H]
    \centering
    \small
    \setlength{\tabcolsep}{5pt}
    \begin{tabular}{@{}
    >{\raggedright\arraybackslash}p{0.34\linewidth}
    >{\raggedright\arraybackslash}p{0.57\linewidth}@{}}
    \toprule
    Item & Audited value \\
    \midrule
    Inference activity &
    GPT-5.6 Sol at maximum reasoning effort; 1,749 completed language-model
    calls: 661 in
    the actor role, 1,068 in the builder role, 18 summarizing completed levels,
    and 2 summarizing animations; \$264.77 recorded list-price cost \\
    Benchmark result &
    7/7 levels, Relative Human Action Efficiency (RHAE; completion and action
    efficiency) 100 under the official scoring rule, \code{WIN}; 478 scored
    environment actions:
    17, 103, 41, 43, 75, 111, and 88 by level, including 3 \code{RESET}
    commands \\
    Program construction &
    98-line initial simulator template present at the first scored action; 38
    actor-requested builder invocations; 1,319-line final
    \path|world_model.py|; three generated auxiliary files totaling 347
    lines \\
    Engine source inspected for this audit &
    \path|ls20.py|; 2,060 lines \\
    \bottomrule
    \end{tabular}
    \caption{Execution and artifact scope of the audited \code{ls20} run.}
    \label{tab:ls20-audit-scope}
\end{table}

During the run, Tycho produced an executable Python simulator, implementing the
prescribed state, transition, rendering, and outcome interfaces.
Its \code{State} dataclass represents
avatar position, the remaining action-counter value (named \code{fuel} in the
generated code), carried shape and color, refills, game-board transformation
operators, launchers, target chambers, moving-object positions and directions,
and an internal map of terrain hidden by fog
(\path|world_model.py|). 
\code{init\_state} parses an observed grid, \code{transition} applies one
action, \code{render} produces the predicted grid, and \code{outcome} predicts
ongoing, completed, or game-over status (\path|world_model.py|, lines 240, 605,
910, and 1006). \code{actions} lists candidate moves; \code{subgoals} and
\code{heuristic} expose targets and cost estimates to search
(\path|world_model.py|).
Replay follows the typed decision-frame and verification protocol in
\Cref{sec:formal,sec:programmatic}; transient animation frames remain separate
evidence and are not graded as state transitions.

\Cref{tab:ls20-audit-mechanics,tab:ls20-audit-representations} use five audit
labels. A \emph{match} means that the generated code implements the stated
engine mechanic within the listed scope. A \emph{lookup table} maps a finite,
enumerated set of inputs to successor patterns and defines a fallback for all
other inputs. An \emph{approximation} agrees on the successful action sequences
replayed here but differs in mechanism or scope. An \emph{omission}
is an engine behavior with no counterpart in the generated transition.
\emph{Data} denotes fixed per-level engine configuration represented as
inferred constants or per-level state. Such values may be valid state
estimates, but they do not constitute a general rule for generating levels.

\begin{table}[H]
    \centering
    \scriptsize
    \setlength{\tabcolsep}{3pt}
    \renewcommand{\arraystretch}{1.08}
    \begin{tabular}{@{}
    >{\raggedright\arraybackslash}p{0.37\linewidth}
        >{\raggedright\arraybackslash}p{0.45\linewidth}
        >{\raggedright\arraybackslash}p{0.09\linewidth}@{}}
        \toprule
        Engine mechanic (source tag) &
        Counterpart in \code{world\_model.py} & Label \\
        \midrule
        $5\times5$ avatar moves one tile; wall tiles block
        (\code{ihdgageizm}) &
        The parser first matches the observed avatar: two rows of color 12 above three
        rows of color 9. If its palette changes, a fallback accepts the same
        geometry in any two non-terrain colors. Movement permits floor and
        operator cells, blocks walls, and tests target entry at tile centers &
        match \\
        \addlinespace
        Ordinary and wall-blocked movement decrements the counter.
        \code{StepsDecrement} is 1 on levels 0, 3, and 5, and defaults to 2 &
        Shared decrement and blocked-move logic; values inferred from counter
        shortening are 1, 2, 2, 1, 2, 1, and 2 for levels 0--6 & match $+$ data \\
        \addlinespace
        Hollow rings refill the counter once and disappear on contact
        (\code{npxgalaybz}) &
        Ring detection by an exact $3\times3$ ring outline; contact refills the
        counter and removes the ring. Respawn restoration is assessed below &
        match \\
        \addlinespace
        Rotation contact advances by $90^\circ$ (\code{rhsxkxzdjz}) &
        Clockwise \code{np.rot90} of the carried shape shown in the HUD when
        the avatar newly contacts the operator; all recorded rotations agree &
        match \\
        \addlinespace
        Color-operator contact advances through ARC palette color IDs
        $12\to9\to14\to8\to12$
        (\code{soyhouuebz}) &
        The colors are read clockwise from the visible wheel and applied in the
        same order & match \\
        \addlinespace
        Matching contact consumes a chamber; the level completes when no chambers
        remain, with two sequential target chambers on level 5
        (\code{rjlbuycveu}) &
        The visible $3\times3$ pattern and color are compared, so rotation is
        included in the pattern;
        the first target is removed when matched and the last completes the
        level & match \\
        \addlinespace
        Entry into a mismatched chamber is rejected without decrementing the
        counter; moving operators undo their tentative step &
        The avatar, counter, and moving-operator positions and directions also
        remain unchanged &
        match \\
        \bottomrule
    \end{tabular}
    \caption{
        Engine rules matched by executable mechanics in the generated
        program.
    }
    \label{tab:ls20-audit-mechanics}
\end{table}

\begin{table}[H]
\centering
\scriptsize
\setlength{\tabcolsep}{3pt}
\renewcommand{\arraystretch}{1.08}
\begin{tabular}{@{}
  >{\raggedright\arraybackslash}p{0.37\linewidth}
  >{\raggedright\arraybackslash}p{0.45\linewidth}
  >{\raggedright\arraybackslash}p{0.09\linewidth}@{}}
\toprule
Engine rule or representation (source tag) &
Counterpart in \code{world\_model.py} & Label \\
\midrule
Shape contact advances an unobserved internal index through six fixed
$3\times3$ patterns, independent of entry direction (\code{ttfwljgohq}) &
The program generates a 24-entry lookup table. It
produces the correct successor for every engine-valid shape and rotation;
an input outside that table is predicted to remain unchanged &
match $+$ lookup \\
\addlinespace
When a command starts with the counter at zero, movement and contact handling
still occur before possible life loss; respawn restores the avatar,
carried shape and color, counter, consumed rings, chambers, and moving-operator
positions and directions; the last life ends the game &
No exhaustion or respawn transition and no life counter are implemented;
\code{transition} cannot reach \code{game\_over} from \code{init\_state}.
\code{RESET} is selected by the actor, handled by the outer harness, and not
modeled by \code{transition} & omission \\
\addlinespace
Launch tiles move to the tile before the next wall or target-chamber coordinate;
a transformation operator on an intermediate tile does not shorten the move
(\code{gbvqrjtaqo}) &
A loop follows visible road until a wall, chamber, or non-road tile. It
matches recorded lanes and crosses operator tiles that retain road pixels,
but tests pixels in the rendered map rather than the engine's precomputed wall
and target coordinates &
match $+$ approx. \\
\addlinespace
Moving operators advance before the avatar and undo on rejected movement.
Invisible engine objects mark allowed path cells; at each step the engine tries
straight, then right, left, and reverse
(\code{xfmluydglp}) &
The generated program also advances operators before the avatar, but its
line-segment patrols and eight-position road loop are inferred from visible
geometry and observed positions rather than the hidden path objects &
match $+$ approx. \\
\addlinespace
Level-6 fog masks pixels farther than 20 pixels from the avatar and renders the
HUD after masking so it remains visible &
The same 20-pixel visibility rule is applied to a $64\times64$ terrain map
reconstructed from observations, with \code{?} for cells never made visible &
match $+$ data \\
\addlinespace
On level 0, an operator contact that produces a target match skips that command's
decrement &
Every rotator contact that begins after the previous state had no rotator
contact is free whenever the initial grid contains no refill rings, whether or
not it produces a target match & approx. \\
\addlinespace
All three transformation operators apply whenever the command's destination
contains them, including contact with the same moving operator on consecutive
commands &
The color wheel applies on every such contact. Shape and rotation trigger only
when contact begins, so the generated program would miss a repeated contact on
consecutive commands; this case was not exercised in the trace &
approx. \\
\addlinespace
Counter cost, terrain, starts, target attributes, and fog are per-level
fixed engine data &
Costs, target state, and observed terrain are retained as constants or
internal per-level state, not produced by a general rule for generating levels &
data \\
\bottomrule
\end{tabular}
\caption{Finite lookup structure, unimplemented transitions, per-level data,
and approximations in the generated
\code{ls20} program. These forms can support successful play without being identical to
the engine's state representation or implementation.}
\label{tab:ls20-audit-representations}
\end{table}

The generated program encodes each visible $3\times3$ shape row by row as a
9-bit integer, one bit per cell. This abridged excerpt from \path|world_model.py|
shows how it maps one shape to the next:
\begin{quote}
\small
\begin{verbatim}
family = (179, 461, 413, 151, 367, 234)
known_successor = {}
for index, source in enumerate(family):
    target = family[(index + 1) % len(family)]
    for quarter_turns in range(4):
        known_successor[bitmap_value(np.rot90(
            bitmap(source), -quarter_turns
        ))] = bitmap_value(np.rot90(
            bitmap(target), -quarter_turns
        ))
new_value = known_successor.get(value, value)
\end{verbatim}
\end{quote}
The engine instead increments an unobserved internal shape index cyclically
through six values. The generated code constructs 24 entries:
the six shapes under all four rotations. It produces the same successor as the
engine for every valid visible shape, while using a finite lookup table rather
than separate shape and rotation indices. For an integer outside the table, it
predicts no shape change instead of marking the carried-shape cells
\code{UNKNOWN}, the simulator's marker for cells on which it makes no claim.
The engine explicitly assigns each mobile operator an invisible path object.
The generated program does not represent these objects, it infers bounded
segments and a road loop from visible geometry and observed positions. Because
the path objects are never rendered, replay can test the resulting operator
positions on the recorded action sequence, but cannot establish recovery of
the engine's path representation or behavior on unsampled branches. The
generated paths are therefore behavioral approximations even where recorded
movement is exact.

\subsection{Generated support code and action use}

The three generated auxiliary files make their level-specific content explicit.
\path|compact_l5_search.py| hard-codes three eight-step operator paths, goal
coordinate $(50,54)$, 9-bit target value 335, ARC palette color ID 9, and an
intermediate successor table. \path|search_mobile.cpp| hard-codes
70-node movement graphs, 9-bit operator value 282, 9-bit target value 371, and the
subsequently rejected modular addition rule. \path|fog_template.py| stores the
terrain map reconstructed from level-6 observations as 64
row strings with \code{?} for unobserved cells. These are useful
level-specific search programs and map estimates, not a shared,
level-independent transition program. This observation is specific to the
audited \code{ls20} run. Determining whether auxiliary-file use generally
increases on later or harder levels would require a corpus file creation
analysis.

The action trace records each builder recommendation alongside the action
submitted in the same actor turn. Each of the 38 builder invocations returned a
\code{recommended\_action}; 35 recommended a directional movement action and
three recommended \code{RESET}. In all 38 cases the subsequently submitted
action matched that recommendation.
This establishes that builder
recommendations were followed in the recorded action stream. It does not
show which actions or level completions would have occurred without them.

\subsection{Recorded hypothesis tests and revisions}

The notes in \path|notes/world_model.md| record a sequence of explicit,
testable revisions:

\begin{enumerate}
  \item On level 3, one contact maps the 9-bit shape value 179 to 461 while
  the visible operator pattern encodes 282. Because $179+282=461$, the notes
  propose arithmetic addition, then record that this single transition is also consistent with
  OR, XOR, or direct replacement.
  \item On level 4, the notes predict that a second contact should yield 231
  under addition modulo 512. The observed value is 413, rejecting that rule.
  A hypothesis that the result depends on approach direction is retained
  provisionally because approach side and the number of prior contacts changed
  together and could not yet be separated.
  \item Deliberately repeated level-5 contacts produce
  $371\to466\to493\to174\to410\to359$. Intermediate observations reject a
  proposed four-value cycle and fixed arithmetic or bitwise rules. Before the final
  contact, a rotation-consistency hypothesis---rotating an input should also
  rotate its successor---predicts $410\to359$; the next observation matches
  that prediction.
  \item Level 6 then supplies
  $151\to367\to234\to179\to461\to413$. Together with the rotated level-5
  transitions, these observations support the six-shape transition mapping,
  applied consistently under rotation, retained in the final program.
\end{enumerate}
Inspection of the engine after the run confirms that the operator contains no arithmetic
rule and that entry direction is not an input to its shape-index update. The
record therefore contains explicit predictions, observations that contradict
or match them, and corresponding code revisions retained in the workspace.

\subsection{Replay}

The aggregate diagnostics in \Cref{sec:evaluation,app:diagnostics} distinguish
pre-action from end-of-run models. This case study instead uses the final
\code{ls20} program for a mechanic-level comparison with the engine, replaying
each level's final successful attempt. Historical program versions preserved
in workspace snapshots were not replayed here; the figures below therefore
describe final fit, not the accuracy available when each action was selected.

The verifier accepts all 394 graded transitions: every concretely predicted
cell matches the recorded next grid without an optional alternative grid. This
set excludes 74 transitions from three \code{RESET}-abandoned attempts and the
\code{RESET} commands themselves, although all three count in the 478-action
benchmark total. Seven level-ending transitions are checked separately through
terminal rendering and outcome. Strict full-grid exactness (all $64\times64$
cells concrete and correct) is 94.92\% at 99.97\% coverage. Only level 6 uses
\code{UNKNOWN}: 20 transitions contain at least one such cell, and strict
exactness is 77.01\%. Every abstained cell is marked \code{?} in the static
level-6 terrain template and appears when the moving fog window reaches terrain
not encoded in that template. Abstention avoids guessing at first exposure,
but some coordinates remain \code{UNKNOWN} on later visits because the
simulator does not assimilate observed successor grids into its threaded
terrain state. Those repeated abstentions are a model limitation rather than
unavoidable fog uncertainty. Outcomes match all seven completions; six
terminal renders are strictly exact and the seventh is accepted with
\code{UNKNOWN}. No game-over terminal was observed.

Replay covers only final successful attempts. A stronger audit would initialize
both systems from the same level start, replay a common prefix, and compare
untaken branches over multiple steps. That test was not run, so no accuracy is
reported for those branches.

\subsection{Claim assessment}

\begin{table}[H]
\centering
\scriptsize
\setlength{\tabcolsep}{5pt}
\begin{tabular}{@{}
  >{\raggedright\arraybackslash}p{0.43\linewidth}
  >{\raggedright\arraybackslash}p{0.48\linewidth}@{}}
\toprule
Supported by this case & Not supported by this case \\
\midrule
Tycho produced executable Python code matching several engine mechanics,
including a 24-entry lookup correct for every engine-valid visible shape state. &
This does not show exact recovery of all engine rules or equivalent recovery
on other games; life loss and respawn are absent. \\
\addlinespace
Notes and code record proposals, distinguishing tests, and resulting revisions
to mechanics. &
The trace cannot determine whether GPT-5.6 Sol encountered this game during
training or establish adaptation independent of ARC-AGI-3 specific system
design. \\
\addlinespace
In accepted replay, every concretely predicted cell matches; \code{UNKNOWN}
occurs only in level 6's incomplete terrain map. &
Replay does not establish agreement on untried states, other action sequences,
or long multi-step predictions. \\
\addlinespace
Within Tycho's human-designed workflow, GPT-5.6 Sol calls produced game specific
hypotheses, code, searches, and actions; all 38 builder recommendations were followed. &
The score does not isolate GPT-5.6 Sol from ARC-AGI-3 specific prompts,
interfaces, workspace, tools, or actor--builder workflow, and does not measure
component contributions. \\
\bottomrule
\end{tabular}
\caption{Claims warranted by the inspected \code{ls20} artifact and important
non-implications.}
\label{tab:ls20-audit-claims}
\end{table}

This case shows where the adaptation occurred. During play, GPT-5.6 Sol,
operating through Tycho, generated and revised an \code{ls20}-specific
simulator, supporting code, and hypotheses. The prompts, action and evidence
interface, persistent workspace, actor--builder protocol, verifier, and planner
were fixed and designed by us for ARC-AGI-3. The result therefore demonstrates
game-specific program construction within a human-designed harness.

\end{document}